\documentclass{article}

\PassOptionsToPackage{numbers}{natbib}
\usepackage[preprint]{neurips_2026}
\usepackage{graphicx}

\usepackage[utf8]{inputenc} 
\usepackage[T1]{fontenc}    
\usepackage{hyperref}       
\usepackage{url}            
\usepackage{booktabs}       
\usepackage{amsfonts}       
\usepackage{nicefrac}       
\usepackage{microtype}      
\usepackage{xcolor}         

\usepackage{amssymb}
\usepackage{amsmath}
\usepackage{mathtools}
\usepackage{amsthm}

\usepackage{algorithm}
\usepackage{algpseudocode}
\usepackage{multirow}

\usepackage{enumitem}

\theoremstyle{plain}
\newtheorem{theorem}{Theorem}[section]
\newtheorem{proposition}[theorem]{Proposition}
\newtheorem{lemma}[theorem]{Lemma}

\theoremstyle{definition}

\theoremstyle{remark}

\title{Provably Learning Diffusion Models under the Manifold Hypothesis: Collapse and Refine}

\author{%
  Wei Huang \\
  RIKEN AIP \& The Institute of Statistical Mathematics \\
  \texttt{wei.huang.vr@riken.jp} \\
  \And
  Andi Han \\
  University of Sydney \\
  \texttt{andi.han@sydney.edu.au} \\
  \AND
  Mingyuan Bai \\
  Agency for Science, Technology and Research \& The Institute of Statistical Mathematics \\
  \texttt{Bai\_Mingyuan\_from.Riken@a-star.edu.sg} \\
  \And
  Huanjian Zhou \\
  The University of Tokyo\\
   \texttt{zhou-huanjian185@g.ecc.u-tokyo.ac.jp}
  \And
   Qixin Zhang \\
  Nanyang Technological University \\ \texttt{qixin.zhang@ntu.edu.sg}\\
 \And
   Taiji Suzuki \\
  The University of Tokyo \& RIKEN AIP\\
  \texttt{taiji@mist.i.u-tokyo.ac.jp} \\
  \And
   Kenji Fukumizu \\
 The Institute of Statistical Mathematics \\
  \texttt{fukumizu@ism.ac.jp} \\
}

\begin{document}

\maketitle

\begin{abstract}
Diffusion models generate high-dimensional data with remarkable quality, yet how training discovers and exploits low-dimensional manifold structure remains incompletely understood. We analyze a two-stage collapse-and-refine procedure motivated by the small-noise score: its singular normal component identifies manifold projection as the leading denoising target, while the remaining correction encodes density and geometry. Under suitable assumptions, we prove that annealed mean-field training asymptotically recovers the population denoiser and quantify the resulting geometric alignment error. With the learned representation fixed, residual score learning reduces to convex random-feature regression. We establish generalization bounds that separate residual learning error, propagated projection error, and geometric approximation bias, with statistical exponents governed by the intrinsic dimension. Under additional structural assumptions, we further derive an end-to-end sampling guarantee. Motivated by this analysis, we develop Score-induced Latent Diffusion (SiLD), a practical two-stage framework combining low-noise denoising representation learning with latent diffusion, without requiring KL regularization of the latent representation. Experiments on Stacked MNIST and molecular generation benchmarks demonstrate competitive generation quality relative to architecture-matched VAE-based latent diffusion baselines.
\end{abstract}

\section{Introduction}

Diffusion models have emerged as a dominant paradigm for generative modeling, demonstrating remarkable capability in synthesizing high-fidelity samples from complex, high-dimensional data distributions~\citep{sohl2015deep,ho2020denoising,song2019generative,song2020score}. The connection between these models and score matching, particularly through the lens of denoising autoencoders, has been firmly established~\citep{vincent2011connection,ho2020denoising}. Despite their empirical success, the theoretical foundations enabling efficient learning in high-dimensional spaces remain a subject of intense inquiry. A central puzzle lies in the ``curse of dimensionality'': without structural assumptions, nonparametric distribution learning can require sample complexity exponential in the ambient dimension $d$~\citep{wainwright2019high,biroli2024dynamical}. The prevailing resolution to this paradox is the manifold hypothesis, which posits that real-world data, while embedded in a high-dimensional ambient space, approximately reside on a low-dimensional manifold of intrinsic dimension $k \ll d$~\citep{fefferman2016testing,loaiza2024deep}.

Recent theoretical works have leveraged this low-dimensional structure to establish improved bounds on statistical estimation and sampling complexity~\citep{de2022convergence,chen2023score,oko2023diffusion,li2026when,tang2024adaptivity,potaptchik2024linear,azangulov2024convergence}, with guarantees that reflect the intrinsic dimension under suitable assumptions. Notably, \cite{li2024adapting} showed that suitably designed DDPM samplers adapt to unknown low-dimensional structure, while \cite{huang2026denoising} established iteration complexity scaling nearly linearly in $k$ without prior knowledge of the manifold. From a structural perspective, \cite{pidstrigach2022score} established conditions under which score-based generative models recover the data manifold, while \cite{stanczuk2024diffusion} showed that diffusion models encode the data manifold by approximating its normal bundle. Furthermore, \cite{farghly2025diffusion} showed that score smoothing implicitly regularizes toward manifold-adaptive solutions. More recently, \cite{boffi2024shallow,gao2024flow,kumar2026flow} investigated how diffusion and flow matching models adapt to low-dimensional structures.

However, how gradient-based optimization discovers and exploits the underlying nonlinear manifold structure remains less understood. A useful perspective comes from feature-learning theory: in supervised settings, training can align network features with task-relevant directions and learn low-dimensional representations~\citep{damian2022neural,mousavi2022neural}, with certain problems exhibiting sequential feature acquisition through saddle-to-saddle dynamics~\citep{abbe2023sgd}. Mean-field theory provides a complementary framework for analyzing such representation dynamics by tracking the evolution of neuron parameter distributions~\citep{mei2018mean,chizat2022mean,suzuki2023convergence}. For diffusion models, recent studies have begun to investigate how the denoising objective enables adaptation to low-dimensional structure. Shallow-network analyses establish statistical guarantees for learning distributions with hidden low-dimensional structure~\citep{boffi2024shallow}, gradient descent on the DDPM objective can recover Gaussian mixture parameters~\citep{shah2023learning}, and optimizing the diffusion loss under a low-rank parameterization connects to subspace clustering~\citep{wang2025diffusion}. These results establish low-dimensional adaptation in specific mixture or subspace models, but leave open how the geometric structure of the score itself drives representation learning on nonlinear manifolds. This motivates the following fundamental question:

\emph{How does a neural network adaptively discover the low-dimensional data support and learn the intrinsic distribution density through denoising-based training?}

We analyze a collapse-and-refine mechanism under the manifold hypothesis: the singular small-noise score guides geometric alignment, and residual learning captures density and curvature. We study a score estimator connecting representation learning to statistical accuracy, then use this analysis to motivate the practical Score-induced Latent Diffusion (SiLD) algorithm. Our contributions are as follows:

\begin{itemize}[leftmargin=*]
    \item \textbf{Convergence Guarantees.} We establish learning guarantees for both stages of the collapse-and-refine procedure. In Stage 1, an annealed population mean-field analysis proves asymptotic convergence to the population denoiser and bounds the resulting geometric alignment error. In Stage 2, fixing the recovered representation reduces residual score learning to convex random-feature regression, linking geometric feature learning to subsequent density refinement.

    \item \textbf{Generalization Guarantees.} We establish generalization bounds in the low-noise manifold regime, showing how the recovered geometry enables residual score estimation with statistical exponents governed by the intrinsic dimension. Under additional structural assumptions, a complementary high-noise analysis and terminal denoising extend these guarantees to end-to-end sampling complexity.

    \item \textbf{Empirical Evaluation.} We implement the collapse-and-refine principle as practical SiLD: low-noise denoising representation learning followed by latent diffusion. We evaluate it on Stacked MNIST and molecular generation benchmarks. SiLD achieves comparable or better generation quality than architecture-matched VAE-based latent diffusion baselines on Stacked MNIST and molecular generation benchmarks.
\end{itemize}

\section{Related Work}

\textbf{Statistical Theory of Diffusion Models.} Beyond the convergence and adaptivity results discussed in the introduction, several works
have examined more fine-grained aspects of diffusion models on manifolds.
\cite{benton2023nearly} used stochastic localization to establish sampling complexity nearly linear in the ambient dimension, assuming accurate score estimates. Most recently, \cite{chakraborty2026generalization} introduced the $(p,q)$-Wasserstein dimension and derived Wasserstein-$p$ statistical guarantees under finite-moment assumptions on the target distribution, without requiring compact support or a smooth manifold density. On the geometric side, \cite{chandramoorthy2025and} studied conditions under which sampling dynamics preserve manifold support despite small errors in the learned vector field, while
\cite{fukumizu2026flow} characterized exponential normal contraction and tangential neutrality in the time-rescaled terminal dynamics of OT-CFM. \cite{liu2025improving} examined the sampling difficulties caused by normal score singularities and proposed methods to mitigate them.
For score estimation, \cite{george2026asymptotic} derived asymptotically exact learning curves for random-feature networks, showing linear intrinsic-dimension scaling for linear manifolds and weaker benefits for nonlinear manifolds within their setting. Our work connects geometric adaptation with statistical learning by relating geometric alignment during training to the accuracy of subsequent residual score estimation.

\textbf{Training Dynamics of Generative Models.}
Recent studies have begun to connect diffusion training with optimization, generalization, and feature learning.
Using the neural tangent kernel (NTK) framework~\citep{jacot2018neural}, \cite{han2024neural} established generalization guarantees for gradient-descent-trained score networks, although a fixed-kernel description does not capture the evolution of learned features.
\cite{shah2023learning} proved Gaussian mixture recovery via gradient descent on the DDPM objective, while \cite{wang2025diffusion} connected diffusion-loss minimization under a low-rank parameterization to subspace clustering.
\cite{han2024feature} analyzed feature learning induced by denoising, and \cite{li2025understanding} characterized representation quality across diffusion noise levels.
Complementary work studies diffusion learning through high-dimensional asymptotics and convex optimization~\citep{cui2023high,cui2025precise,zhang2024analyzing,zeno2025diffusion}.
Spectral structure has also been investigated during sampling~\citep{wang2023diffusion} and training~\citep{wang2025analytical}, while \cite{bonnaire2025diffusion} identified distinct training timescales for generalization and memorization.

\textbf{Latent Diffusion Models.}
Latent diffusion models (LDMs) perform diffusion in compressed representations learned by autoencoders~\citep{rombach2022high}, including variational autoencoders~\citep{kingma2013auto}.
In conventional two-stage designs, the representation is trained before the latent diffusion model, leaving the connection between their learning objectives implicit.
Related approaches strengthen this connection through joint training of encoders and diffusion priors~\citep{vahdat2021score}, or by learning representations with a diffusion prior and decoder~\citep{heek2026unified}.
Motivated by our learning-theoretic analysis, Score-induced Latent Diffusion (SiLD) combines low-noise denoising representation learning with diffusion in learned latent coordinates.
This practical design follows the collapse-and-refine principle suggested by the analysis, without requiring a VAE-style KL penalty.

\section{A Two-Stage Framework for Score Learning}

\subsection{Preliminaries}

\textbf{Notation.}
We use $\|\cdot\|$ to denote the Euclidean norm for vectors and the Frobenius norm for
matrices unless otherwise specified. For a matrix $A$, $\|A\|_{\mathrm{op}}$ denotes
its operator norm. We
employ standard asymptotic notations such as $O(\cdot)$, $\Omega(\cdot)$, and $o(\cdot)$. We write $f \asymp g$ if
$f = O(g)$ and $g = O(f)$.

\textbf{Data distribution.}
Let $p_{\mathrm{data}}$ denote the unknown target distribution supported on a
$k$-dimensional compact smooth manifold $\mathcal{M}$ embedded in
$\mathbb{R}^d$, with $k \ll d$. For any point $x$ in the tubular neighborhood of
$\mathcal{M}$, we denote by $\Pi_{\mathcal{M}}(x)$ its projection onto $\mathcal{M}$
and by $d_{\mathcal{M}}(x) = \inf_{z \in \mathcal{M}} \|x - z\|$ its distance to
$\mathcal{M}$. The tangent and normal spaces of $\mathcal{M}$ at a point $z$ are
denoted $T_z\mathcal{M}$ and $N_z\mathcal{M}$, respectively. We denote by $\tau > 0$
the reach of $\mathcal{M}$, i.e., the largest radius such that the projection
$\Pi_{\mathcal{M}}$ is uniquely defined within the tubular neighborhood
${U}_\tau = \{x \in \mathbb{R}^d : d_{\mathcal{M}}(x) < \tau\}$.

\textbf{Diffusion process.}
We consider the Variance-Preserving (VP) forward process~\citep{ho2020denoising}, which
corrupts a clean sample $x_0 \sim p_{\mathrm{data}}$ as $x_t = \sqrt{\bar{\alpha}_t}\, x_0 + \sqrt{h(t)}\, \varepsilon,
    \quad \varepsilon \sim \mathcal{N}(0, I_d),$
where $\{\bar{\alpha}_t\}_{t \in [0,T]}$ is a monotonically decreasing noise schedule, and $h(t)=1 - \bar{\alpha}_t$ is the noise variance at time $t$. We denote the marginal distribution of $x_t$ by $p_t$. The score function of
$p_t$ is $s^*(x, t) := \nabla_x \log p_t(x)$.

\textbf{Score matching objective.}
Following~\cite{ho2020denoising,vincent2011connection}, the score function is estimated
by minimizing the denoising score matching (DSM) objective:
$ \mathcal{L}_w(\theta) = \frac{1}{2} \mathbb{E}_{t, x_0, \varepsilon}\left[ w(t)
    \left\| s_\theta(x_t, t) +  {\varepsilon}/{\sqrt{h(t)}} \right\|^2
    \right],$ $t\sim\mathrm{Unif}[t_{\min},T]$, $0<t_{\min}<T$, and $w$ is a bounded, strictly positive weighting function. The DSM identity gives $ \mathcal{L}_w(\theta) = \frac12\mathbb{E}_{t,x_t} \left[ w(t)\|s_\theta(x_t,t)-s^*(x_t,t)\|^2 \right] +C_w,$ where $C_w$ is finite and independent of $\theta$. Our main analysis uses $w(t)=1$.

\subsection{Manifold Hypothesis and Score Singularity}

\textbf{Manifold hypothesis.} We assume that $p_{\mathrm{data}}$ is supported on a $k$-dimensional compact smooth embedded manifold $\mathcal{M}\subset\mathbb{R}^d$ without boundary, with $k\ll d$. Let $p_0$ denote its density with respect to the Riemannian volume measure on $\mathcal{M}$. At small noise levels, the score separates into a singular normal component and a bounded correction containing density and geometric information~\citep{de2022convergence,li2026when}. The following proposition makes this structure explicit.

\begin{proposition} \label{prop:score_decomp}
Assume $p_0 \in C^4(\mathcal{M})$ with $0 < p_{\min} \leq p_0 \leq p_{\max} < \infty$. There exist constants $h_0\in(0,1)$ and $C>0$, depending on $\mathcal{M}$ and $p_0$, such that, for $0<h(t)\leq h_0$ and $x\in U_{\tau/2}$,
\begin{equation}
\label{eq:score-decomposition}
    s^*(x,t) = \underbrace{-\,\frac{x - \Pi_{\mathcal{M}}(x)}{h(t)}}_{\text{(I): }  {O}(h(t)^{-1})}
    + \underbrace{\nabla_x \log p_0\bigl(\Pi_{\mathcal{M}}(x)\bigr)
    + \nabla_x H\bigl(\Pi_{\mathcal{M}}(x),\, x - \Pi_{\mathcal{M}}(x)\bigr)}_{\text{(II): } O(1)}
    +  {O}(h(t)),
\end{equation}
where $ H: \bigl\{(z,\nu): z\in\mathcal{M},\ \nu\in N_z\mathcal{M},\ \|\nu\|<\tau\bigr\} \longrightarrow\mathbb{R}$ is a smooth function.
\end{proposition}

The proof is deferred to the appendix \ref{sec:proof_pro}. Term~\emph{(I)} is the \emph{normal restoring force}, pointing from $x$ toward its projection onto $\mathcal{M}$. Term~\emph{(II)} is the \emph{intrinsic-density and geometric correction}: the first summand $\nabla_x\log p_0(\Pi_{\mathcal{M}}(x))$ denotes the Euclidean gradient of the composition $x\mapsto\log p_0(\Pi_{\mathcal{M}}(x))$, which lies in $T_{\Pi_{\mathcal{M}}(x)}\mathcal{M}$ and encodes the intrinsic density gradient. The second summand accounts for VP shrinkage and the geometry of the embedding. Explicitly, $ H(z,\nu) = -\frac12\langle\nu,z\rangle -\frac12\log\det(I_k-B_\nu),$ and $(B_\nu)_{ij} := \langle\nu,\mathrm{II}_z(e_i,e_j)\rangle,$ where $\{e_i\}_{i=1}^k$ is an orthonormal basis of $T_z\mathcal{M}$ and $\mathrm{II}$ is the second fundamental form. The first term is a \emph{VP-shrinkage correction} arising from the rescaling $y=x/\sqrt{1-h(t)}$, while the second is an \emph{extrinsic-curvature correction}. The gradient of this geometric correction generally depends on the normal displacement and can have both tangent and normal components.

\paragraph{Implications of scale separation.} Proposition~\ref{prop:score_decomp} reveals a fundamental asymmetry in the learning problem. At fixed off-manifold locations, the normal restoring force has magnitude $d_{\mathcal{M}}(x)/h(t)$, while the density and geometric corrections remain bounded as $h(t)\to0$. This \emph{scale separation}~\citep{li2026when} has two implications. First, it identifies manifold projection as the leading target of low-noise denoising: $ x+h(t)s^*(x,t) = \Pi_{\mathcal{M}}(x)+O(h(t))$.
Thus, an absolute score estimation error of size $\delta$ induces an error of at most $h(t)\delta+O(h(t))$ in the corresponding estimate of the projection. This motivates learning geometric alignment before refining the residual score, with the resulting learning dynamics and statistical accuracy studied in the subsequent analysis. Second, the normal component is the gradient of an explicit scalar potential: $ -\frac{x-\Pi_{\mathcal{M}}(x)}{h(t)} = -\nabla_x\left( \frac{d_{\mathcal{M}}(x)^2}{2h(t)} \right),$ by Federer's identity $\nabla_x[\tfrac12 d_{\mathcal{M}}(x)^2] =x-\Pi_{\mathcal{M}}(x)$ on $U_\tau$.

\subsection{The Collapse-and-Refine Procedure: A Two-Stage Framework}
\label{sec:method}

The scale separation in Proposition~\ref{prop:score_decomp} motivates a two-stage score-learning procedure: Stage~1 learns geometric alignment through low-noise denoising, and Stage~2 fixes the learned representation and estimates the remaining score correction. We analyze a conservative-form two-layer network for Stage~1 and random-feature regression for Stage~2. This procedure provides the theoretical foundation for the collapse-and-refine principle and motivates the practical SiLD algorithm introduced later.

\textbf{Stage 1: Geometric Alignment.}
The leading singular term of the small-noise score is the normal restoring force \emph{(I)}. Since this term is a conservative vector field, we constrain our network to the same class by parametrizing it as the gradient of a scalar potential. Concretely, we train a \emph{conservative-form} two-layer network:
\begin{equation}
    f_1(x; \theta) \;=\; \frac{1}{h(t_1)} \Bigl( W\,\mathrm{diag}(a)\,\sigma(W^\top x + b) \;-\; x \Bigr),
    \label{eq:stage1_net}
\end{equation}
with parameters $\theta = (W, a, b)$, where $W = [w_1, \ldots, w_m] \in \mathbb{R}^{d \times m}$, $a \in \mathbb{R}^{m}$, $b \in \mathbb{R}^{m}$, and $\sigma(\cdot)$ is a nonlinear activation. This direction-parallel structure is precisely what makes $f_1$ a conservative vector field: as we show in Lemma~\ref{lem:structural-alignment}, $f_1 = -\nabla_x \Phi_{\mathrm{net}}$ for an explicit scalar potential $\Phi_{\mathrm{net}}$, matching the category of the target score at small noise.
We use the induced denoising map as an approximate projection
onto $\mathcal{M}$, with its alignment analyzed below:
\begin{equation}
    \hat{x} \;:=\; h(t_1)\, f_1(x; \theta) + x \;=\; W\,\mathrm{diag}(a)\,\sigma(W^\top x + b) \;\approx\; \Pi_{\mathcal{M}}(x).
    \label{eq:projection}
\end{equation}

\textbf{Stage 2: Density Estimation.}
Once Stage 1 has converged and the Stage-1 parameters are frozen, we decouple the singular normal component by constructing the full score network as:
\begin{equation}
    s_\theta(x, t_2) = -\frac{x - \hat{x}}{h(t_2)} + f_2(\hat{x}, t_2; \theta_2),
    \label{eq:stage2_net}
\end{equation}
where $f_2(\hat{x},t_2;\theta_2)=U\Phi(\hat{x},t_2)$
is a random-feature network with frozen features, and
$U\in\mathbb{R}^{d\times m}$ is the only trainable parameter.
By construction, the leading singular normal terms cancel up to the
Stage-1 alignment error.
With the Stage-1 map fixed, minimizing the DSM objective
is equivalent to the convex regression problem
\begin{equation}
\label{eq:stage2_loss}
\min_{\theta_2}
\mathbb{E}_{t_2,x_0,\varepsilon}
\left\|
f_2(\hat{x},t_2;\theta_2)
-
\frac{\sqrt{\bar{\alpha}_{t_2}}\,x_0-\hat{x}}{h(t_2)}
\right\|^2,
\end{equation}
where
$x=\sqrt{\bar{\alpha}_{t_2}}\,x_0
+\sqrt{h(t_2)}\,\varepsilon$
and $\hat{x}$ is obtained by applying
\eqref{eq:projection} to this noisy input. For $x\in U_{\tau/2}$ and $h(t)\leq h_0$, define $
s_{\mathrm{res}}^*(x,t)
:=
s^*(x,t)
+\frac{x-\Pi_{\mathcal{M}}(x)}{h(t)}.$ By Proposition~\ref{prop:score_decomp}, $
s_{\mathrm{res}}^*(x,t)
=
\nabla_x\log p_0\bigl(\Pi_{\mathcal{M}}(x)\bigr)
+
\nabla_xH\bigl(\Pi_{\mathcal{M}}(x),
x-\Pi_{\mathcal{M}}(x)\bigr)
+
O(h(t))$.
Conditional on the noisy input and the fixed Stage-1 map,
the regression label has mean
$s_{\mathrm{res}}^*(x,t_2)
+(\Pi_{\mathcal{M}}(x)-\hat{x})/h(t_2)$.
Thus, the Stage-2 regression target also includes a correction for projection error.
Since the residual can depend on normal displacement,
the analysis retains the geometric approximation bias
associated with using only $\hat{x}$ as input.

Algorithm~\ref{alg:uld} summarizes the two-stage construction.
The analysis below considers population Stage-1 dynamics and a clipped
Stage-1 map, with independent Stage-2 samples and random features and
a separate empirical ridge fit at each fixed low-noise level.

\begin{algorithm}[t]
\caption{Two-Stage Collapse-and-Refine Training}
\label{alg:uld}
\begin{algorithmic}[1]
\Require Dataset $\mathcal{D}=\{x_i\}_{i=1}^n
\subset\mathbb{R}^d$, noise levels
$h(t_1)\ll h(t_2)$, regularization $\eta\geq0$

\State \textbf{Stage 1} (Geometric Alignment,
at noise level $h(t_1)$):
\State {Train the mean-field version of the Stage-1 network
in~\eqref{eq:stage1_net} at noise level $h(t_1)$
using the annealed dynamics~\eqref{eq:Lt_flow}}
\State Freeze the Stage-1 parameters, defining the map
$x\mapsto\hat{x}:=x+h(t_1)f_1(x)$

\State {\textbf{Stage 2} (Density Estimation},
at noise levels $h(t_2)$):
\State Sample $t_2$ according to the training noise schedule,
$x_0\sim\mathrm{Unif}(\mathcal{D})$, and
$\varepsilon\sim\mathcal{N}(0,I_d)$
\State Form
$x=\sqrt{\bar{\alpha}_{t_2}}\,x_0
+\sqrt{h(t_2)}\,\varepsilon$
and compute $\hat{x}=x+h(t_1)f_1(x)$
\State With $\Phi$ frozen, fit the residual network
$f_2(\hat{x},t_2)=U\Phi(\hat{x},t_2)$:
\Statex
\[
\widehat U
=
\arg\min_U\;
\mathbb{E}_{t_2,x_0,\varepsilon}
\left\|
U\Phi(\hat{x},t_2)
-
\frac{\sqrt{\bar{\alpha}_{t_2}}\,x_0-\hat{x}}{h(t_2)}
\right\|^2
+
\eta\|U\|_{\mathrm{F}}^2
\]
\State \Return Score network
$\displaystyle
s_\theta(x,t_2)
=
-\frac{x-\hat{x}}{h(t_2)}
+
\widehat U\Phi(\hat{x},t_2),
\quad
\hat{x}=x+h(t_1)f_1(x)$
\end{algorithmic}
\end{algorithm}

\section{Theoretical Analysis}
\label{sec:theory}

We now establish theoretical guarantees for the two-stage procedure introduced in Section~\ref{sec:method}, connecting geometric alignment in Stage~1 with residual learning in Stage~2. All proofs are deferred to the appendix \ref{app:proofs}.

\subsection{Stage 1: Dimensional Collapse}
\label{sec:stage1}

We first justify the architectural choice in \eqref{eq:stage1_net}. The leading normal term in the score decomposition~\eqref{eq:score-decomposition} is the negative gradient of the squared-distance potential. The following lemma shows that our conservative-form network respects this structure and can approximate the corresponding projection map.

\begin{lemma}[Structural Alignment] \label{lem:structural-alignment} Let $\sigma\in C^2(\mathbb{R})$ be non-polynomial with primitive $\Psi(z):=\int_0^z \sigma(u)\,du$. Fix $t_1$ with $h(t_1)>0$ and consider the conservative-form network~\eqref{eq:stage1_net} with parameters $\theta=(W,a,b)$ and the neural potential \[ \Phi_{\mathrm{net}}(x;\theta) := \frac{\|x\|^2}{2h(t_1)} - \frac{1}{h(t_1)} \sum_{j=1}^m a_j\Psi(w_j^\top x+b_j). \] Then $f_1(\cdot;\theta) =-\nabla_x\Phi_{\mathrm{net}}(\cdot;\theta)$ on $\mathbb{R}^d$ for every $\theta$. Moreover, for any $\varepsilon>0$, there exist $m=m(\varepsilon,\mathcal{M},\tau,\sigma)\in\mathbb{N}$ and parameters $\theta=(W,a,b)$ with $W\in\mathbb{R}^{d\times m}$ and $a,b\in\mathbb{R}^m$ such that \begin{equation} \label{eq:approx-target} \sup_{x\in\overline{U_{\tau/2}}} \left\| W\,\mathrm{diag}(a)\,\sigma(W^\top x+b) - \Pi_{\mathcal{M}}(x) \right\| < \varepsilon. \end{equation} \end{lemma}

We then analyze the Stage-1 training dynamics under the mean-field
framework~\citep{mei2018mean,chizat2022mean,suzuki2023convergence}.
To establish convergence while allowing feature learning, we consider
annealed mean-field Langevin dynamics: the output coefficients $a$ are
frozen, while the input weights $w$ and biases $b$ evolve with vanishing
Gaussian noise and weight decay~\citep{chizat2022mean,nitanda2022convex}.
The conservative-form architecture remains unchanged.
Lemma~\ref{lem:structural-alignment} establishes expressivity of the full
class.

Let $q\in\mathcal P_2(\mathbb R^{d+2})$ denote the joint distribution
over neuron parameters $\theta=(w,a,b)$, with trainable coordinates
$\xi=(w,b)$. We initialize $q_0=\pi$, where $\pi$ is the reference
distribution under which
$a\sim\nu_a=\tfrac12(\delta_{-A}+\delta_A)$ and
$\xi\sim\mathcal N(0,\sigma_0^2 I_{d+1})$ are independent, with
$A,\sigma_0>0$. Throughout this analysis, we take
$\sigma=\operatorname{erf}$. The mean-field neural network is
\begin{equation}
    P_q(x)
    :=
    \int_{\mathbb R^{d+2}}
    a\,w\,\sigma(w^\top x+b)\,q(\mathrm d\theta).
\end{equation}
This is the mean-field version of the map in~\eqref{eq:projection},
with the usual $1/m$ normalization absorbed into the finite-width
output coefficients. Fix $t=t_1$, write $h=h(t_1)$, and denote the
population denoiser by $D_t(x):=x+h\,s^*(x,t)$.
The excess DSM loss and its normalized version satisfy
\begin{equation}
\label{eq:normalized_dsm}
    L_t(q)
    :=
    \frac12\mathbb E_{x\sim p_t}
    \|f_1(x;q)-s^*(x,t)\|^2,
    \qquad
    J_t(q)
    :=
    h^2L_t(q)
    =
    \frac12\mathbb E_{x\sim p_t}
    \|P_q(x)-D_t(x)\|^2.
\end{equation}

For a positive, decreasing temperature $\lambda(s)$, we consider
the entropy-regularized objective~\citep{nitanda2022convex}
\begin{equation}
\label{eq:annealed_free_energy}
\begin{aligned}
    \mathcal E_s(q)
    &:=J_t(q)+\lambda(s)\operatorname{KL}(q\|\pi)
\end{aligned}
\end{equation}

Thus, the KL term combines
a quadratic penalty on $(w,b)$ with entropy regularization,
whose coefficients vanish as $\lambda(s)\to0$.

We study the corresponding Wasserstein gradient flow in the
trainable coordinates:
\begin{equation}
\label{eq:Lt_flow}
    \partial_s q_s
    =
    \nabla_\xi\cdot
    \left[
        q_s\left(
            \nabla_\xi\frac{\delta J_t}{\delta q}(q_s)
            +\frac{\lambda(s)}{\sigma_0^2}\xi
        \right)
    \right]
    +\lambda(s)\Delta_\xi q_s.
\end{equation}
The differential operators act on $(w,b)$ for each fixed $a$,
preserving the $a$-marginal $\nu_a$. The training temperature
$\lambda(s)$ controls parameter noise and weight decay and is
distinct from the data noise variance $h$. The Stage-1 alignment risk is defined as
\begin{equation}
\label{eq:stage1_obj}
    F_t(q)
    :=
    \frac12\mathbb E_{x\sim p_t}
    \|P_q(x)-\Pi_{\mathcal M}(x)\|^2.
\end{equation}
Here $\Pi_{\mathcal M}$ denotes the unique nearest-point projection
on $U_\tau$ and any measurable nearest-point selection outside
$U_\tau$. Our analysis first establishes convergence of $J_t$
and then controls $F_t$ using Proposition~\ref{prop:score_decomp}. Our convergence analysis requires neither a uniform-in-time
moment confinement assumption nor a geometric PL condition.

\begin{theorem}[Dimensional Collapse]
\label{thm:dimensional_collapse}
Under the assumptions of Proposition~\ref{prop:score_decomp}, fix a sufficiently small noise level $0<h(t_1)\le h_0$, with $d$ held fixed. There exists $\Gamma_0>0$ such that, for
$\Gamma>\Gamma_0$, the flow~\eqref{eq:Lt_flow} with
 $   \lambda(s) =
    \left(1+\frac{\log(1+s)}{\Gamma}\right)^{-1/3} $
satisfies $J_t(q_s)\to0$ as $s\to\infty$. Moreover, for all $s\geq0$,
\begin{equation}
\label{eq:collapse_rate}
    F_t(q_s)
    \leq
    2J_t(q_s)+C_{\mathrm{geom}}h(t_1)^2,
\end{equation}
where $C_{\mathrm{geom}}$ depends only on $\mathcal M$ and $p_0$.
Consequently, for fixed $d$,
$\limsup_{s\to\infty}
\mathbb E_{x\sim p_t}\|P_{q_s}(x)-\Pi_{\mathcal M}(x)\|=O(h(t_1))$
as $h(t_1)\downarrow0$.
\end{theorem}

Theorem~\ref{thm:dimensional_collapse} shows that annealed mean-field training learns an approximate manifold projection through low-noise denoising. The bound separates the optimization error $2J_t(q_s)$, which vanishes as training proceeds, from the $O(h(t_1)^2)$ finite-noise projection bias. Consequently, sufficiently long training yields an expected alignment error of order $h(t_1)$: the learned map brings noisy inputs close to their projections onto the data manifold. This establishes the geometric role of Stage~1 and motivates freezing the learned map before refining the remaining density and geometric corrections in Stage~2.

\subsection{Stage 2: Density Estimation on the Manifold}
\label{sec:stage2}

We next analyze residual learning with the Stage-1 map $P_{q_s}$ fixed.
At $h(t_2)\leq h_0$, write $a_h=\sqrt{1-h(t_2)}$ and define the
clipped Stage-1 map
$P(x)=\operatorname{Proj}_{B(0,R)}P_{q_s}(x)$, where $\|X_0\|\leq R$.
Clipping does not increase projection error.
With frozen random features $\Phi(z)=m^{-1/2}\sigma(V_x^\top z+b)$,
we retain the score model
\begin{equation}
 s_U(x,t_2)
 =
 -\frac{x-P(x)}{h(t_2)}+U\Phi(P(x)),
 \qquad U\in\mathbb R^{d\times m}.
 \label{eq:stage2_rf}
\end{equation}
For an independent Stage-2 sample
$x_i=a_hx_{0,i}+\sqrt{h(t_2)}\,\varepsilon_i$,
minimizing regularized DSM is equivalent to convex ridge regression
with computable targets:
\begin{equation}
 \widehat U
 =
 \arg\min_U
 \left\{
 \frac1n\sum_{i=1}^n
 \left\|
 U\Phi(P(x_i))
 -
 \frac{a_hx_{0,i}-P(x_i)}{h(t_2)}
 \right\|^2
 +\eta\|U\|_F^2
 \right\}.
 \label{eq:stage2_ridge}
\end{equation}

Let $X\sim p_{t_2}$, $Z=\Pi_{\mathcal M}(X)$,
$\varepsilon_{\mathrm{proj},2}^2=\mathbb E\|P(X)-Z\|^2$, and
$\mathcal R_h(U)=\mathbb E\|s_U(X,t_2)-s^*(X,t_2)\|^2$.
The projection uses a fixed measurable nearest-point selection outside
$U_\tau$.

The principal on-manifold residual in
Proposition~\ref{prop:score_decomp} is
\begin{equation}
 r_0(z):=\nabla_{\mathcal M}\log p_0(z)
       +\frac12\mathcal{H}_{\mathcal M}(z)
       -\frac12\operatorname{Proj}_{N_z\mathcal M}z,
 \qquad z\in\mathcal M,
 \label{eq:stage2_principal_residual}
\end{equation}
where $\mathcal{H}_{\mathcal M}(z)=\sum_{i=1}^k\mathrm{II}_z(e_i,e_i)$
is the trace of the second fundamental form in an orthonormal
tangent frame. Thus $r_0$ contains density, curvature, and
VP-shrinkage contributions and need not be tangent to $\mathcal M$.

Let $\rho$ be a fixed feature law on $\omega=(v,b)$ with
$v$-marginal $\mathcal N(0,I_d)$.
We draw $\omega_1,\ldots,\omega_m$ independently from $\rho$,
independently of Stage 1 and the Stage-2 data, and assume
$|\sigma|\le\kappa$ and $|\sigma'|\le L_\sigma$. Here $\omega_j=(v_j,b_j)$, $V_x=[v_1,\ldots,v_m]$, and
$b=(b_1,\ldots,b_m)^\top$.
To control the random-feature approximation error, assume
that $r_0$ admits the representation
\begin{equation}
 \begin{gathered}
 r_0(z)
 =
 \int c(\omega)\sigma(v^\top z+b)\,\rho(d\omega),
 \qquad z\in\mathcal M,\\
 \operatorname*{ess\,sup}_{\omega\sim\rho}
 \|c(\omega)\|\le B<\infty,
 \end{gathered}
 \label{eq:stage2_bounded_rf_representation}
\end{equation}
for some measurable $c:\mathbb R^{d+1}\to\mathbb R^d$.
This is an additional approximation assumption on the residual
relative to the Stage-2 feature class, not a consequence of the
manifold hypothesis alone.
The coefficient function $c$ is used only in the analysis.

For the following guarantee, assume the hypotheses and Stage-1
training dynamics of Theorem~\ref{thm:dimensional_collapse}.
Write $k=\dim\mathcal M\ge1$. Suppose the law of $Z$ has density at most $\rho_+$ on
$\mathcal M$, uniformly for $0<h(t_2)\le h_0$, and the
Laplace--Beltrami eigenvalues satisfy
$(1+\Lambda_{N+1})^{-1}\le C_WN^{-2/k}$ for $N\ge1$.

\begin{theorem}[Stage-2 generalization]
\label{thm:stage2_intrinsic_main}
Under the preceding assumptions, there exists a constant
$C\ge1$ such that the following holds for any $\delta\in(0,1)$
and sufficiently small $0<h(t_2)\le h_0$.
Choose $ h(t_1)=h(t_2)^{3/2}$, $\eta=h(t_2),$
 $m\ge\frac{C\log(16/\delta)}{h(t_2)}$, and
 $n\ge\frac{C\kappa^2}{h(t_2)}
          \log\frac{16m}{\delta}.$ Then, for each fixed, sufficiently large Stage-1 training time
$s$, conditionally on Stage 1, with probability at least
$1-\delta$ over the frozen features and the Stage-2 sample,
\begin{equation}
 \mathcal R_h(\widehat U)
 \le
 C\left(
 h(t_2)
 +
 \frac{\log(16/\delta)}
      {n\,h(t_2)^{\,2+k/(k+2)}}
 \right).
 \label{eq:stage2_intrinsic_main}
\end{equation}
\end{theorem}

The statistical exponent is governed by $k$, while the $h(t_2)$
term controls geometric bias, random-feature approximation,
regularization, and propagated Stage-1 error.
The bound is uniform in $d$ only when its geometric and
approximation constants are controlled uniformly.
Stage-1 convergence and noise-level transfer ensure that the
required finite training time exists; no training-time bound
uniform in the noise levels is asserted.
The proof is given in Appendix~\ref{app:stage2_detailed_bounds}.

\paragraph{Geometric correction under additional structure.}
For connected manifolds, parallel second fundamental form confines
the active geometry to an ambient subspace of dimension at most
$1+k+k(k+1)/2$.
For this class, an analytic curvature head removes the
leading $H$-dependent bias
and retains a $k$-controlled statistical term.
This conditional result uses the true projection and curvature;
applying it to Stage 1 requires a separate $C^2$ geometry
guarantee, which does not follow from
Theorem~\ref{thm:dimensional_collapse}.
The conditional statement and its proof are given in
Appendix~\ref{app:stage2_curvature_correction}.

\subsection{End-to-End Sampling Guarantee}
\label{sec:end2end}

Theorems~\ref{thm:dimensional_collapse} and~\ref{thm:stage2_intrinsic_main}
analyze score learning in the manifold regime. To study the reverse SDE over
the full integration path, we partition $[t_{\min},T]$ at $t_{\max}$, where
$h_*:=h(t_{\max})\leq h_0$~\cite{de2022convergence,benton2023nearly}.
In Phase I, $[t_{\min},t_{\max}]$, we use the subspace-restricted version
of the two-stage score model~\eqref{eq:stage2_net} specified below.
In Phase II, $[t_{\max},T]$, we use a separate
high-noise estimator because the projection-based score expansion need not
apply away from the manifold.

Assume an available linear subspace $E\supset\mathcal M$ satisfies
$r:=\dim E\le g(k)$, where $g$ is a fixed function independent of $d$,
and write $Q=\operatorname{Proj}_E$.
Impose the geometric, density, and feature assumptions of
Theorems~\ref{thm:dimensional_collapse} and~\ref{thm:stage2_intrinsic_main}
in $E$, with their quantitative parameters uniformly controlled independently
of $d$, and assume the Euclidean covering bound
\begin{equation}
 N(\mathcal M,\varepsilon)\le A_{\mathrm{cov}}\varepsilon^{-k},
 \qquad 0<\varepsilon\le1.
 \label{eq:end2end_covering}
\end{equation}
Here $A_{\mathrm{cov}}$ is independent of $d$.
Set $h_-:=h(t_{\min})$, with
$0<h_-<h_*\le\bar h\le\min(h_0,1/2)$ and $h_*<h(T)<1$,
where $\bar h$ is sufficiently small depending only on the reduced parameters.
Train Stage 1 in $E$ at $h(t_1)=h_-^{3/2}$.
With this map frozen, use $n_{\mathrm I}$ shared clean/noise pairs and
regularization $\eta=h$ for the time-indexed ridge score $\widehat s_E$
of Appendix~\ref{app:full_sampling_construction}, and set
\begin{equation}
 s_\theta^{\mathrm{LN}}(x,t)
 =\widehat s_E(Qx,h(t))-\frac{(I-Q)x}{h(t)}.
 \label{eq:end2end_low_score}
\end{equation}

\paragraph{High-noise extension.}
We gate the two score estimators by $\alpha(t):=\mathbf{1}\{t\leq t_{\max}\}$:
\begin{equation}
 s_{\theta}^{\mathrm{full}}(x,t)
 =\alpha(t)s_{\theta}^{\mathrm{LN}}(x,t)
 +(1-\alpha(t))s^{\mathrm{HN}}(x,t).
 \label{eq:full_score}
\end{equation}
For Phase II, let $z_1,\ldots,z_{n_{\mathrm H}}\overset{\mathrm{iid}}{\sim}p_0$
be clean training samples independent of the two learning stages, and write
$a_t=\sqrt{1-h(t)}$. We use an analytic Gaussian component and sample attention:
\begin{equation}
 \begin{aligned}
 s^{\mathrm{HN}}(x,t)
 &=-\frac{x}{h(t)}+\frac{a_t}{h(t)}
     \sum_{i=1}^{n_{\mathrm H}}\omega_i(x,t)z_i,\\
 \omega_i(x,t)
 &\propto\exp\!\left(\frac{a_t x^\top z_i}{h(t)}
              -\frac{a_t^2\|z_i\|^2}{2h(t)}\right),
 \qquad\sum_{i=1}^{n_{\mathrm H}}\omega_i(x,t)=1.
 \end{aligned}
 \label{eq:hn_attention}
\end{equation}

This is the exact score of the VP Gaussian mixture induced by the empirical
distribution of the $z_i$. Assume $p_0$ is supported in $B(0,R)$ and
$\beta(t):=-\partial_t\log a_t^2>0$. Then
\begin{equation}
 \begin{gathered}
 \mathcal E_{\mathrm{HN}}
 :=\int_{t_{\max}}^T\!\beta(t)\,
 \mathbb E_{X\sim p_t}\|s^{\mathrm{HN}}(X,t)-s^*(X,t)\|^2\,dt,\\
 \mathbb E_{\mathrm{data}}\mathcal E_{\mathrm{HN}}
 \leq \frac{C_{\mathrm{HN}}}{n_{\mathrm H}h_*^{1+k/2}},
 \quad C_{\mathrm{HN}}:=\frac{32R^2e^{9/2}A_{\mathrm{cov}}}{1+k/2}.
 \end{gathered}
 \label{eq:hn_integrated}
\end{equation}
In particular, this high-noise term has no explicit factor of $d$; a proof
is given in Appendix~\ref{app:hn_attention}.

Let $p_{\mathrm{gen},t_{\min}}$ be the law obtained by running the reverse
SDE with~\eqref{eq:full_score} from $\mathcal N(0,h(T)I_d)$ to $t_{\min}$,
and define $p_{\mathrm{gen}}:=D_\#p_{\mathrm{gen},t_{\min}}$, where
\begin{equation}
 D(x)=\operatorname{Proj}_{B_E(0,R)}
       \left(\frac{Qx}{\sqrt{1-h_-}}\right).
 \label{eq:end2end_terminal}
\end{equation}
Assume the reverse SDEs are well posed and the Girsanov change of
measure is valid.

\begin{theorem}[End-to-end sampling guarantee]
\label{thm:full_sampling}
Under the preceding assumptions, there exists a constant $C\ge1$
such that the following holds for any $\delta\in(0,1)$ and
$n\ge2$ with $n\delta/\log n$ sufficiently large.
Take $n_{\mathrm I}=n_{\mathrm H}=n$ and choose
$h_*=(\log n/(n\delta))^{2/(k+6)}$, $h_-=h_*/2$, and
$a_T^2R^2/h(T)\le h_*^2$.
Set $L=\log(16n/h_-^3)$ and take
$m\ge CL/h_-$ and
$n\ge (C\kappa^2/h_-)\log(16mn/h_-^3)$.
Then, for each fixed, sufficiently large Stage-1 training time $s$,
conditionally on Stage 1, with probability at least $1-\delta$
over the frozen features and subsequent samples,
\begin{equation}
 W_2(p_0,p_{\mathrm{gen}})
 \le C\left(\frac{\log n}{n\delta}\right)^{1/(k+6)}.
 \label{eq:full_sampling_rate}
\end{equation}
\end{theorem}

The proof is given in Appendix~\ref{app:full_sampling_proof}.
Theorem~\ref{thm:full_sampling} connects the two-stage learning guarantees to convergence of the final generated distribution in Wasserstein distance. The low-noise estimator combines learned geometric alignment with residual score estimation, while the high-noise estimator extends score control beyond the regime where the manifold expansion applies. Terminal denoising controls the remaining perturbation at the stopping time. Together, these components yield an end-to-end sampling bound, with the prescribed noise levels ensuring that their contributions to the final error vanish as the sample size increases.

\section{Experiments}
\label{sec:experiment}

\begin{algorithm}[h]
\caption{Practical Score-induced Latent Diffusion (SiLD)}
\label{alg:practical-sild}

\begin{algorithmic}[1]

\Require Dataset $\mathcal D$; deterministic encoder
$E_\phi:\mathbb R^d\to\mathbb R^{d_z}$ and decoder $G_\psi$;
latent diffusion network $\epsilon_\theta$;
low-noise and latent time-sampling distributions
$q_{\mathrm{low}}$ and $q_{\mathrm{lat}}$;
latent noise schedule $\{\bar\alpha_t^{\mathrm{lat}}\}$
with $\bar\alpha_T^{\mathrm{lat}}\approx 0$

\Statex \textbf{Stage 1: Low-noise geometric representation learning}

\For{each minibatch $x_0 \sim \mathcal D$}
    \State Sample $t_1 \sim q_{\mathrm{low}}$ and
    $\epsilon \sim \mathcal N(0,I_d)$

    \State
    $x_{t_1} \gets
    \sqrt{\bar\alpha_{t_1}}x_0
    +
    \sqrt{1-\bar\alpha_{t_1}}\epsilon$

    \State
    $\hat{x}_0 \gets G_\psi(E_\phi(x_{t_1}))$

    \State
    $\mathcal L_{\mathrm{stage1}}
    \gets
    \lVert \hat{x}_0-x_0\rVert_2^2$

    \State Update $(\phi,\psi)$ using
    $\mathcal L_{\mathrm{stage1}}$
\EndFor

\State Freeze $E_\phi$ and $G_\psi$

\Statex \textbf{Stage 2: Density learning in the learned coordinates}

\For{each minibatch $x_0 \sim \mathcal D$}
    \State $z_0 \gets E_\phi(x_0)$

    \State Sample $t_2 \sim q_{\mathrm{lat}}$ and
    $\xi \sim \mathcal N(0,I_{d_z})$

    \State
    $z_{t_2}
    \gets
    \sqrt{\bar\alpha_{t_2}^{\mathrm{lat}}}z_0
    +
    \sqrt{1-\bar\alpha_{t_2}^{\mathrm{lat}}}\xi$

    \State
    $\mathcal L_{\mathrm{stage2}}
    \gets
    \lVert
    \epsilon_\theta(z_{t_2},t_2)-\xi
    \rVert_2^2$

    \State Update $\theta$ using $\mathcal L_{\mathrm{stage2}}$
\EndFor

\Statex \textbf{Sampling}

\State Sample $z_T \sim \mathcal N(0,I_{d_z})$
\State Run the standard latent reverse diffusion process to obtain $z_0$
\State \Return $G_\psi(z_0)$

\end{algorithmic}
\end{algorithm}

Algorithm~\ref{alg:practical-sild} is a theory-inspired
implementation; the guarantees above concern the analyzed two-stage
procedure. We evaluate practical SiLD against architecture-matched
LDM~\cite{rombach2022high} baselines on image and molecular benchmarks.
For images, we report FID~\cite{heusel2017gans} and reconstruction MSE;
for molecules, validity, uniqueness, novelty, internal diversity (IntDiv),
QED~\cite{bickerton2012quantifying}, and FCD~\cite{preuer2018frechet}.

\begin{figure}[t]
\centering
\includegraphics[width=\textwidth]{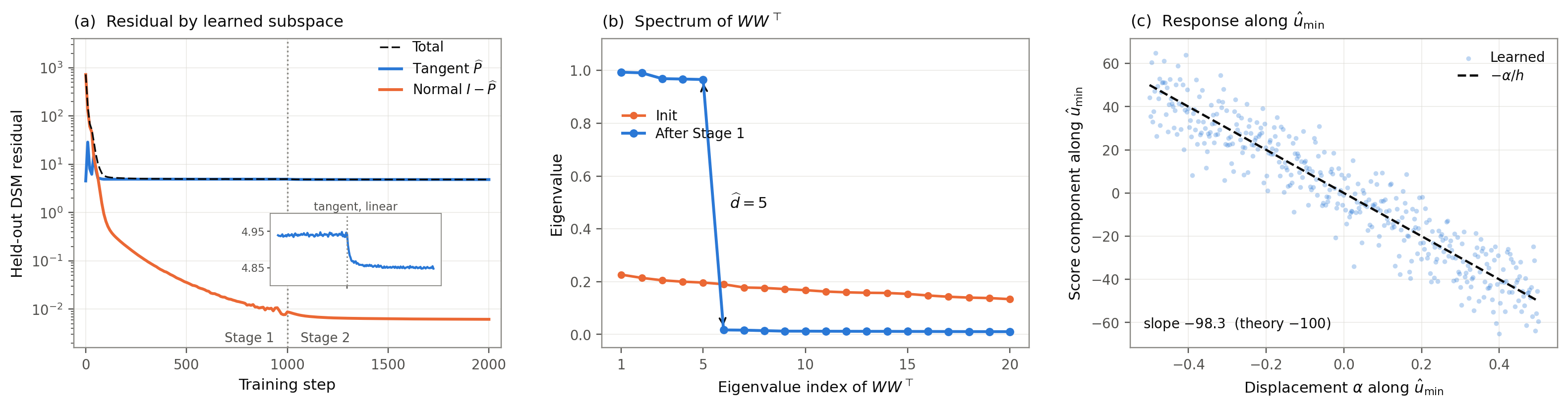}
\caption{%
Two-stage training on a synthetic mixture supported on a $k=5$ subspace of
$\mathbb R^{100}$ ($\sigma_t=0.1$, $h_t=\sigma_t^2=10^{-2}$, width $m=200$,
$1000$ steps per stage). All panels use only model- and data-side quantities;
the ground-truth subspace is used to synthesize data, never to compute a plotted diagnostic. \textbf{(a)} Held-out DSM residual split by the projector
$\widehat P$ onto the model's own top-$\hat d$ eigenvectors. The normal part
falls by five orders of magnitude during Stage~1 and stays flat in Stage~2,
while the tangent part carries essentially the remaining DSM residual; the inset
shows the residual Stage-2 gain is confined to the tangent directions.
\textbf{(b)} Spectrum of $WW^\top$ before and after Stage~1: the eigenvalues
separate into $\hat d=5$ near-one directions and a near-zero tail, so the
largest-eigengap estimate recovers the intrinsic dimension without being told
$k$. \textbf{(c)} Learned linear response along the model's own inferred
normal direction $\hat u_{\min}$, with slope $-98.3$ against the predicted
$-1/h_t = -100$.%
}
\label{fig:toy-two-stage}
\end{figure}

\textbf{Synthetic validation.}
Figure~\ref{fig:toy-two-stage} isolates the mechanism behind the two-stage
procedure on a Gaussian mixture supported on a $5$-dimensional subspace of
$\mathbb R^{100}$, using a score network $s_\theta(x)=(\mathrm{ReLU}(xW+b)V^\top
+ xW/h_t)W^\top - x/h_t$ whose linear part is tied to the learned operator
$WW^\top$. Stage~1 trains $W$ with DSM plus an idempotence penalty
$\lVert PP-P\rVert_F^2$, $P=WW^\top$; Stage~2 freezes $W$ and fits only
$(V,b)$. After Stage~1 the operator is numerically a projector (relative
defect $0.043$) of estimated rank $\hat d=5$, matching the true intrinsic dimension,
and the response along the inferred normal directions matches the $-1/h_t$
prediction to within $1.7\%$. The held-out DSM residual splits as
$4.847$ tangent versus $6.2\times10^{-3}$ normal out of a total of
$4.854$: $99.87\%$ of the remaining DSM residual lies in the inferred
tangent directions. Together with the Stage-2 trajectory, this is
consistent with Stage~2 primarily refining the tangent component while
preserving the normal response learned in Stage~1.
As a post-hoc check that no panel depends on,
$\lVert WW^\top - P_{\mathrm{true}}\rVert_F = 0.096$.

\textbf{Stacked MNIST.}
We evaluate on Stacked MNIST, a 1000-mode benchmark that stacks three random MNIST digits into the R/G/B channels ($d=2352$). SiLD, its $\sigma_{\mathrm{geom}}\!=\!0$ member, and LDM-CNN use matched CNN encoder--decoder backbones and the same latent score architecture, training budget, and Stage-2 interval $t_2\in[0.001,0.999]$. The $\sigma_{\mathrm{geom}}\!=\!0$ member is a clean deterministic autoencoder without KL regularization, whereas SiLD uses additive Stage-1 corruption $x_{\mathrm{noisy}}=x_0+\sigma_{\mathrm{geom}}\epsilon$, with $\sigma_{\mathrm{geom}}=0.005$ and $\epsilon\sim\mathcal N(0,I_d)$. For this experiment, additive corruption replaces the VP form in Algorithm~\ref{alg:practical-sild}; the Stage-1 noise scale is independent of the Stage-2 time interval.

All three methods achieve full $1000/1000$ mode coverage
(Table~\ref{tab:stacked_mnist}). At matched capacity, SiLD reduces FID
by approximately $50\%$ relative to LDM-CNN
($7.95 \pm 0.26$ vs.\ $16.15 \pm 0.10$, mean$\pm$population std over
3 seeds) and achieves $1.77\times$ lower reconstruction MSE.
SiLD also achieves a lower mean FID than its
$\sigma_{\mathrm{geom}}\!=\!0$ clean-autoencoder variant
($8.20 \pm 0.14$), although this smaller difference is not statistically
significant across the three paired seeds
(Appendix~\ref{app:stage1-noise}).
These results demonstrate improved generation quality and reconstruction
accuracy over the architecture-matched LDM-CNN baseline, while preserving
full mode coverage.

\begin{table}[t]
\centering
\caption{%
Stacked MNIST (1000-mode benchmark). All three models use matched CNN backbones and the same Stage-2 interval $t_2\in[0.001,0.999]$, with different Stage-1 representation-learning objectives. The clean-AE row is the $\sigma_{\mathrm{geom}}\!=\!0$ member of the SiLD family, a clean deterministic autoencoder with no input corruption and no KL term; its difference from SiLD is not significant (paired per-seed $t=-0.96$, $n=3$). Results are mean$\pm$population std over 3 seeds (42, 43, and 44); all methods achieve full 1000/1000 mode coverage. Bold marks the two deterministic variants. Appendix~\ref{app:stage1-noise} sweeps $\sigma_{\mathrm{geom}}$ over three levels.%
}
\small
\begin{tabular}{lccc}
\toprule
Stage-1 objective & \#modes / 1000 & Recon MSE $\downarrow$ & FID $\downarrow$ \\
\midrule
Real data           & 1000 & --                                & --                       \\
\midrule
LDM-CNN ($\beta$-VAE AE) & 1000 & $(6.28 \pm 0.05)\!\times\!10^{-3}$ & $16.15 \pm 0.10$ \\
\textbf{SiLD} ($\sigma_{\mathrm{geom}}\!=\!0$; clean AE) & 1000 & $\mathbf{(3.54 \pm 0.04)\!\times\!10^{-3}}$ & $\mathbf{8.20 \pm 0.14}$ \\
\textbf{SiLD (ours)}     &  {1000} & $\mathbf{(3.54 \pm 0.02)\!\times\!10^{-3}}$ & $\mathbf{7.95 \pm 0.26}$ \\
\bottomrule
\end{tabular}
\label{tab:stacked_mnist}
\end{table}

\begin{table*}[t]
\centering
\caption{Unconditional molecule generation on four MoleculeNet benchmarks
(10K samples per method). All methods use SELFIES; validity is checked
with RDKit and FCD against held-out test data. Stage-2 sampling uses the
paired interval $t_2\in[t_0,0.999]$ with $t_0=0.05$; the Stage-1 objective
family, including its $\sigma\!\to\!0$ member and reconstruction error, is
compared at a common $t_2\in[0.001,0.999]$ in
Table~\ref{tab:stage1_family_mol}. Bold indicates the best
generative result for each dataset and metric.}
\label{tab:main_results}
\resizebox{\textwidth}{!}{%
\begin{tabular}{lllcccccc}
\toprule
\textbf{Dataset} & \textbf{Model} & \textbf{Type}
& \textbf{Validity$\uparrow$} & \textbf{Uniqueness$\uparrow$}
& \textbf{Novelty$\uparrow$} & \textbf{Int.\ Div.$\uparrow$}
& \textbf{QED$\uparrow$} & \textbf{FCD$\downarrow$} \\
\midrule
\multirow{7}{*}{\textbf{QM9}}
 & Real data      & ---              & 100.0\% & 96.18\% & 0.0\%   & 0.919 & 0.466 & 5.05 \\
\cmidrule{2-9}
 & CharRNN        & RNN              & \textbf{100.0\%} & 96.32\% & 1.05\%  & 0.917 & 0.465 & 5.13 \\
 & Transformer AR & Transformer      & \textbf{100.0\%} & 96.46\% & 5.26\%  & \textbf{0.920} & 0.465 & 4.98 \\
\cmidrule{2-9}
 & VAE-only       & VAE prior        & \textbf{100.0\%} & 84.43\% & \textbf{88.93\%} & 0.900 & 0.439 & 7.43 \\
 & Latent GAN     & GAN in latent    & \textbf{100.0\%} & 72.63\% & 74.58\% & 0.892 & \textbf{0.474} & 7.80 \\
 & LDM            & Latent diffusion & \textbf{100.0\%} & 96.93\% & 37.6\%  & 0.915 & 0.472 & \textbf{4.28} \\
 & \textbf{SiLD (ours)} & Latent diffusion & \textbf{100.0\%} & \textbf{97.44\%} & 35.9\% & 0.918 & 0.471 & 4.41 \\
\midrule
\multirow{7}{*}{\textbf{HIV}}
 & Real data      & ---              & 100.0\% & 88.93\% & 0.0\%   & 0.898 & 0.513 & 3.24 \\
\cmidrule{2-9}
 & CharRNN        & RNN              & 99.97\% & 88.84\% & 4.39\%  & 0.899 & \textbf{0.514} & 3.24 \\
 & Transformer AR & Transformer      & 99.93\% & 95.16\% & 60.29\% & 0.899 & 0.506 & \textbf{2.91} \\
\cmidrule{2-9}
 & VAE-only       & VAE prior        & \textbf{100.0\%} & 94.18\% & 99.96\% & \textbf{0.912} & 0.388 & 17.66 \\
 & Latent GAN     & GAN in latent    & \textbf{100.0\%} & 95.26\% & 99.99\% & 0.899 & 0.365 & 19.83 \\
 & LDM            & Latent diffusion & \textbf{100.0\%} & 0.96\%  & \textbf{100.0\%} & 0.014 & 0.114 & 60.91 \\
 & \textbf{SiLD (ours)} & Latent diffusion & \textbf{100.0\%} & \textbf{98.74\%} & 99.9\% & 0.904 & 0.430 & 10.43 \\
\midrule
\multirow{7}{*}{\textbf{MUV}}
 & Real data      & ---              & 100.0\% & 95.03\% & 0.0\%   & 0.860 & 0.675 & 4.52 \\
\cmidrule{2-9}
 & CharRNN        & RNN              & \textbf{100.0\%} & 95.12\% & 8.17\%  & 0.862 & \textbf{0.676} & 4.36 \\
 & Transformer AR & Transformer      & \textbf{100.0\%} & 99.33\% & 83.08\% & 0.866 & 0.668 & \textbf{4.06} \\
\cmidrule{2-9}
 & VAE-only       & VAE prior        & \textbf{100.0\%} & 94.98\% & \textbf{100.0\%} & \textbf{0.919} & 0.364 & 33.48 \\
 & Latent GAN     & GAN in latent    & \textbf{100.0\%} & 96.06\% & \textbf{100.0\%} & 0.904 & 0.379 & 34.88 \\
 & LDM            & Latent diffusion & \textbf{100.0\%} & 95.60\% & \textbf{100.0\%} & 0.811 & 0.237 & 41.20 \\
 & \textbf{SiLD (ours)} & Latent diffusion & \textbf{100.0\%} & \textbf{99.42\%} & \textbf{100.0\%} & 0.890 & 0.518 & 19.48 \\
\midrule
\multirow{7}{*}{\textbf{PCBA}}
 & Real data      & ---              & 100.0\% & 98.91\% & 0.0\%   & 0.861 & 0.637 & 9.12 \\
\cmidrule{2-9}
 & CharRNN        & RNN              & \textbf{100.0\%} & 99.53\% & 64.8\%  & 0.864 & \textbf{0.635} & 8.08 \\
 & Transformer AR & Transformer      & \textbf{100.0\%} & \textbf{99.94\%} & 95.12\% & 0.869 & 0.624 & \textbf{7.38} \\
\cmidrule{2-9}
 & VAE-only       & VAE prior        & \textbf{100.0\%} & 91.65\% & 99.74\% & \textbf{0.915} & 0.389 & 28.62 \\
 & Latent GAN     & GAN in latent    & \textbf{100.0\%} & 94.93\% & \textbf{99.82\%} & 0.898 & 0.357 & 30.41 \\
 & LDM            & Latent diffusion & \textbf{100.0\%} & 0.57\%  & 94.7\%  & 0.016 & 0.090 & 77.10 \\
 & \textbf{SiLD (ours)} & Latent diffusion & \textbf{100.0\%} & 98.57\% & 99.8\% & 0.895 & 0.491 & 17.79 \\
\bottomrule
\end{tabular}%
}
\end{table*}

\textbf{Molecular generation.}
We evaluate on four MoleculeNet benchmarks (QM9, MUV, HIV, PCBA) using the SELFIES representation~\cite{krenn2020self}, which guarantees syntactic validity by construction and thus isolates \emph{distributional} learning as the object of comparison. SiLD and LDM share identical MLP encoder, decoder, and latent score network ($\sim$13M parameters); only the Stage-1 objective differs. We additionally compare against CharRNN, a 2-layer LSTM at matching scale, and report prior string-based methods on QM9 for external reference. Additional baselines are
CharRNN~\citep{segler2018generating},
Transformer AR~\citep{wang2023cmolgpt,noutahi2024gotta},
VAE-only~\citep{gomez2018automatic}, and
Latent GAN~\citep{prykhodko2019novo}.

Table~\ref{tab:main_results} shows mixed results on QM9: SiLD slightly
improves uniqueness and IntDiv over LDM, but has a higher FCD. On HIV,
MUV, and PCBA, SiLD improves uniqueness, IntDiv, and FCD over LDM while
maintaining at least $99.8\%$ novelty. The gap is most pronounced on HIV
and PCBA, where LDM's uniqueness falls below $1\%$; on PCBA, SiLD reduces
FCD from $77.10$ to $17.79$. Transformer AR nevertheless achieves the
lowest FCD on HIV, MUV, and PCBA. The $\sigma=0$ member of the SiLD family yields comparable results to
SiLD on all four datasets in the single-seed ablation
(Appendix~\ref{app:stage1-family-mol}). Both variants improve uniqueness,
IntDiv, and FCD over the tested KL-regularized baseline on HIV, MUV,
and PCBA. Pipeline diagnostics further show that collapse is already
present in the KL-regularized autoencoder, before latent diffusion:
encoding real HIV molecules and decoding them returns only $0.95\%$
unique outputs, and $99.6\%$ of the raw latent coordinates have standard
deviation below $10^{-3}$, compared with none for SiLD
(Appendix~\ref{app:collapse}).

\section{Conclusion}
\label{sec:conclusion}

We analyzed a two-stage score-learning procedure under the manifold
hypothesis. In Stage 1, annealed mean-field training learns an approximate
manifold projection at small noise. In Stage 2, residual score estimation
on the recovered representation admits generalization bounds governed by
the intrinsic dimension under suitable assumptions. Together with an
analytic high-noise estimator, the analysis yields a conditional
end-to-end sampling guarantee under the stated subspace assumption.
Motivated by these results, we developed the practical Score-induced
Latent Diffusion (SiLD) algorithm, combining low-noise denoising
representation learning with latent diffusion. Compared with
architecture-matched LDM baselines, SiLD improves generation quality on
Stacked MNIST and improves diversity and FCD on three of the four
molecular benchmarks.

\section*{Acknowledgments}

We thank Atsushi Nitanda and Guoji Fu for the useful discussion. WH was supported by JSPS KAKENHI (24K20848) and JST BOOST (JPMJBY24G6). TS was partially supported by JSPS KAKENHI (24K02905) and JST CREST (JPMJCR2015). This research is supported by the National Research Foundation, Singapore and the Ministry of Digital Development and Information under the AI Visiting Professorship Programme (award number AIVP-2024-004). Any opinions, findings and conclusions or recommendations expressed in this material are those of the author(s) and do not reflect the views of National Research Foundation, Singapore and the Ministry of Digital Development and Information. KF was supported by JSPS KAKENHI Grant-in-Aid for Transformative Research Areas (A) 22H05106.

\bibliography{reference}
\bibliographystyle{plainnat}

\clearpage
\appendix

\section{Proofs of Main Results}
\label{app:proofs}

We collect the proofs of all theoretical results in this appendix, organized in the order they appear in the main text.

\subsection{Proofs of Proposition \ref{prop:score_decomp}}
\label{sec:proof_pro}

\begin{proof}[Proof of Proposition~\ref{prop:score_decomp}]
We reduce the VP marginal to an ambient Gaussian convolution
of the data measure on $\mathcal{M}$, establish a uniform
Laplace approximation with first-derivative control, and
return to the original coordinates.
Throughout the proof, write
\[
h:=h(t),\qquad
a:=\sqrt{1-h},\qquad
\lambda:=\frac{h}{a^2}.
\]
For a function defined on a neighborhood of a compact set $K$,
we use
\[
\|f\|_{C^1(K)}
:=
\sup_{x\in K}|f(x)|
+
\sup_{x\in K}\|\nabla f(x)\|.
\]
Constants below may depend on $\mathcal{M}$ and $p_0$, but
are independent of the spatial variable and of sufficiently
small $h$.

Under the VP forward process,
$x_t\mid x_0=z\sim\mathcal{N}(az,hI_d)$.
Consequently,
\begin{equation}
\label{eq:pt-marginal}
p_t(x)
=
(2\pi h)^{-d/2}
\int_{\mathcal{M}}
p_0(z)
\exp\left(-\frac{\|x-az\|^2}{2h}\right)
\,d\mathrm{Vol}_{\mathcal{M}}(z).
\end{equation}

Set $y:=x/a$ and define
\[
q_\lambda(y)
:=
(2\pi\lambda)^{-d/2}
\int_{\mathcal{M}}
p_0(z)
\exp\left(-\frac{\|y-z\|^2}{2\lambda}\right)
\,d\mathrm{Vol}_{\mathcal{M}}(z).
\]
Since $\|x-az\|^2=a^2\|y-z\|^2$, we have
\[
p_t(x)=a^{-d}q_\lambda(x/a),
\]
and hence
\begin{equation}
\label{eq:log-pt-rescale}
\log p_t(x)
=
-d\log a+\log q_\lambda(x/a).
\end{equation}

Let
\[
K:=\overline{U_{\tau/2}},
\qquad
K':=\overline{U_{3\tau/4}},
\qquad
L:=\sup_{x\in K}\|x\|.
\]
Both sets are compact and contained in $U_\tau$.
Choose $h_0\in(0,1/2)$ sufficiently small that
\[
\left((1-h_0)^{-1/2}-1\right)L<\tau/4.
\]
Then, for $x\in K$ and $0<h\leq h_0$,
\[
d_{\mathcal{M}}(x/a)
\leq
d_{\mathcal{M}}(x)+\|x/a-x\|
<
3\tau/4.
\]
The same bound holds along the segment joining $x$ to $x/a$.
Thus all subsequent spatial expansions take place inside $K'$.
We may decrease $h_0$ further when necessary.

For $y\in K'$, write
\[
z_y^*:=\Pi_{\mathcal{M}}(y),
\qquad
\nu_y:=y-z_y^*.
\]
Choose geodesic normal coordinates $u\in\mathbb{R}^k$
centered at $z_y^*$, with an orthonormal tangent frame
$\{e_1,\ldots,e_k\}$.
Locally,
\begin{align*}
z_y(u)
&=
z_y^*
+\sum_i e_i u^i
+\frac12\sum_{i,j}
\mathrm{II}_{z_y^*}(e_i,e_j)u^iu^j
+O(\|u\|^3),\\
\sqrt{\det g_y(u)}
&=
1-\frac16\mathrm{Ric}_{ij}(z_y^*)u^iu^j
+O(\|u\|^3).
\end{align*}
Using $\nu_y\perp T_{z_y^*}\mathcal{M}$ gives
\begin{equation}
\label{eq:dist-expansion}
\|y-z_y(u)\|^2
=
\|\nu_y\|^2
+
u^\top A_yu
+
O(\|u\|^3),
\qquad
A_y:=I_k-B_{\nu_y},
\end{equation}
where
\[
(B_{\nu_y})_{ij}
:=
\langle\nu_y,\mathrm{II}_{z_y^*}(e_i,e_j)\rangle.
\]
The reach bound implies
\[
\|B_{\nu_y}\|_{\mathrm{op}}
\leq \frac{\|\nu_y\|}{\tau}
\leq \frac34,
\]
so
\[
\frac14 I_k\preceq A_y\preceq\frac74 I_k
\qquad\text{for all }y\in K'.
\]
The projection and the local geometric quantities are smooth
on $U_\tau$.
Smoothly varying coordinate charts can therefore be chosen
locally in $y$, and compactness gives a finite cover with
uniform bounds.

Define the nonnegative phase
\[
F_y(z)
:=
\frac12\left(
\|y-z\|^2-d_{\mathcal{M}}(y)^2
\right).
\]
Its unique minimum is attained at $z_y^*$.
We establish
\begin{equation}
\label{eq:laplace-c1}
\|Q_\lambda-A_0\|_{C^1(K')}
\leq C\lambda,
\end{equation}
where
\[
Q_\lambda(y)
:=
(2\pi\lambda)^{-k/2}
\int_{\mathcal{M}}
p_0(z)e^{-F_y(z)/\lambda}
\,d\mathrm{Vol}_{\mathcal{M}}(z),
\qquad
A_0(y)
:=
\frac{p_0(z_y^*)}{\sqrt{\det A_y}}.
\]

To prove~\eqref{eq:laplace-c1}, work on one of the finitely
many parameter neighborhoods described in Step~3.
In the local coordinates above, set
$\varphi_y(u):=F_y(z_y(u))$.
Then
\[
\varphi_y(0)=0,\qquad
D_u\varphi_y(0)=0,\qquad
D_u^2\varphi_y(0)=A_y.
\]
Taylor's formula gives the exact representation
\[
\varphi_y(u)
=
\frac12u^\top\widehat A_y(u)u,
\qquad
\widehat A_y(u)
:=
2\int_0^1(1-r)D_u^2\varphi_y(ru)\,dr.
\]
For a sufficiently small coordinate neighborhood,
$\widehat A_y(u)$ is uniformly positive definite.
The change of variables
\[
v=\widehat A_y(u)^{1/2}u
\]
is a smooth local diffeomorphism, depending smoothly on $y$,
and satisfies
\[
F_y(z_y(u_y(v)))=\frac12\|v\|^2.
\]

Choose a smooth cutoff that equals one near $z_y^*$ and is
supported inside this coordinate neighborhood.
After the change of variables, its contribution to
$Q_\lambda$ is
\[
(2\pi\lambda)^{-k/2}
\int_{\mathbb{R}^k}
e^{-\|v\|^2/(2\lambda)}b(y,v)\,dv,
\]
where the amplitude $b$ includes $p_0$, the volume Jacobian,
and the cutoff, and is extended by zero outside its support.
The support can be chosen uniformly compact on the parameter
neighborhood.
At $v=0$, the Jacobian equals $\det(A_y)^{-1/2}$, so
\[
b(y,0)=A_0(y).
\]

Since $p_0\in C^4(\mathcal{M})$ and the geometry is smooth,
$b$ and each $\partial_{y_j}b$ have uniformly bounded second
derivatives with respect to $v$.
For either $c=b$ or $c=\partial_{y_j}b$, Taylor's formula
therefore yields
\[
c(y,v)
=
c(y,0)+D_vc(y,0)v+r_c(y,v),
\qquad
|r_c(y,v)|\leq C\|v\|^2.
\]
Integrating against the centered Gaussian with covariance
$\lambda I_k$, the linear term vanishes and the remainder
is bounded by $C\lambda$.
Differentiation with respect to $y$ under the integral is
justified by the uniform amplitude bounds.

Outside the region where the cutoff equals one, the phase
$F_y$ has a uniform positive lower bound.
Indeed, its minimum is unique, and the relevant parameter
sets and $\mathcal{M}$ are compact.
The complementary integral and its first $y$-derivatives,
after normalization by $(2\pi\lambda)^{-k/2}$, are bounded by
\[
C\lambda^{-k/2}(1+\lambda^{-1})e^{-c/\lambda}
\]
for some $c>0$.
They are therefore $O(\lambda)$.
Taking the maximum of the constants over the finite
parameter cover proves~\eqref{eq:laplace-c1}.

The function $A_0$ has a positive lower bound and bounded
first derivatives on $K'$.
Consequently,
\begin{equation}
\label{eq:laplace-q}
q_\lambda(y)
=
(2\pi\lambda)^{-(d-k)/2}
p_0(z_y^*)
\frac{e^{-d_{\mathcal{M}}(y)^2/(2\lambda)}}
{\sqrt{\det(I_k-B_{\nu_y})}}
\bigl(1+E_\lambda(y)\bigr),
\end{equation}
where
\[
\|E_\lambda\|_{C^1(K')}\leq C\lambda.
\]
For sufficiently small $\lambda$, $\|E_\lambda\|_\infty\leq1/2$.
Thus
\[
e_\lambda(y):=\log(1+E_\lambda(y))
\]
also satisfies
\[
\|e_\lambda\|_{C^1(K')}\leq C\lambda.
\]

Combining~\eqref{eq:log-pt-rescale}
and~\eqref{eq:laplace-q} gives
\begin{equation}
\label{eq:logpt-raw}
\begin{aligned}
\log p_t(x)
={}&
-\frac{d_{\mathcal{M}}(x/a)^2}{2\lambda}
+\log p_0\bigl(\Pi_{\mathcal{M}}(x/a)\bigr)\\
&-\frac12\log\det(I_k-B_{\nu_{x/a}})
+C(t)+e_\lambda(x/a),
\end{aligned}
\end{equation}
where
\[
C(t)
=
-d\log a-\frac{d-k}{2}\log(2\pi\lambda)
\]
is independent of $x$.
Because $a$ is bounded away from zero and $\lambda=O(h)$,
the chain rule gives
\[
\|e_\lambda(\,\cdot\,/a)\|_{C^1(K)}\leq Ch.
\]

We now expand the remaining terms in the $C^1(K)$ norm.
Since
\[
a^{-1}-1=\frac h2+O(h^2)
\]
and $d_{\mathcal{M}}^2$ is smooth on $U_\tau$,
Taylor's formula and Federer's identity give
\[
d_{\mathcal{M}}(x/a)^2
=
d_{\mathcal{M}}(x)^2
+h\langle\nu_x,x\rangle
+O_{C^1(K)}(h^2),
\qquad
\nu_x:=x-\Pi_{\mathcal{M}}(x).
\]
Here $O_{C^1(K)}(h^j)$ denotes a remainder with
$C^1(K)$ norm bounded by a constant times $h^j$.
Using the exact identity
\[
\frac1\lambda=\frac{1-h}{h}=\frac1h-1,
\]
we obtain
\begin{align*}
-\frac{d_{\mathcal{M}}(x/a)^2}{2\lambda}
={}&
-\frac{d_{\mathcal{M}}(x)^2}{2h}
+\frac12d_{\mathcal{M}}(x)^2
-\frac12\langle\nu_x,x\rangle
+O_{C^1(K)}(h)\\
={}&
-\frac{d_{\mathcal{M}}(x)^2}{2h}
-\frac12\langle\nu_x,\Pi_{\mathcal{M}}(x)\rangle
+O_{C^1(K)}(h).
\end{align*}
The last equality uses
$x=\Pi_{\mathcal{M}}(x)+\nu_x$ and
$d_{\mathcal{M}}(x)^2=\|\nu_x\|^2$.

The function
\[
F(x)
:=
\log p_0\bigl(\Pi_{\mathcal{M}}(x)\bigr)
-\frac12\log\det(I_k-B_{\nu_x})
\]
has bounded derivatives through order two on $K'$.
Since the segment joining $x$ to $x/a$ stays in $K'$,
\[
F(x/a)=F(x)+O_{C^1(K)}(h).
\]
Substituting these expansions into~\eqref{eq:logpt-raw}
yields
\begin{equation}
\label{eq:logpt-final}
\log p_t(x)
=
-\frac{d_{\mathcal{M}}(x)^2}{2h}
+\log p_0\bigl(\Pi_{\mathcal{M}}(x)\bigr)
+H\bigl(\Pi_{\mathcal{M}}(x),\nu_x\bigr)
+C(t)+\rho_h(x),
\end{equation}
with
\[
\|\rho_h\|_{C^1(K)}\leq Ch,
\]
where
\begin{equation}
\label{eq:H-def}
H(z,\nu)
:=
-\frac12\langle\nu,z\rangle
-\frac12\log\det(I_k-B_\nu),
\qquad
z\in\mathcal{M},\quad
\nu\in N_z\mathcal{M},\quad
\|\nu\|<\tau.
\end{equation}
The matrix $I_k-B_\nu$ is positive definite on this domain,
so $H$ is smooth.
Its two terms are the VP-shrinkage and extrinsic-curvature
corrections, respectively.

The $C^1$ remainder bound justifies differentiating
\eqref{eq:logpt-final}.
Using
$\nabla_x[\tfrac12d_{\mathcal{M}}(x)^2]
=x-\Pi_{\mathcal{M}}(x)$ gives
\[
s^*(x,t)
=
-\frac{x-\Pi_{\mathcal{M}}(x)}{h}
+\nabla_x\log p_0\bigl(\Pi_{\mathcal{M}}(x)\bigr)
+\nabla_xH\bigl(\Pi_{\mathcal{M}}(x),\nu_x\bigr)
+\nabla_x\rho_h(x),
\]
where
\[
\sup_{x\in K}\|\nabla_x\rho_h(x)\|\leq Ch.
\]
The density and geometric correction terms are uniformly
bounded on $K$, by positivity and regularity of $p_0$,
smoothness of the projection, and compactness.
This proves the stated decomposition and its uniform
remainder bound.
\end{proof}

\subsection{Proof of Lemma \ref{lem:structural-alignment}}

\begin{proof}[Proof of Lemma~\ref{lem:structural-alignment}]
\emph{Gradient identity.}
Since $\Psi'(z)=\sigma(z)$, the chain rule gives
$\nabla_x\Psi(w_j^\top x+b_j)
=\sigma(w_j^\top x+b_j)\,w_j$.
Differentiating $\Phi_{\mathrm{net}}$ termwise,
\[
\begin{aligned}
\nabla_x\Phi_{\mathrm{net}}(x;\theta)
&=
\frac{x}{h(t_1)}
-
\frac{1}{h(t_1)}
\sum_{j=1}^m a_j\,\sigma(w_j^\top x+b_j)\,w_j\\
&=
\frac{x}{h(t_1)}
-
\frac{1}{h(t_1)}
W\,\mathrm{diag}(a)\,\sigma(W^\top x+b)
=
-f_1(x;\theta),
\end{aligned}
\]
which establishes the identity on all of $\mathbb R^d$.

\smallskip
\emph{Projection approximation.}
Define the target potential
$\Phi^*(x,t_1):=d_{\mathcal M}(x)^2/(2h(t_1))$
and the smooth part
\[
G^*(x)
:=
\tfrac12\|x\|^2-\tfrac12d_{\mathcal M}(x)^2
=
\langle x,\Pi_{\mathcal M}(x)\rangle
-
\tfrac12\|\Pi_{\mathcal M}(x)\|^2,
\qquad x\in U_\tau,
\]
where the second equality follows by expanding
\[
d_{\mathcal M}(x)^2
=
\|x-\Pi_{\mathcal M}(x)\|^2
=
\|x\|^2
-
2\langle x,\Pi_{\mathcal M}(x)\rangle
+
\|\Pi_{\mathcal M}(x)\|^2.
\]
Correspondingly, define the ridge sum
\[
G_\theta(x)
:=
\sum_{j=1}^m a_j\,\Psi(w_j^\top x+b_j),
\]
so that
\[
\Phi_{\mathrm{net}}(x;\theta)
=
\frac{\|x\|^2}{2h(t_1)}
-
\frac{G_\theta(x)}{h(t_1)},
\qquad
\Phi^*(x,t_1)
=
\frac{\|x\|^2}{2h(t_1)}
-
\frac{G^*(x)}{h(t_1)}.
\]
Taking gradients,
\begin{equation}
\label{eq:grad-identities}
\begin{aligned}
\nabla_xG_\theta(x)
&=
\sum_{j=1}^m a_j\,\sigma(w_j^\top x+b_j)\,w_j
=
W\,\mathrm{diag}(a)\,\sigma(W^\top x+b),\\
\nabla_xG^*(x)
&=
\Pi_{\mathcal M}(x),
\end{aligned}
\end{equation}
where the second identity follows from Federer's formula
$\nabla_x[\tfrac12d_{\mathcal M}(x)^2]
=x-\Pi_{\mathcal M}(x)$
on $U_\tau$~\cite{federer1959curvature}.
Therefore,
\[
W\,\mathrm{diag}(a)\,\sigma(W^\top x+b)
-
\Pi_{\mathcal M}(x)
=
\nabla_xG_\theta(x)-\nabla_xG^*(x),
\]
and \eqref{eq:approx-target} reduces to approximating $G^*$
by ridge sums $G_\theta$ in the $C^1$-norm on
$K:=\overline{U_{\tau/2}}$, where
\[
\|g\|_{C^1(K)}
:=
\sup_{x\in K}|g(x)|
+
\sup_{x\in K}\|\nabla_xg(x)\|.
\]

The set $K$ is compact since $\mathcal M$ is compact.
Under the standing smoothness assumption on $\mathcal M$,
the nearest-point projection $\Pi_{\mathcal M}$ is smooth
on $U_\tau$, and hence $G^*\in C^\infty(U_\tau)$.
Since $K\subset U_\tau$, choose
$\chi\in C_c^\infty(U_\tau)$ equal to one on a neighborhood
of $K$. Define
\[
\widetilde G^*(x)
:=
\begin{cases}
\chi(x)G^*(x), & x\in U_\tau,\\
0, & x\notin U_\tau.
\end{cases}
\]
Then $\widetilde G^*\in C^\infty(\mathbb R^d)$ and
$\widetilde G^*=G^*$ on a neighborhood of $K$.

Since $\sigma\in C^2(\mathbb R)$ is non-polynomial,
its primitive $\Psi\in C^3(\mathbb R)$ is also
non-polynomial: if $\Psi$ were polynomial, then
$\sigma=\Psi'$ would be polynomial, contradicting
the hypothesis.
The conservative-form parametrization $\theta=(W,a,b)$
produces precisely the standard ridge dictionary
\[
\mathcal R(\Psi)
:=
\left\{
\sum_{j=1}^m a_j\,\Psi(w_j^\top x+b_j)
\,:\,
m\in\mathbb N,\;
a_j,b_j\in\mathbb R,\;
w_j\in\mathbb R^d
\right\}.
\]
By the derivative approximation theorem for ridge
functions~\cite[Thm.~4.1]{pinkus1999approximation},
for any $\delta>0$, there exist $m\in\mathbb N$ and
parameters
$\theta=(W,a,b)\in
\mathbb R^{d\times m}\times\mathbb R^m\times\mathbb R^m$
such that
\[
\sup_{x\in K}
\left|G_\theta(x)-\widetilde G^*(x)\right|
+
\sup_{x\in K}
\left\|\nabla_xG_\theta(x)
-\nabla_x\widetilde G^*(x)\right\|
<\delta.
\]
Since $\widetilde G^*=G^*$ on a neighborhood of $K$,
this gives
$\|G_\theta-G^*\|_{C^1(K)}<\delta$.
Choosing $\delta:=\varepsilon$, we obtain
\[
\begin{aligned}
&\sup_{x\in K}
\left\|
W\,\mathrm{diag}(a)\,\sigma(W^\top x+b)
-
\Pi_{\mathcal M}(x)
\right\|\\
&\qquad=
\sup_{x\in K}
\left\|\nabla_xG_\theta(x)-\nabla_xG^*(x)\right\|
\leq
\|G_\theta-G^*\|_{C^1(K)}
<
\varepsilon,
\end{aligned}
\]
which is \eqref{eq:approx-target}.
\end{proof}

\subsection{Proof of Theorem~\ref{thm:dimensional_collapse}}
\label{app:proof_dimensional_collapse}

The annealed regularization provides moment control directly.
Since $\lambda'(s)\leq0$ and $q_0=\pi$, the free energy satisfies
$\mathcal E_s(q_s)\leq J_t(q_0)$. Consequently,
\begin{equation}
\label{eq:annealed_moment_bound}
    \int(\|w\|^2+b^2)\,q_s(\mathrm d\theta)
    \leq
    4\sigma_0^2
    \left(
        \frac{J_t(q_0)}{\lambda(s)}
        +\frac{d+1}{2}\log 2
    \right).
\end{equation}

\begin{proof}[Proof of Theorem~\ref{thm:dimensional_collapse}]
Fix $t=t_1$ and $h=h(t_1)>0$, and write $p=p_t$, $D=D_t$,
$J=J_t$, and $n_\xi=d+1$. Set
\[
    Y=\sqrt{1-h}\,X_0,\qquad
    X=Y+\sqrt h\,Z,\qquad
    Z\sim\mathcal N(0,I_d),\qquad
    R_h=\sqrt{1-h}\,R,
\]
where $Z$ is independent of $X_0$ and
$R=\sup_{z\in\mathcal M}\|z\|$. Thus $\|Y\|\leq R_h$.
We use $\xi=(w,b)$ for the trainable coordinates and
$\theta=(w,a,b)$ for the full neuron parameter. The reference law
$\pi$ has independent marginals
\[
    a\sim\nu_a=\tfrac12(\delta_{-A}+\delta_A),
    \qquad
    \xi\sim\gamma:=\mathcal N(0,\sigma_0^2I_{n_\xi}).
\]
All parameter distributions have $a$-marginal $\nu_a$.
For such a distribution, write $q^a$ for its conditional law of
$\xi$ given $a$, and set
\[
    m_2(q):=\int\|\xi\|^2q(\mathrm d\theta),\qquad
    \mathcal Q:=
    \{q:q_a=\nu_a,\ \operatorname{KL}(q\|\pi)<\infty\}.
\]
Finite relative entropy with respect to $\pi$ implies a finite
second moment, as shown below. Throughout the proof
$\sigma(u)=\operatorname{erf}(u)$.

Differentiating the Gaussian convolution defining $p$ gives
\begin{equation}
    D(x)=x+h\nabla\log p(x)=\mathbb E[Y\mid X=x],
    \qquad \|D\|_\infty\leq R_h.
    \label{eq:app_denoiser_identity}
\end{equation}
In particular, $p$ is smooth and strictly positive, and $D$ is
the gradient of the smooth potential
\[
    V(x):=\tfrac12\|x\|^2+h\log p(x)-h\log p(0).
\]
Since $V(0)=0$ and $\|\nabla V\|_\infty\leq R_h$, we have
$|V(x)|\leq R_h\|x\|$.

Let $\mathcal H=L^2(p;\mathbb R^d)$ and let $\mathcal G$ be the
closure in $\mathcal H$ of
$\{\nabla\varphi:\varphi\in C_c^\infty(\mathbb R^d)\}$.
Then $D\in\mathcal G$. Indeed, choose a smooth cutoff $\chi_L$
equal to one on the ball of radius $L$, supported on the ball of
radius $2L$, with $\|\nabla\chi_L\|_\infty\leq C/L$.
The bounds on $D$ and $V$ imply
\[
    \|\nabla(\chi_LV)-D\|_{\mathcal H}^2
    \leq C R_h^2\,\mathbb P(\|X\|>L)\longrightarrow0.
\]

We next show that $D$ can be approximated by the present network
with the fixed signed output coefficients. For a unit vector
$v\in\mathbb R^d$ and $c\in\mathbb R$, define
\[
    \psi_{v,c}(x)=v\,\operatorname{sign}(v^\top x+c).
\]
Let $g\in\mathcal H$ be orthogonal to all these fields.
The vector measure $g(x)p(x)\,\mathrm dx$ is finite.
For each $v$, define the finite signed measure
\[
    \mu_v(B)
    :=\int_{\{v^\top x\in B\}}
       v^\top g(x)p(x)\,\mathrm dx.
\]
Orthogonality gives
$\int\operatorname{sign}(u+c)\,\mu_v(\mathrm du)=0$
for every $c$.
Taking $c\to\infty$ first shows that $\mu_v(\mathbb R)=0$;
varying $c$ then shows that $\mu_v$ vanishes on all half-lines,
hence $\mu_v=0$.
There are no atoms at the thresholds because
$p(x)g(x)\,\mathrm dx$ is absolutely continuous.
Consequently, for every $\omega\in\mathbb R^d$,
\[
    \omega^\top\widehat{pg}(\omega)=0.
\]
Thus $\nabla\cdot(pg)=0$ in the sense of distributions, so $g$
is orthogonal to $\mathcal G$ and in particular to $D$.
The Hilbert space orthogonal-complement characterization
therefore gives
\[
    D\in\overline{\operatorname{span}
    \{\psi_{v,c}:\|v\|=1,\ c\in\mathbb R\}}^{\mathcal H}.
\]

For any finite sum $\sum_{j=1}^N\beta_j\psi_{v_j,c_j}$,
dominated convergence gives
\[
    \sum_{j=1}^N\beta_jv_j
       \operatorname{erf}\bigl(r(v_j^\top x+c_j)\bigr)
    \longrightarrow
    \sum_{j=1}^N\beta_j\psi_{v_j,c_j}(x)
    \quad\text{in }\mathcal H
\]
as $r\to\infty$.
To realize the left-hand side as $P_q$, place joint mass
$|\beta_j|/(Ar)$ at
\[
    a=A\operatorname{sign}(\beta_j),\qquad
    w=rv_j,\qquad b=rc_j.
\]
Choose $r$ sufficiently large that the allocated mass at each
sign of $a$ is at most $1/2$, and place the remaining mass of
that sign at $w=0,b=0$, whose output is zero.
This gives a finite atomic distribution with precisely the
prescribed $a$-marginal.

Replacing each conditional atom in $\xi$ by a Gaussian with
arbitrarily small positive covariance preserves that marginal,
gives finite relative entropy with respect to $\pi$, and changes
the prediction by an arbitrarily small amount in $\mathcal H$.
For the last assertion, use the continuity of the neuron map,
the bound
$\|aw\sigma(w^\top x+b)\|\leq A\|w\|$,
and the finite second moment of $X$.
Relative entropy is finite because each conditional law is a
finite mixture of nondegenerate Gaussians.
We have proved
\begin{equation}
    \inf_{q\in\mathcal Q}J(q)
    =\inf_{q\in\mathcal Q}
       \tfrac12\|P_q-D\|_{\mathcal H}^2
    =0.
    \label{eq:app_zero_approximation_error}
\end{equation}
This argument establishes approximation for the fixed-$a$
parametrization; expressivity of the unrestricted finite network
alone would not imply~\eqref{eq:app_zero_approximation_error}.

Write $e_q=P_q-D$. Since $|\sigma|\leq1$,
\begin{equation}
    \|e_q\|_\infty
    \leq R_h+A\sqrt{m_2(q)}
    =:M(q).
    \label{eq:app_residual_sup_bound}
\end{equation}
A version of the first variation of $J$ is
\[
    \zeta_q(a,w,b)
    :=\frac{\delta J}{\delta q}(q)(w,a,b)
    =a\,\mathbb E\!\left[
        \langle e_q(X),w\rangle\sigma(w^\top X+b)
    \right].
\]
Its derivatives are
\begin{align}
    \nabla_w\zeta_q
    &=a\,\mathbb E\!\left[
        e_q(X)\sigma(w^\top X+b)
        +\langle e_q(X),w\rangle
         \sigma'(w^\top X+b)X
    \right],\nonumber\\
    \partial_b\zeta_q
    &=a\,\mathbb E\!\left[
        \langle e_q(X),w\rangle\sigma'(w^\top X+b)
    \right].
    \label{eq:app_first_variation_derivatives}
\end{align}

We record a consequence of Gaussian smoothing.
For $j=1,2$, let
\[
    G_{h,j}:=
    \sup_{\|v\|=1,\,u\in\mathbb R}
    \mathbb E\!\left[
        (1+\|X\|)^j\mid v^\top X=u
    \right]p_{v^\top X}(u).
\]
These constants are finite and depend only on $h,R_h,d$.
To see this, condition first on $Y=y$.
Writing $y_\perp=y-(v^\top y)v$, the weighted marginal density is
\[
    \mathbb E_Y\!\left[
      \varphi_h(u-v^\top Y)\,
      \mathbb E_{Z_\perp}
      (1+\|uv+Y_\perp+\sqrt h Z_\perp\|)^j
    \right],
\]
where $\varphi_h$ is the one-dimensional $\mathcal N(0,h)$
density and $Z_\perp$ is standard Gaussian on $v^\perp$.
Its supremum is finite because $\|Y\|\leq R_h$ and a Gaussian
density times a fixed-degree polynomial is bounded.
In particular,
\begin{equation}
    G_{h,1}
    \leq
    \frac{1+2R_h+\sqrt{h(d-1)}}{\sqrt{2\pi h}}
    +\frac{1}{\sqrt{2\pi e}}.
    \label{eq:app_weighted_marginal_bound}
\end{equation}
For $r>0$, a change of variables gives, for every integrable
function $\eta$,
\begin{equation}
    r\,\mathbb E\!\left[
       (1+\|X\|)^j|\eta(rv^\top X+b)|
    \right]
    \leq G_{h,j}\|\eta\|_{L^1(\mathbb R)}.
    \label{eq:app_smoothing_derivative_bound}
\end{equation}
Applying this with $r=\|w\|$, $v=w/\|w\|$, and
$\eta=\sigma'$, for which $\|\sigma'\|_{L^1}=2$, yields
\begin{equation}
    \sup_{a,\xi}\|\nabla_\xi\zeta_q(a,\xi)\|
    \leq A(1+2G_{h,1})M(q).
    \label{eq:app_global_lipschitz_variation}
\end{equation}
The case $w=0$ follows directly from
\eqref{eq:app_first_variation_derivatives}.
Differentiating once more and applying
\eqref{eq:app_smoothing_derivative_bound}
with $j=2$ and $\eta=\sigma''$ gives
\[
    \sup_{a,\xi}
    \|\nabla_\xi^2\zeta_q(a,\xi)\|_{\mathrm{op}}
    \leq C_h M(q),
\]
since $\sigma'$ is bounded and $\sigma''$ is integrable.
Here and below $C_h$ may also depend on $A,R_h,d$.

For completeness, these estimates also ensure that the flow
is well defined. On laws with fixed $a$-marginal, use the distance
\[
    \mathsf W_2(q,\widetilde q)^2
    :=\int W_2(q^a,\widetilde q^a)^2\,\nu_a(\mathrm da).
\]
Coupling parameters with the same $a$ and using the boundedness
and Lipschitz continuity of $\sigma$ gives
\[
    \|P_q(x)-P_{\widetilde q}(x)\|
    \leq C_A(1+\|x\|)
       (1+\sqrt{m_2(q)}+\sqrt{m_2(\widetilde q)})
       \mathsf W_2(q,\widetilde q).
\]
Combining this bound with
\eqref{eq:app_first_variation_derivatives} and
\eqref{eq:app_smoothing_derivative_bound} shows that
$q\mapsto\nabla_\xi\zeta_q$ is Lipschitz in $\mathsf W_2$,
uniformly in $\xi$, on every set with bounded second moments.
The preceding Hessian bound gives the corresponding spatial
Lipschitz estimate.
Thus Picard iteration for the conditional McKean--Vlasov
equations gives a local solution to
\begin{equation}
    \mathrm d\xi_s
    =-\left(
        \nabla_\xi\zeta_{q_s}(a,\xi_s)
        +\frac{\lambda(s)}{\sigma_0^2}\xi_s
      \right)\mathrm ds
      +\sqrt{2\lambda(s)}\,\mathrm dB_s,
    \qquad a\sim\nu_a\text{ fixed}.
    \label{eq:app_conditional_sde}
\end{equation}
For $0<\lambda(s)\leq1$, It\^o's formula and
\eqref{eq:app_global_lipschitz_variation} imply
\[
    \frac{\mathrm d}{\mathrm ds}m_2(q_s)
    \leq C_h(1+m_2(q_s)).
\]
Gronwall's inequality prevents finite-time moment blowup and
extends the solution to all finite times.
The same argument, or the variation-of-constants formula for
\eqref{eq:app_conditional_sde}, controls higher moments on each
finite time interval.
On such an interval the non-Gaussian part of the drift is bounded,
its spatial derivatives are bounded, and $\lambda$ is bounded
below by a positive constant.
The conditional laws therefore have smooth positive densities.
The energy identities below follow by integration by parts with
spatial cutoffs and passage to the limit, using these bounds;
equivalently one may use the integrated entropy-dissipation
inequality, which suffices for every subsequent estimate.

For fixed $\lambda>0$, write
\[
    \mathcal E_\lambda(q)
    =J(q)+\lambda\operatorname{KL}(q\|\pi),
\]
and define
\[
    \mathcal D_\lambda(q)
    :=\int\left\|
        \nabla_\xi\zeta_q(a,\xi)
        +\lambda\nabla_\xi
           \log\frac{q^a(\xi)}{\gamma(\xi)}
       \right\|^2q(\mathrm d\theta).
\]
Along~\eqref{eq:Lt_flow},
\begin{equation}
    \frac{\mathrm d}{\mathrm ds}
       \mathcal E_{\lambda(s)}(q_s)
    =-\mathcal D_{\lambda(s)}(q_s)
      +\lambda'(s)\operatorname{KL}(q_s\|\pi)
    \leq0.
    \label{eq:app_free_energy_dissipation}
\end{equation}
Independence and symmetry of $a$ at initialization give
$P_\pi=0$.
Hence $J(\pi)\leq R_h^2/2$.
With $B:=1+R_h^2/2$, we obtain
\begin{equation}
    J(q_s)
    \leq\mathcal E_{\lambda(s)}(q_s)
    \leq J(\pi)<B,
    \qquad
    \operatorname{KL}(q_s\|\pi)\leq B/\lambda(s).
    \label{eq:app_energy_sublevel}
\end{equation}
The entropy variational inequality applied to
$\|\xi\|^2/(4\sigma_0^2)$ gives
\[
    \frac{m_2(q)}{4\sigma_0^2}
    \leq\operatorname{KL}(q\|\pi)
       +\log\int
           e^{\|\xi\|^2/(4\sigma_0^2)}\,\gamma(\mathrm d\xi)
    =\operatorname{KL}(q\|\pi)+\frac{n_\xi}{2}\log2.
\]
The assertion also follows for an initially unknown second moment
by applying the inequality to truncated quadratic functions and
then using monotone convergence.
This proves~\eqref{eq:annealed_moment_bound}, with $J(\pi)$ in
place of $B$.
In particular, uniformly over $0<\lambda\leq1$ and
$\mathcal E_\lambda(q)\leq B$,
\begin{equation}
    \operatorname{Lip}_\xi(\zeta_q(a,\cdot))
    \leq L_*\lambda^{-1/2},
    \label{eq:app_temperature_lipschitz}
\end{equation}
where one may take
\[
    L_*:=A(1+2G_{h,1})
    \left[
      R_h+2A\sigma_0
       \sqrt{B+\tfrac{n_\xi}{2}\log2}
    \right].
\]

For $a\in\{-A,A\}$, define the conditional proximal Gibbs law
\begin{equation}
    \pi_{q,\lambda}^a(\mathrm d\xi)
    :=\frac{1}{Z_a(q,\lambda)}
       \exp\!\left(-\frac{\zeta_q(a,\xi)}{\lambda}\right)
       \gamma(\mathrm d\xi).
    \label{eq:app_conditional_gibbs}
\end{equation}
The normalization is finite because $\zeta_q$ is Lipschitz
and $\gamma$ is Gaussian.
Normalization is performed separately for each sign of $a$.
We write $\pi_{q,\lambda}$ for the joint law with
$a$-marginal $\nu_a$ and these conditional laws.

We first verify a Gaussian perturbation bound.
Suppose $\ell:\mathbb R^{n_\xi}\to\mathbb R$ is
$L$-Lipschitz and let
\[
    U(\xi)=\frac{\|\xi\|^2}{2\sigma_0^2}+\ell(\xi),
    \qquad
    f(\xi)=\frac{\|\xi\|^2}{4\sigma_0^2}+\ell(\xi).
\]
Let $f^{**}$ be the lower convex envelope of $f$.
For any finite convex combination
$\sum_i\alpha_i\xi_i=\xi$, Jensen's inequality and the
Lipschitz property give, with
$v^2=\sum_i\alpha_i\|\xi_i-\xi\|^2$,
\[
    \sum_i\alpha_i f(\xi_i)
    \geq f(\xi)+\frac{v^2}{4\sigma_0^2}-Lv
    \geq f(\xi)-\sigma_0^2L^2.
\]
Taking the convex envelope yields
$0\leq f-f^{**}\leq\sigma_0^2L^2$.
Therefore
\[
    U(\xi)
    =\underbrace{
        \frac{\|\xi\|^2}{4\sigma_0^2}+f^{**}(\xi)
      }_{U_0(\xi)}
      +\underbrace{
        f(\xi)-f^{**}(\xi)
      }_{V_0(\xi)},
\]
where $U_0$ is $(2\sigma_0^2)^{-1}$-strongly convex and
$\operatorname{osc}(V_0)\leq\sigma_0^2L^2$.
The Bakry--\'Emery criterion, extended to convex nonsmooth
potentials by approximation, gives a log-Sobolev constant
$(2\sigma_0^2)^{-1}$ for the law proportional to $e^{-U_0}$.
The Holley--Stroock bounded-perturbation principle then gives
the lower bound
\[
    \frac{1}{2\sigma_0^2}e^{-\sigma_0^2L^2}
\]
for the law proportional to $e^{-U}$.
These two standard log-Sobolev principles are also recalled in
the mean-field Langevin analysis of~\cite{nitanda2022convex}.

Apply this argument with $\ell=\zeta_q(a,\cdot)/\lambda$.
By~\eqref{eq:app_temperature_lipschitz}, its Lipschitz constant
is at most $L_*\lambda^{-3/2}$.
Consequently, for every $\mathcal E_\lambda(q)\leq B$ and
$0<\lambda\leq1$, each $\pi_{q,\lambda}^a$ satisfies
\begin{equation}
\begin{aligned}
    \operatorname{KL}(\rho\|\pi_{q,\lambda}^a)
    &\leq\frac{1}{2\kappa_\lambda}
       I(\rho\|\pi_{q,\lambda}^a),\\
    \kappa_\lambda
    &:=\frac{1}{2\sigma_0^2}e^{-C_0/\lambda^3},
    \qquad C_0:=\sigma_0^2L_*^2,
\end{aligned}
\label{eq:app_verified_lsi}
\end{equation}
where
$I(\rho\|\mu)
:=\int\|\nabla\log(\mathrm d\rho/\mathrm d\mu)\|^2
\,\mathrm d\rho$.
We use this inequality conditionally and average over $\nu_a$;
no log-Sobolev inequality across the discrete labels is needed.

Let
\[
    v(\lambda):=\inf_{q\in\mathcal Q}\mathcal E_\lambda(q).
\]
The functional $J$ is convex under linear interpolation of
measures, since $P_q$ is linear in $q$.
For any $q,\widetilde q\in\mathcal Q$, convexity and
\eqref{eq:app_conditional_gibbs} give
\begin{align*}
    \mathcal E_\lambda(\widetilde q)-\mathcal E_\lambda(q)
    &\geq \int\zeta_q\,\mathrm d(\widetilde q-q)
       +\lambda\bigl[
          \operatorname{KL}(\widetilde q\|\pi)
          -\operatorname{KL}(q\|\pi)
       \bigr]\\
    &=\lambda\bigl[
          \operatorname{KL}(\widetilde q\|\pi_{q,\lambda})
          -\operatorname{KL}(q\|\pi_{q,\lambda})
       \bigr].
\end{align*}
The conditional normalizing constants cancel because both
measures have the same $a$-marginal.
Taking the infimum over $\widetilde q$ shows
\begin{equation}
    \mathcal E_\lambda(q)-v(\lambda)
    \leq\lambda\operatorname{KL}(q\|\pi_{q,\lambda}).
    \label{eq:app_gap_kl}
\end{equation}
This is the proximal Gibbs gap argument
of~\cite{nitanda2022convex}, here applied with fixed label masses.
Moreover,
\[
    \mathcal D_\lambda(q)
    =\lambda^2\int
       I(q^a\|\pi_{q,\lambda}^a)\,\nu_a(\mathrm da).
\]
Combining~\eqref{eq:app_verified_lsi} and
\eqref{eq:app_gap_kl} yields, on the energy sublevel,
\begin{equation}
    \mathcal D_\lambda(q)
    \geq 2\lambda\kappa_\lambda
          \bigl[\mathcal E_\lambda(q)-v(\lambda)\bigr].
    \label{eq:app_dissipation_gap}
\end{equation}
In particular, this inequality has been verified from the model
and the energy bound, rather than imposed as a PL assumption.

First, $0\leq v(\lambda)\leq J(\pi)<B$.
For each $\delta>0$,
\eqref{eq:app_zero_approximation_error} provides
$q_\delta\in\mathcal Q$ with $J(q_\delta)\leq\delta$, and hence
\[
    v(\lambda)
    \leq\delta+\lambda\operatorname{KL}(q_\delta\|\pi).
\]
Sending $\lambda\downarrow0$ and then $\delta\downarrow0$ gives
$v(\lambda)\to0$.
Also, $v$ is nondecreasing and concave, being the infimum of
affine functions of $\lambda$ with nonnegative slopes.
Extending $v(0)=0$, concavity implies
\begin{equation}
    0\leq v'(\lambda)
    \leq\frac{v(\lambda)}{\lambda}
    \leq\frac B\lambda
    \quad\text{for almost every }\lambda>0.
    \label{eq:app_value_derivative}
\end{equation}

Set
\[
    G(s):=\mathcal E_{\lambda(s)}(q_s)-v(\lambda(s))\geq0.
\]
The chain rule,~\eqref{eq:app_free_energy_dissipation},
\eqref{eq:app_dissipation_gap}, and
\eqref{eq:app_value_derivative} imply, almost everywhere,
\begin{align}
    G'(s)
    &\leq-2\lambda(s)\kappa_{\lambda(s)}G(s)
        +\lambda'(s)
          \bigl[
            \operatorname{KL}(q_s\|\pi)-v'(\lambda(s))
          \bigr]\nonumber\\
    &\leq-\alpha(s)G(s)+b(s),
    \label{eq:app_tracking_inequality}
\end{align}
where
\[
    \alpha(s):=2\lambda(s)\kappa_{\lambda(s)},
    \qquad
    b(s):=B\frac{|\lambda'(s)|}{\lambda(s)}.
\]
Take $\Gamma_0=C_0$ and $\Gamma>\Gamma_0$, and use
\[
    \lambda(s)
    =\left(1+\frac{\log(1+s)}\Gamma\right)^{-1/3}.
\]
Writing $\rho=C_0/\Gamma<1$, we obtain
\[
    \alpha(s)
    =\frac{e^{-C_0}}{\sigma_0^2}
       \lambda(s)(1+s)^{-\rho},
    \qquad
    b(s)=\frac{B}{3(1+s)(\Gamma+\log(1+s))}.
\]
Therefore
\[
    \int_0^\infty\alpha(s)\,\mathrm ds=\infty,
    \qquad
    \frac{b(s)}{\alpha(s)}\longrightarrow0.
\]
For any $\eta>0$, choose $S_\eta$ so that
$b(s)\leq\eta\alpha(s)$ for $s\geq S_\eta$.
Integrating~\eqref{eq:app_tracking_inequality} gives
\[
    G(s)\leq
       G(S_\eta)
       e^{-\int_{S_\eta}^s\alpha(u)\,\mathrm du}
       +\eta\left(
          1-e^{-\int_{S_\eta}^s\alpha(u)\,\mathrm du}
       \right).
\]
Thus $\limsup_{s\to\infty}G(s)\leq\eta$, and since $\eta$
is arbitrary, $G(s)\to0$.
Finally,
\begin{equation}
    0\leq J(q_s)
    \leq\mathcal E_{\lambda(s)}(q_s)
    =G(s)+v(\lambda(s))
    \longrightarrow0.
    \label{eq:app_denoising_convergence}
\end{equation}
This follows the free-energy tracking approach to annealing
in~\citet{chizat2022mean}, with the temperature dependence
verified explicitly in~\eqref{eq:app_verified_lsi}.
The constant $\Gamma_0=C_0$ depends only on
$h,R,A,\sigma_0,d$ through the bounds above.

Proposition~\ref{prop:score_decomp} and
\eqref{eq:app_denoiser_identity} imply that, for sufficiently
small $h$,
\[
    \sup_{x\in U_{\tau/2}}
       \|D_t(x)-\Pi_{\mathcal M}(x)\|
    \leq C_*h,
\]
where $C_*$ depends on $\mathcal M$ and $p_0$ but not on $h$.
Outside this tube, any nearest-point selection satisfies
$\|\Pi_{\mathcal M}(x)\|\leq R$, while
$\|D_t(x)\|\leq R_h\leq R$.
Consequently, for every $s\geq0$,
\begin{align*}
    F_t(q_s)
    &=\tfrac12\mathbb E
       \|P_{q_s}(X)-\Pi_{\mathcal M}(X)\|^2\\
    &\leq\mathbb E\|P_{q_s}(X)-D_t(X)\|^2
          +\mathbb E\|D_t(X)-\Pi_{\mathcal M}(X)\|^2\\
    &\leq2J_t(q_s)+C_*^2h^2
          +4R^2\mathbb P(X\notin U_{\tau/2}).
\end{align*}
To bound the remaining probability, note that $X_0\in\mathcal M$
and
\[
    d_{\mathcal M}(X)
    \leq\|X-X_0\|
    \leq(1-\sqrt{1-h})R+\sqrt h\|Z\|.
\]
For sufficiently small $h$, we have
$(1-\sqrt{1-h})R\leq\tau/4$, so
\[
    \mathbb P(X\notin U_{\tau/2})
    \leq\mathbb P\!\left(
         \|Z\|^2\geq\frac{\tau^2}{16h}
       \right).
\]
Since $\mathbb E e^{\|Z\|^2/4}=2^{d/2}$, the exponential
Markov inequality gives
\begin{equation}
    \mathbb P(X\notin U_{\tau/2})
    \leq\min\left\{
        1,\,
        2^{d/2}\exp\!\left(-\frac{\tau^2}{64h}\right)
    \right\}.
    \label{eq:app_tube_tail}
\end{equation}
For fixed $d$, this probability is $o(h^2)$ as $h\downarrow0$;
in particular, it is at most $h^2$ for all sufficiently small $h$.
Taking $C_{\mathrm{geom}}=C_*^2+4R^2$ therefore proves~\eqref{eq:collapse_rate}.
Combining~\eqref{eq:app_denoising_convergence} with the alignment
bound and Cauchy--Schwarz yields
\[
    \limsup_{s\to\infty}
    \mathbb E_{X\sim p_t}
       \|P_{q_s}(X)-\Pi_{\mathcal M}(X)\|
    \leq
    \left[
       2C_{\mathrm{geom}}h^2
       +8R^2\mathbb P(X\notin U_{\tau/2})
    \right]^{1/2}
    =O(h).
\]
Here the training-time limit is taken for each fixed $h>0$,
followed by $h\downarrow0$.
Neither a convergence rate uniform in $h$ nor an
ambient-dimension-free optimization bound is used.
\end{proof}

\subsection{Proof of Theorem~\ref{thm:stage2_intrinsic_main}}
\label{app:stage2_detailed_bounds}
\label{app:stage2_closed_bound}

\begin{proof}
Throughout, $C$ denotes a finite constant depending only on the
fixed geometric, density, and feature-class parameters, including
$B$ and the fixed ambient dimension $d$; its value may change
between displays. It is independent of the noise levels and of
$n,m,s,\delta$.
Fix $\delta\in(0,1)$ and set
\[
 u=\log(16/\delta),\qquad \gamma=\frac{k}{k+2}.
\]
Take $h(t_2)>0$ sufficiently small and choose
$h(t_1)=h(t_2)^{3/2}$ and $\eta=h(t_2)$ as in the theorem.
The small-noise threshold below is independent of $\delta$.
The constant in the lower bounds on $m$ and $n$ is chosen
sufficiently large to satisfy all numerical requirements below.
Here $k\ge1$ is the dimension of the compact manifold.
All population expectations involving $P$ are conditional on the
fixed Stage-1 map; the feature draw and Stage-2 observations are
independent of that map and of one another.

Set $K=\overline{U_{\tau/2}}$ and write $p_h$ for the law of
$\sqrt{1-h}X_0+\sqrt h\,\varepsilon$.
Define on $U_\tau$
\[
 G(x):=\nabla_x\!\left[
  \log p_0(\Pi_{\mathcal M}(x))+
  H(\Pi_{\mathcal M}(x),x-\Pi_{\mathcal M}(x))\right].
\]
The assumptions $p_0\in C^4(\mathcal M)$, $p_0\ge p_{\min}>0$,
and smoothness of the compact manifold imply that $G$ is $C^1$
on a neighborhood of $K$. The proof of
Proposition~\ref{prop:score_decomp} establishes the log-density
expansion
\begin{gather*}
 \begin{aligned}
 \log p_h(x)
 &=-\frac{d_{\mathcal M}(x)^2}{2h}
   +\log p_0(\Pi_{\mathcal M}(x))\\
 &\quad+H(\Pi_{\mathcal M}(x),x-\Pi_{\mathcal M}(x))
   +C(h)+\rho_h(x),
 \end{aligned}\\
 \|\rho_h\|_{C^1(K)}\le C_*h,
\end{gather*}
where $C(h)$ is independent of $x$. Differentiating this
$C^1$-controlled expansion gives
\begin{equation}
 \sup_{x\in K}
 \left\|\nabla\log p_h(x)+\frac{x-\Pi_{\mathcal M}(x)}h-G(x)\right\|
 \le C_*h,\qquad 0<h\le h_0.
 \label{eq:stage2_uniform_score_expansion}
\end{equation}
This estimate is a consequence of the proposition, not an
additional Stage-2 assumption.

We also verify the expression for $r_0$ in
\eqref{eq:stage2_principal_residual}. At $x=z\in\mathcal M$,
$D\Pi_{\mathcal M}(z)=\operatorname{Proj}_{T_z\mathcal M}$.
For an ambient direction $w$, the derivative of
$\nu(x):=x-\Pi_{\mathcal M}(x)$ at $z$ is
$\operatorname{Proj}_{N_z\mathcal M}w$. Thus the derivative of
the VP term $-\langle\nu(x),\Pi_{\mathcal M}(x)\rangle/2$ is
$-\langle\operatorname{Proj}_{N_z\mathcal M}w,z\rangle/2$.
Since $B_0=0$, differentiating the determinant term at $z$ gives
\[
 D\!\left[-\tfrac12\log\det(I-B_{\nu(x)})\right]_{x=z}[w]
 =\frac12\operatorname{Tr}
       B_{\operatorname{Proj}_{N_z\mathcal M}w}
 =\frac12\langle\mathcal{H}_{\mathcal M}(z),w\rangle.
\]
The density term contributes $\nabla_{\mathcal M}\log p_0(z)$.
Hence $G(z)=r_0(z)$; the ambient gradient is taken before
restriction to the manifold.

For $j\in\{1,2\}$, the population denoiser satisfies
\[
 D_{t_j}(x)
 :=x+h(t_j)s^*(x,t_j)
 =\mathbb E[\sqrt{1-h(t_j)}X_0\mid X_{t_j}=x],
 \qquad \|D_{t_j}(x)\|\le R.
\]
This follows by differentiating the Gaussian convolution.
Set
\[
 T_{t_2}:=\mathbb P_{p_{t_2}}(X\notin U_{\tau/2}).
\]
Since $X_0\in\mathcal M$,
\[
 d_{\mathcal M}(X)
 \le (1-\sqrt{1-h(t_2)})R+\sqrt{h(t_2)}\,\|\varepsilon\|.
\]
For small enough $h(t_2)$, the deterministic shift is at most
$\tau/4$. Using
$\mathbb E\exp(\|\varepsilon\|^2/4)=2^{d/2}$ gives
\begin{equation}
 T_{t_2}
 \le 2^{d/2}\exp\left(-\frac{\tau^2}{64h(t_2)}\right)
 \le h(t_2)^3.
 \label{eq:closed43_tail}
\end{equation}
The last inequality holds for sufficiently small $h(t_2)$
with $d$ fixed.

Both noisy densities are continuous and strictly positive, so
\[
 K_{12}:=\sup_{x\in K}\frac{p_{t_2}(x)}{p_{t_1}(x)}<\infty
\]
for the fixed pair of noise levels.
The uniform score estimate
\eqref{eq:stage2_uniform_score_expansion} gives
\[
 D_{t_1}(x)-\Pi_{\mathcal M}(x)
 =h(t_1)G(x)+h(t_1)e_{t_1}(x),
 \qquad \sup_{x\in K}\|e_{t_1}(x)\|\le C_*h(t_1).
\]
Since $G$ is continuous on the compact set $K$, it follows that
\[
 \sup_{x\in K}\|D_{t_1}(x)-\Pi_{\mathcal M}(x)\|
 \le C_{\mathrm g}h(t_1),
 \qquad C_{\mathrm g}:=\sup_K\|G\|+C_*h_0.
\]
Clipping is nonexpansive, and $D_{t_1}(x)\in B(0,R)$.
Consequently, with
$P_s=\operatorname{Proj}_{B(0,R)}P_{q_s}$,
\begin{align}
 \varepsilon_{\mathrm{proj},2}^2
 &:=\mathbb E_{p_{t_2}}
       \|P_s(X)-\Pi_{\mathcal M}(X)\|^2\notag\\
 &\le 4K_{12}J_{t_1}(q_s)
       +2C_{\mathrm g}^2h(t_1)^2+4R^2T_{t_2}.
 \label{eq:closed43_transfer}
\end{align}
Indeed, on the tube use
$\|P_s-\Pi_{\mathcal M}\|^2
\le2\|P_s-D_{t_1}\|^2+2C_{\mathrm g}^2h(t_1)^2$,
change measure using $K_{12}$, and use
$\mathbb E_{p_{t_1}}\|P_{q_s}-D_{t_1}\|^2=2J_{t_1}(q_s)$.
Outside the tube both $P_s$ and $\Pi_{\mathcal M}$ have norm
at most $R$.

By Theorem~\ref{thm:dimensional_collapse},
$J_{t_1}(q_s)\to0$ for this fixed $h(t_1)>0$.
Thus there exists $s_0(h(t_2))<\infty$ such that
$K_{12}J_{t_1}(q_s)\le h(t_2)^3$ for every $s\ge s_0(h(t_2))$.
The threshold may depend on the fixed Stage-1 initialization and
flow parameters, but not on $n,m$, or $\delta$.
Since $h(t_1)^2=h(t_2)^3$, equations
\eqref{eq:closed43_tail}--\eqref{eq:closed43_transfer} imply
\begin{equation}
 \varepsilon_{\mathrm{proj},2}^2\le Ch(t_2)^3.
 \label{eq:closed43_projection}
\end{equation}
No noise-uniform bound on $K_{12}$ is needed to obtain this
finite-time existence statement.
Fix any such $s$ and condition on the resulting Stage-1 map
$P=P_s$ for the remainder of the proof.

Define the global residual
\[
 r_{t_2}(x)
 :=s^*(x,t_2)+\frac{x-\Pi_{\mathcal M}(x)}{h(t_2)}
 =\frac{D_{t_2}(x)-\Pi_{\mathcal M}(x)}{h(t_2)},
\]
using a measurable nearest-point selection outside the tube.
On $K$, write
\[
 G(x)=\nabla_x\bigl[
       \log p_0(\Pi_{\mathcal M}(x))
       +H(\Pi_{\mathcal M}(x),x-\Pi_{\mathcal M}(x))
       \bigr].
\]
Equation~\eqref{eq:stage2_uniform_score_expansion} gives
$\|r_{t_2}(x)-G(x)\|\le C_*h(t_2)$, while
$r_0(z)=G(z)$ for $z\in\mathcal M$.
The $C^1$ regularity established above gives
$L_N:=\sup_{x\in K}\|DG(x)\|_{\mathrm{op}}<\infty$.
If $z=\Pi_{\mathcal M}(x)$ and $x\in K$, the normal segment
$z+\theta(x-z)$, $0\le\theta\le1$, lies in $K$ and has
the same projection $z$. The fundamental theorem of calculus
therefore yields
\[
 \|G(x)-G(\Pi_{\mathcal M}(x))\|
 \le L_Nd_{\mathcal M}(x),\qquad x\in K.
\]
Let $B_0=\sup_{z\in\mathcal M}\|r_0(z)\|<\infty$ and
$Z=\Pi_{\mathcal M}(X)$. Since
$\|r_{t_2}(x)\|\le2R/h(t_2)$ globally,
\begin{align}
 \beta_h^2
 &:=\mathbb E\|r_{t_2}(X)-r_0(Z)\|^2\notag\\
 &\le 2L_N^2\bigl[
        d\,h(t_2)+(1-\sqrt{1-h(t_2)})^2R^2
       \bigr]
       +2C_*^2h(t_2)^2\notag\\
 &\quad+
       \left(\frac{2R}{h(t_2)}+B_0\right)^2T_{t_2}
 \le Ch(t_2).
 \label{eq:closed43_bias}
\end{align}
Here we used
$\mathbb E d_{\mathcal M}(X)^2
\le\mathbb E\|X-X_0\|^2
\le d\,h(t_2)+(1-\sqrt{1-h(t_2)})^2R^2$
and \eqref{eq:closed43_tail}.

Let $\mu$ be the law of $Z$, and set
$\psi_j(z)=\sigma(v_j^\top z+b_j)$.
Under \eqref{eq:stage2_bounded_rf_representation}, define
\[
 U_m=\frac1{\sqrt m}
       \begin{bmatrix}c(\omega_1)&\cdots&c(\omega_m)\end{bmatrix},
 \qquad \|U_m\|_F^2\le B^2.
\]
The functions $F_j=c(\omega_j)\psi_j$ are independent random
vectors in $\mathcal H=L^2(\mu;\mathbb R^d)$ with
$\mathbb E_\omega F_j=r_0$ and $\|F_j\|_{\mathcal H}\le\kappa B$.
Therefore
\[
 \mathbb E_\omega
 \left\|\frac1m\sum_{j=1}^mF_j-r_0\right\|_{\mathcal H}^2
 =\frac1m\left(\mathbb E_\omega\|F_1\|_{\mathcal H}^2
                -\|r_0\|_{\mathcal H}^2\right)
 \le\frac{\kappa^2B^2}{m}.
\]
Independence and centering make all cross terms vanish.
By Cauchy--Schwarz, the expected norm is at most
$\kappa B/\sqrt m$. Changing one feature changes that norm
by at most $2\kappa B/m$.
The bounded-differences inequality
\citep{mcdiarmid1989method}, applied to that norm, gives
\begin{equation}
 a_m(U_m):=\mathbb E\|U_m\Phi(Z)-r_0(Z)\|^2
 \le\frac{\kappa^2B^2}{m}(1+\sqrt{2u})^2
 \le Ch(t_2)
 \label{eq:closed43_approximation}
\end{equation}
except on an event of probability at most $e^{-u}$.
The last inequality uses $m\ge Cu/h(t_2)$.

For later use, Gaussian quadratic-form concentration states
that, for a fixed positive semidefinite matrix $A$,
\begin{equation}
 \frac1m\sum_{j=1}^mv_j^\top A v_j
 \le \operatorname{Tr}A
      +2\sqrt{\frac{u\operatorname{Tr}(A^2)}m}
      +\frac{2u\|A\|_{\mathrm{op}}}m
 \label{eq:closed43_gaussian}
\end{equation}
with failure probability at most $e^{-u}$.
This follows by diagonalizing $A$ and applying the exponential
Markov inequality to the centered sum $Q$, whose moment
generating function satisfies
\[
 \log\mathbb E e^{\theta Q}
 \le\frac{m\theta^2\operatorname{Tr}(A^2)}
            {1-2\theta\|A\|_{\mathrm{op}}},
 \qquad 0<2\theta\|A\|_{\mathrm{op}}<1.
\]
More explicitly, if $\lambda_l$ are the eigenvalues of $A$, then
$Q=\sum_{j,l}\lambda_l(g_{jl}^2-1)$ for independent standard
Gaussians $g_{jl}$. The bound follows from
$-\log(1-v)-v\le v^2/[2(1-v)]$, $0\le v<1$.
For $a=m\operatorname{Tr}(A^2)>0$ and
$b=2\|A\|_{\mathrm{op}}$, take
$\theta=\sqrt u/(\sqrt a+b\sqrt u)$ to obtain
$\mathbb P(Q>2\sqrt{au}+bu)\le e^{-u}$.
If $A=0$, the bound is immediate.
Apply \eqref{eq:closed43_gaussian} to
$\Sigma_e=\mathbb E[(P(X)-Z)(P(X)-Z)^\top]$.
Since $\operatorname{Tr}\Sigma_e=
\varepsilon_{\mathrm{proj},2}^2$ and $m\ge u$, Lipschitz
continuity gives, on a second feature event,
\begin{align}
 \Delta_\Phi^2
 &:=\mathbb E\|\Phi(P(X))-\Phi(Z)\|^2\notag\\
 &\le\frac{L_\sigma^2}{m}\sum_{j=1}^m v_j^\top\Sigma_e v_j
 \le5L_\sigma^2\varepsilon_{\mathrm{proj},2}^2
 \le Ch(t_2)^3.
 \label{eq:closed43_feature_transfer}
\end{align}
This event also has failure probability at most $e^{-u}$.

Let $V_{\mathcal M}=\operatorname{Vol}_{\mathcal M}(\mathcal M)$
and
\[
 A_{\mathcal M}
 :=\int_{\mathcal M}\operatorname{Proj}_{T_z\mathcal M}
                 \,d\operatorname{Vol}_{\mathcal M}(z).
\]
Its trace is $kV_{\mathcal M}$, its operator norm is at most
$V_{\mathcal M}$, and
$\operatorname{Tr}(A_{\mathcal M}^2)\le kV_{\mathcal M}^2$.
Since
\[
 \nabla_{\mathcal M}\psi_j(z)
 =\sigma'(v_j^\top z+b_j)\operatorname{Proj}_{T_z\mathcal M}v_j,
\]
we have
\[
 S_1^2:=\frac1m\sum_{j=1}^m\|\psi_j\|_{H^1(\mathcal M)}^2
 \le V_{\mathcal M}\kappa^2
      +\frac{L_\sigma^2}{m}\sum_{j=1}^mv_j^\top A_{\mathcal M}v_j.
\]
By \eqref{eq:closed43_gaussian} and $m\ge u$, on a third
feature event with failure probability at most $e^{-u}$,
\begin{equation}
 S_1^2\le
 V_{\mathcal M}\bigl[
 \kappa^2+L_\sigma^2(k+2\sqrt{k}+2)\bigr]
 =:S_*^2.
 \label{eq:closed43_sobolev}
\end{equation}
Only tangent Gaussian components enter this bound.

Let $\{e_j\}_{j\ge1}$ be an orthonormal eigenbasis of
$-\Delta_{\mathcal M}$ in $L^2(\operatorname{Vol}_{\mathcal M})$,
with eigenvalues $\Lambda_j$ in nondecreasing order, and let
$\mathcal P_N$ be projection onto its first $N$ elements.
For every $f\in H^1(\mathcal M)$, the density and spectral
conditions imply
\begin{align}
 \|f-\mathcal P_Nf\|_{L^2(\mu)}^2
 &\le\rho_+\sum_{j>N}|\langle f,e_j\rangle|^2\notag\\
 &\le\frac{\rho_+}{1+\Lambda_{N+1}}
       \sum_{j>N}(1+\Lambda_j)|\langle f,e_j\rangle|^2\notag\\
 &\le\rho_+C_WN^{-2/k}\|f\|_{H^1(\mathcal M)}^2.
 \label{eq:closed43_spectral_approximation}
\end{align}
Define operators $\mathsf T,\mathsf T_N:
\mathbb R^m\to L^2(\mu)$ by
\[
 \mathsf T a=\frac1{\sqrt m}\sum_{j=1}^ma_j\psi_j,
 \qquad
 \mathsf T_N a=\frac1{\sqrt m}\sum_{j=1}^ma_j\mathcal P_N\psi_j.
\]
Then $C_Z:=\mathbb E[\Phi(Z)\Phi(Z)^\top]=\mathsf T^*\mathsf T$,
$\operatorname{rank}(\mathsf T_N)\le N$, and
\[
 \|\mathsf T-\mathsf T_N\|_{\mathrm{HS}}^2
 \le\rho_+C_WS_*^2N^{-2/k}=:D_*N^{-2/k}.
\]
The best rank-$N$ approximation principle gives
$\sum_{j>N}\mu_j(C_Z)\le D_*N^{-2/k}$, where $\mu_j(C_Z)$
are its eigenvalues in nonincreasing order.
For clarity, if $N<m$, let $\mathsf Q$ be the orthogonal
projection in $\mathbb R^m$ onto $\ker\mathsf T_N$.
Then $\operatorname{rank}(\mathsf Q)\ge m-N$ and
\[
 \sum_{j>N}\mu_j(C_Z)
 \le\operatorname{Tr}(\mathsf Q C_Z\mathsf Q)
 =\|\mathsf T\mathsf Q\|_{\mathrm{HS}}^2
 =\|(\mathsf T-\mathsf T_N)\mathsf Q\|_{\mathrm{HS}}^2
 \le\|\mathsf T-\mathsf T_N\|_{\mathrm{HS}}^2.
\]
The first inequality is the variational characterization of
the sum of the smallest eigenvalues. For $N\ge m$ the
eigenvalue tail is zero.
For a positive semidefinite matrix $A$, define
\[
 \mathcal N_\eta(A)
 :=\operatorname{Tr}[A(A+\eta I)^{-1}].
\]
For every $N\ge1$,
\[
 \mathcal N_\eta(C_Z)
 \le N+\frac1\eta\sum_{j>N}\mu_j(C_Z)
 \le N+\frac{D_*}{\eta}N^{-2/k}.
\]
Choosing $N=\max\{1,\lceil(D_*/\eta)^\gamma\rceil\}$ yields
$\mathcal N_\eta(C_Z)\le C\eta^{-\gamma}$ for $0<\eta\le1$.

Now set $\varphi=\Phi(P(X))$ and
\[
 C_\varphi=\mathbb E[\varphi\varphi^\top],\qquad
 C_e=\mathbb E[(\Phi(P(X))-\Phi(Z))
               (\Phi(P(X))-\Phi(Z))^\top].
\]
The inequality $C_\varphi\preceq2C_Z+2C_e$ follows from
$(a+b)^2\le2a^2+2b^2$.
For positive semidefinite $A,B$, inverse monotonicity gives
\[
 \mathcal N_\eta(A+B)
 \le\mathcal N_\eta(A)+\mathcal N_\eta(B).
\]
Indeed, split the trace into the contributions of $A$ and $B$
and use $(A+B+\eta I)^{-1}\preceq(A+\eta I)^{-1}$ and its
counterpart for $B$.
Together with operator monotonicity of $A(A+\eta I)^{-1}$,
this implies
\begin{align}
 \mathcal N_\eta:=\mathcal N_\eta(C_\varphi)
 &\le\mathcal N_{\eta/2}(C_Z)
       +\frac{2\operatorname{Tr}C_e}{\eta}\notag\\
 &\le C\eta^{-\gamma}+\frac{2\Delta_\Phi^2}{\eta}
 \le Ch(t_2)^{-\gamma},
 \qquad \eta=h(t_2).
 \label{eq:closed43_capacity}
\end{align}
The last step uses \eqref{eq:closed43_feature_transfer}.
Thus \eqref{eq:closed43_approximation},
\eqref{eq:closed43_feature_transfer}, and
\eqref{eq:closed43_capacity} hold jointly with feature
probability at least $1-3e^{-u}$.
No dimension or rank constraint on the image of $P$ was used.

The regression target is
\[
 Y=\frac{\sqrt{1-h(t_2)}X_0-P(X)}{h(t_2)},
 \qquad \|Y\|\le\frac{2R}{h(t_2)}.
\]
Its conditional mean is
\[
 r_P(X):=\mathbb E[Y\mid X]
 =\frac{D_{t_2}(X)-P(X)}{h(t_2)}.
\]
Consequently,
\[
 \mathcal R_h(U)=\mathbb E\|U\varphi-r_P(X)\|^2,
 \qquad
 \mathbb E\|U\varphi-Y\|^2
 =\mathcal R_h(U)+\mathbb E\|Y-r_P(X)\|^2.
\]
The second term is independent of $U$.
Define
\[
 A_{m,\eta}:=\inf_U
       \{\mathcal R_h(U)+\eta\|U\|_F^2\},
\]
and let $U_\eta$ be its unique minimizer.
For any $U$, the score error has the exact decomposition
\begin{align*}
 s_U(X,t_2)-s^*(X,t_2)
 &=U[\Phi(P(X))-\Phi(Z)]
   +[U\Phi(Z)-r_0(Z)]\\
 &\quad+[r_0(Z)-r_{t_2}(X)]
   +\frac{P(X)-Z}{h(t_2)}.
\end{align*}
Applying the squared four-term inequality with $U=U_m$ gives
\begin{align}
 A_{m,\eta}
 &\le4\beta_h^2
       +\frac{4\varepsilon_{\mathrm{proj},2}^2}{h(t_2)^2}
       +4a_m(U_m)
       +(\eta+4\Delta_\Phi^2)\|U_m\|_F^2\notag\\
 &\le Ch(t_2)
 \label{eq:closed43_oracle}
\end{align}
on the feature event, by
\eqref{eq:closed43_projection}, \eqref{eq:closed43_bias},
\eqref{eq:closed43_approximation},
\eqref{eq:closed43_feature_transfer}, and $\eta=h(t_2)$.
Since the score risk is nonnegative,
\begin{equation}
 \eta\|U_\eta\|_F^2\le A_{m,\eta}\le Ch(t_2),
 \qquad \|U_\eta\|_F^2\le C.
 \label{eq:closed43_ridge_norm}
\end{equation}
As $\|\varphi\|\le\kappa$, this also controls the residual envelope:
\begin{equation}
 M_\eta:=\operatorname*{ess\,sup}\|Y-U_\eta\varphi\|
 \le\frac{2R}{h(t_2)}+\kappa\|U_\eta\|_F
 \le\frac{C}{h(t_2)}.
 \label{eq:closed43_envelope}
\end{equation}
Thus the norm bound is a consequence of the oracle estimate,
not an additional assumption on the fitted network.

Condition on a feature realization in the preceding event.
Write $C_\eta=C_\varphi+\eta I_m$ and
\[
 \ell_\eta:=\operatorname*{ess\,sup}
             \varphi^\top C_\eta^{-1}\varphi
 \le\frac{\kappa^2}{\eta}.
\]
The population normal equation and completion of the square give
\begin{gather}
 U_\eta C_\eta=\mathbb E[Y\varphi^\top],\notag\\
 \mathcal R_h(U)+\eta\|U\|_F^2
 =A_{m,\eta}+\|(U-U_\eta)C_\eta^{1/2}\|_F^2.
 \label{eq:closed43_quadratic_identity}
\end{gather}
For the independent training sample, let
\[
 C_n=\frac1n\sum_{i=1}^n\varphi_i\varphi_i^\top,
 \qquad E_i=Y_i-U_\eta\varphi_i,
\]
and define
\[
 \Xi_i=(E_i\varphi_i^\top-\eta U_\eta)C_\eta^{-1/2},
 \qquad V_n=\frac1n\sum_{i=1}^n\Xi_i.
\]
Then $\mathbb E\Xi_i=0$ and, in Frobenius norm,
\begin{equation}
 \|\Xi_i\|_F\le2M_\eta\sqrt{\ell_\eta},
 \qquad
 \mathbb E\|\Xi_i\|_F^2\le M_\eta^2\mathcal N_\eta.
 \label{eq:closed43_summands}
\end{equation}
The variance bound follows by dropping the squared mean from
$\mathbb E\|(Y-U_\eta\varphi)
                \varphi^\top C_\eta^{-1/2}\|_F^2$.

For $v=C_\eta^{-1/2}\varphi$, put $Q=vv^\top$ and $T=\mathbb E Q$.
Then $0\preceq Q,T\preceq\ell_\eta I_m$, $T\preceq I_m$, and
\[
 \|Q-T\|_{\mathrm{op}}\le\ell_\eta,
 \qquad
 \mathbb E(Q-T)^2\preceq\ell_\eta T\preceq\ell_\eta I_m.
\]
Indeed, $Q^2=\|v\|^2Q\preceq\ell_\eta Q$, and hence
$\mathbb E(Q-T)^2=\mathbb E Q^2-T^2\preceq\ell_\eta T$.
Matrix Bernstein \citep[Theorem~6.1]{tropp2012user} therefore yields
\begin{equation}
 \mathbb P\left(
 \|C_\eta^{-1/2}(C_n-C_\varphi)C_\eta^{-1/2}\|_{\mathrm{op}}
 >\frac12\right)
 \le2m\exp\left(-\frac{3n}{28\ell_\eta}\right)
 \le2e^{-u}.
 \label{eq:closed43_sample_covariance}
\end{equation}
Here the theorem's sample-size condition and
$\ell_\eta\le\kappa^2/h(t_2)$ justify the final inequality.
Specifically, the two-sided matrix Bernstein bound, applied at
threshold $n/2$ with summand norm at most $\ell_\eta$ and
total variance at most $n\ell_\eta$, has exponent
\[
 \frac{(n/2)^2}{2(n\ell_\eta+\ell_\eta n/6)}
 =\frac{3n}{28\ell_\eta}.
\]
A sample-size constant at least $28/3$ ensures
$n\ge(28/3)\ell_\eta(u+\log m)$, as required.
If $\ell_\eta=0$, then the features vanish almost surely and
$\mathcal R_h(\widehat U)=A_{m,\eta}\le Ch(t_2)$ directly;
hence only $\ell_\eta>0$ needs the displayed argument.

For independent centered Hilbert-space vectors $\xi_i$ with
$\|\xi_i\|\le L$ and
$\sum_i\mathbb E\|\xi_i\|^2\le V$, the Bennett inequality
of \citet[Theorem~3.4]{pinelis1994optimum} gives its Bernstein
consequence
\[
 \mathbb P\!\left(
  \left\|\sum_i\xi_i\right\|>
  \sqrt{2Vu}+\frac{2Lu}{3}
 \right)\le2e^{-u}.
\]
This follows from
$(1+x)\log(1+x)-x\ge x^2/[2(1+x/3)]$, $x\ge0$.
The space of $d\times m$ matrices with the Frobenius norm is
a Hilbert space. Apply this bound to
\eqref{eq:closed43_summands} with
$L=2M_\eta\sqrt{\ell_\eta}$ and
$V=nM_\eta^2\mathcal N_\eta$, divide by $n$, and enlarge
the linear term. We obtain
\begin{equation}
 \|V_n\|_F
 \le M_\eta\left(
       \sqrt{\frac{2\mathcal N_\eta u}{n}}
       +\frac{4\sqrt{\ell_\eta}\,u}{n}
       \right)
 \label{eq:closed43_vector_concentration}
\end{equation}
with failure probability at most $2e^{-u}$.

The empirical normal equation implies
\[
 (\widehat U-U_\eta)(C_n+\eta I_m)=V_nC_\eta^{1/2}.
\]
On the good covariance event, whose failure probability is
bounded in \eqref{eq:closed43_sample_covariance}, the matrix
$C_\eta^{-1/2}(C_n+\eta I_m)C_\eta^{-1/2}$ has smallest
eigenvalue at least $1/2$. Hence
\[
 \|(\widehat U-U_\eta)C_\eta^{1/2}\|_F\le2\|V_n\|_F.
\]
Combining this with \eqref{eq:closed43_quadratic_identity},
\eqref{eq:closed43_vector_concentration}, and
$(a+b)^2\le2a^2+2b^2$ gives
\begin{equation}
 \mathcal R_h(\widehat U)
 \le A_{m,\eta}
       +\frac{16M_\eta^2\mathcal N_\eta u}{n}
       +\frac{128M_\eta^2\ell_\eta u^2}{n^2}.
 \label{eq:closed43_ridge_generalization}
\end{equation}
All concentrations use total vector norms, without a
coordinatewise union bound over the output dimension.

Substitute \eqref{eq:closed43_capacity},
\eqref{eq:closed43_oracle}, \eqref{eq:closed43_envelope}, and
$\ell_\eta\le\kappa^2/h(t_2)$ into
\eqref{eq:closed43_ridge_generalization}:
\[
 \mathcal R_h(\widehat U)
 \le C\left(
 h(t_2)+\frac{u}{n\,h(t_2)^{2+\gamma}}
       +\frac{\kappa^2u^2}{n^2h(t_2)^3}
 \right).
\]
The sample-size condition ensures
$\kappa^2u/[nh(t_2)]\le C$, and $h(t_2)\le1$ gives
\[
 \frac{\kappa^2u^2}{n^2h(t_2)^3}
 =\frac{u}{nh(t_2)^2}\frac{\kappa^2u}{nh(t_2)}
 \le C\frac{u}{n\,h(t_2)^{2+\gamma}}.
\]
Consequently,
\[
 \mathcal R_h(\widehat U)
 \le C\left(
 h(t_2)+\frac{\log(16/\delta)}
                   {n\,h(t_2)^{2+k/(k+2)}}
 \right),
\]
which is \eqref{eq:stage2_intrinsic_main}.
The three feature events have total failure probability at
most $3e^{-u}$, and the two conditional sample events contribute
at most $4e^{-u}$. Their joint failure probability is therefore
at most $7e^{-u}=7\delta/16<\delta$.
The argument applies to every fixed, sufficiently large $s$;
it does not assert a simultaneous event over all training times
or a noise-uniform bound on the required Stage-1 training time.
\end{proof}

\subsubsection{Curvature correction under additional structure}
\label{app:stage2_curvature_correction}

This extension assumes that the true projection and second fundamental
form are available. It does not follow from the Stage-1 $L^2$ alignment
guarantee. Throughout this subsection, write $h=h(t_2)$ and
$a_h=\sqrt{1-h}$. We use the convention
$\langle B_\nu w_1,w_2\rangle
=\langle\mathrm{II}(w_1,w_2),\nu\rangle$ for the shape operator.

\begin{proposition}[Curvature-corrected score with exact geometry]
\label{prop:stage2_curvature_head}
Under the assumptions of Proposition~\ref{prop:score_decomp},
suppose that $\mathcal M$ is connected and its second fundamental
form is parallel, $\nabla^\perp\mathrm{II}=0$. Then there is a fixed
linear subspace $E\subset\mathbb R^d$ containing $\mathcal M$, with
$\ell:=\dim E\leq1+k+k(k+1)/2$.
Write $x=u+v$ with $u=\operatorname{Proj}_E x$ and $v\in E^\perp$.
For $u\in U_{\tau/2}^{E}$, set
\[
 z=\Pi_{\mathcal M}^{E}(u),\qquad \nu=u-z,\qquad
 Q_\nu=(I_{T_z\mathcal M}-B_\nu)^{-1},\qquad
 g(z)=\nabla_{\mathcal M}\log p_0(z).
\]
Use the VP geometric correction
\[
 H(z,\nu)=-\frac12\langle\nu,z\rangle
          -\frac12\log\det(I_{T_z\mathcal M}-B_\nu).
\]
Here and below gradients differentiate the composed functions of $u$.
For all sufficiently small $h$,
\begin{equation}
 s^*(x,t_2)=-\frac{v+\nu}{h}
       +Q_\nu g(z)+\nabla_uH(z,\nu)+e_h(u),
 \qquad \|e_h(u)\|\leq C_*h.
 \label{eq:stage2_curvature_decomposition}
\end{equation}
Consequently the model
\begin{equation}
 s_U^{\rm geo}(x,t_2)=-\frac{v+\nu}{h}
       +Q_\nu\operatorname{Proj}_{T_z\mathcal M}U\Phi(z)
       +\nabla_uH(z,\nu)
 \label{eq:stage2_curvature_model}
\end{equation}
removes the leading $H$-dependent approximation bias on the tube.
It is affine in $U$, so regularized DSM fitting remains convex.
\end{proposition}

\begin{proof}
At $z\in\mathcal M$, define the first normal space
$C_z=\operatorname{span}\{\mathrm{II}_z(e_i,e_j):1\leq i,j\leq k\}$.
Along a smooth curve, parallel transport an orthonormal tangent
frame $e_1,\ldots,e_k$. The Gauss formula shows that the ambient
derivatives of these frame vectors lie in $C_z$. The Weingarten
formula and $\nabla^\perp\mathrm{II}=0$ show that the ambient
derivatives of $\mathrm{II}(e_i,e_j)$ are tangent. Thus the
vectors spanning $T_z\mathcal M+C_z$ satisfy a linear differential
system within their span, so this subspace is constant along the
curve. Connectedness makes it a fixed subspace $E_0$ of dimension
at most $k+k(k+1)/2$. For any $z_0\in\mathcal M$, every tangent
vector lies in $E_0$, hence
$\mathcal M\subset z_0+E_0\subset E:=\operatorname{span}(z_0,E_0)$.
This proves the dimension bound.

Because $X_0\in E$, the Gaussian convolution factorizes exactly:
\[
 p_h(x)=p_h^E(u)(2\pi h)^{-(d-\ell)/2}
                 \exp\!\left(-\frac{\|v\|^2}{2h}\right).
\]
Here $p_h^E$ is the density of $a_hX_0+\sqrt h\,G_E$,
with $G_E\sim\mathcal N(0,I_E)$. Moreover,
$\|u+v-z'\|^2=\|u-z'\|^2+\|v\|^2$ for every
$z'\in\mathcal M$. Hence the nearest-point sets at $u+v$ in
$\mathbb R^d$ and at $u$ in $E$ coincide, and the reach in $E$
equals the ambient reach. The intrinsic geometry, the volume
measure, and the $C^4$ regularity and positivity of $p_0$ are
unchanged when $\mathcal M$ is regarded as a submanifold of $E$.
Proposition~\ref{prop:score_decomp}, applied in $E$, therefore gives
\begin{equation}
 \sup_{u\in\overline{U_{\tau/2}^{E}}}
 \left\|\nabla_u\log p_h^E(u)+\frac{u-z}{h}
 -\nabla_u\bigl[\log p_0(z)+H(z,u-z)\bigr]\right\|
 \leq C_*h,
 \label{eq:stage2_reduced_score_expansion}
\end{equation}
with a constant depending on the geometry and density in $E$,
but not on the extraneous $d-\ell$ Gaussian coordinates.
Therefore $s^*(x,t_2)=-v/h+\nabla_u\log p_h^E(u)$, regardless
of the size of $v$. On the tube,
\[
 D\Pi_{\mathcal M}^{E}(u)
 =Q_\nu\operatorname{Proj}_{T_z\mathcal M},\qquad
 \nabla_u\log p_0(z)=Q_\nu g(z).
\]
Substituting the derived score-level estimate
\eqref{eq:stage2_reduced_score_expansion} proves
\eqref{eq:stage2_curvature_decomposition}.

The geometric head can be written explicitly. Let
$z_T=\operatorname{Proj}_{T_z\mathcal M}z$ and $z_N=z-z_T$,
with the normal space taken inside $E$. Then
\begin{equation}
 \nabla_uH(z,\nu)
 =\frac12Q_\nu B_\nu z_T-\frac12z_N
  +\frac12\sum_{i,j=1}^k(Q_\nu)_{ij}\mathrm{II}_z(e_i,e_j).
 \label{eq:stage2_explicit_curvature_gradient}
\end{equation}
Indeed, moving the footpoint with velocity $w$ and transporting
$\nu$ normally in parallel gives ambient velocity $(I-B_\nu)w$.
In this variation the derivative of $-\langle\nu,z\rangle/2$
is $\langle B_\nu w,z_T\rangle/2$, whereas that of the log
determinant vanishes by parallelness of $\mathrm{II}$. A normal
variation $n$ has derivative
$-\langle n,z_N\rangle/2+\operatorname{Tr}(Q_\nu B_n)/2$.
These directional derivatives give the displayed formula.
The reach bound implies
$\|B_\nu\|_{\mathrm{op}}\leq1/2$,
$\|Q_\nu\|_{\mathrm{op}}\leq2$, and
$\|\nabla_uH\|\leq R+k/\tau$ on the closed tube.
Subtracting~\eqref{eq:stage2_curvature_decomposition} from
\eqref{eq:stage2_curvature_model} gives
\[
 s_U^{\rm geo}-s^*
 =Q_\nu\operatorname{Proj}_{T_z\mathcal M}
     [U\Phi(z)-g(z)]-e_h(u),
\]
which proves the claimed removal of the leading geometric term.
\end{proof}

\begin{theorem}[Intrinsic statistical term after curvature correction]
\label{thm:stage2_corrected_intrinsic}
Assume Proposition~\ref{prop:stage2_curvature_head} and the spectral
condition preceding Theorem~\ref{thm:stage2_intrinsic_main}.
Write $Z=\Pi_{\mathcal M}^{E}(u)$, using any measurable nearest-point
selection outside the tube.
Suppose the subprobability measure
$\mu_Z^{\rm in}(A)=\mathbb P(Z\in A,\ u\in U_{\tau/2}^{E})$
has density at most $\rho_+$ on $\mathcal M$.
Identify $E$ with $\mathbb R^\ell$ and take
$U\in\mathbb R^{\ell\times m}$. Draw independent feature pairs
$(v_j,b_j)$, independently of the data, with
$v_j\sim\mathcal N(0,I_E)$, and use
$\Phi_j(z)=m^{-1/2}\sigma(v_j^\top z+b_j)$ with
$|\sigma|\leq\kappa$ and $|\sigma'|\leq L_\sigma$.
Extend~\eqref{eq:stage2_curvature_model} by $-x/h$ outside
$U_{\tau/2}^{E}$. Put
\[
 A_xU=\mathbf1_{\rm in}\,Q_\nu
       \operatorname{Proj}_{T_z\mathcal M}U\Phi(z),\qquad
 Y_x=\begin{cases}
 (a_hX_0-z)/h-\nabla_uH(z,\nu),&u\in U_{\tau/2}^{E},\\
 a_hX_0/h,&u\notin U_{\tau/2}^{E}.
 \end{cases}
\]
Let $\widehat U$ minimize
$n^{-1}\sum_i\|A_{x_i}U-Y_{x_i}\|^2+\eta\|U\|_F^2$
for an independent sample, where $\eta>0$.
On the parameter Hilbert space $\mathbb R^{\ell\times m}$, define
\[
 C_A=\mathbb E[A_X^*A_X],\qquad
 \mathcal N_\eta^A=\operatorname{Tr}[C_A(C_A+\eta I)^{-1}],\qquad
 \ell_\eta^A=\operatorname*{ess\,sup}
       \|A_X(C_A+\eta I)^{-1/2}\|_{\mathrm{op}}^2.
\]
Let $U_\eta$ be the population ridge minimizer and set
$M_\eta=\operatorname*{ess\,sup}\|Y_X-A_XU_\eta\|$.
For $\delta\in(0,1)$, conditionally on the features, if
$n\geq10\ell_\eta^A\log(8\ell m/\delta)$, then with sample
probability at least $1-\delta/2$,
\begin{align}
 \mathcal R_h^{\rm geo}(\widehat U)
 &\leq\inf_U\{8a_m^{\rm dens}(U)+\eta\|U\|_F^2\}
       +2C_*^2h^2+\frac{R^2}{h^2}T_h^E\notag\\
 &\quad+\frac{16M_\eta^2\mathcal N_\eta^A\log(8/\delta)}{n}
       +\frac{128M_\eta^2\ell_\eta^A\log^2(8/\delta)}{n^2},
 \label{eq:stage2_corrected_intrinsic}
\end{align}
where
\[
 \mathcal R_h^{\rm geo}(U)=\mathbb E\|s_U^{\rm geo}(X,t_2)-s^*(X,t_2)\|^2,
 \quad a_m^{\rm dens}(U)=\mathbb E[\mathbf1_{\rm in}\|U\Phi(Z)-g(Z)\|^2],
 \quad T_h^E=\mathbb P(u\notin U_{\tau/2}^{E}).
\]
Moreover, with feature probability at least $1-\delta/2$,
simultaneously for every $\eta>0$,
\begin{equation}
 \mathcal N_\eta^A\leq\ell\min\left\{m,
  1+2(4D_{u_\delta}/\eta)^{k/(k+2)}\right\},\qquad
 \ell_\eta^A\leq4\kappa^2/\eta,
 \label{eq:stage2_corrected_capacity}
\end{equation}
where $u_\delta=\log(2/\delta)$ and
\[
 D_{u_\delta}=\rho_+C_W\operatorname{Vol}_{\mathcal M}(\mathcal M)
 \left[\kappa^2+L_\sigma^2
    \left(k+2\sqrt{ku_\delta/m}+2u_\delta/m\right)\right].
\]
Both estimates therefore hold jointly with probability at least
$1-\delta$ if $n\geq(40\kappa^2/\eta)\log(8\ell m/\delta)$.
\end{theorem}

\begin{proof}
Condition on the features and exact geometry. Set
$D_h^E(u)=\mathbb E[a_hX_0\mid u]$, so $\|D_h^E(u)\|\leq R$.
Independence of the Gaussian component in $E^\perp$ gives
\[
 \mathbb E[Y_X\mid X=x]=
 \begin{cases}
 (D_h^E(u)-z)/h-\nabla_uH(z,\nu),&u\in U_{\tau/2}^{E},\\
 D_h^E(u)/h,&u\notin U_{\tau/2}^{E}.
 \end{cases}
\]
The score error is exactly $A_XU-\mathbb E[Y_X\mid X]$.
Thus regression risk equals score risk plus the $U$-independent
quantity $\mathbb E\|Y_X-\mathbb E[Y_X\mid X]\|^2$.
On the tube, the preceding proposition and $\|Q_\nu\|\leq2$
give error at most $8\|U\Phi(z)-g(z)\|^2+2C_*^2h^2$.
Outside, the score error is $-D_h^E(u)/h$. Consequently
\begin{equation}
 \mathcal R_h^{\rm geo}(U)
 \leq8a_m^{\rm dens}(U)+2C_*^2h^2+R^2T_h^E/h^2.
 \label{eq:stage2_corrected_oracle}
\end{equation}
Since $d_{\mathcal M}^{E}(u)\leq(1-a_h)R+\sqrt h\|G_E\|$,
if $(1-a_h)R\leq\tau/4$ then exponential Markov gives
\[
 T_h^E\leq\min\{1,2^{\ell/2}e^{-\tau^2/(64h)}\}.
\]
For fixed geometry this tail is $o(h^4)$, so the geometric
remainder in~\eqref{eq:stage2_corrected_oracle} is $O(h^2)$.

For the sample error, write $C_\eta=C_A+\eta I$,
$C_n=n^{-1}\sum_iA_{x_i}^*A_{x_i}$, and
\[
 \xi_i=C_\eta^{-1/2}
  [A_{x_i}^*(Y_{x_i}-A_{x_i}U_\eta)-\eta U_\eta],
 \qquad V_n=n^{-1}\sum_i\xi_i.
\]
The population normal equation gives $\mathbb E\xi_i=0$.
The operator norm and Hilbert--Schmidt norm bounds give
\[
 \|\xi_i\|_F\leq2M_\eta\sqrt{\ell_\eta^A},\qquad
 \mathbb E\|\xi_i\|_F^2\leq M_\eta^2\mathcal N_\eta^A.
\]
For $Q_x=C_\eta^{-1/2}A_x^*A_xC_\eta^{-1/2}$ we have
$0\preceq Q_x\preceq\ell_\eta^AI$, $\mathbb EQ_x\preceq I$,
and $\mathbb E(Q_x-\mathbb EQ_x)^2\preceq\ell_\eta^AI$.
The same matrix Bernstein calculation as in
\eqref{eq:closed43_sample_covariance}, now in dimension $\ell m$,
shows that
$\|C_\eta^{-1/2}(C_n-C_A)C_\eta^{-1/2}\|\leq1/2$
except with probability $\delta/4$.
The Hilbert-space Bernstein bound used in
\eqref{eq:closed43_vector_concentration}, with $w=\log(8/\delta)$,
also gives
\[
 \|V_n\|_F\leq M_\eta
 \left(\sqrt{2\mathcal N_\eta^Aw/n}
              +4\sqrt{\ell_\eta^A}w/n\right)
\]
except with probability $\delta/4$.
On the covariance event the empirical normal equation implies
$\|C_\eta^{1/2}(\widehat U-U_\eta)\|_F\leq2\|V_n\|_F$.
Completion of the square therefore yields
\[
 \mathcal R_h^{\rm geo}(\widehat U)
 \leq\inf_U\{\mathcal R_h^{\rm geo}(U)+\eta\|U\|_F^2\}
       +16M_\eta^2\mathcal N_\eta^Aw/n
       +128M_\eta^2\ell_\eta^Aw^2/n^2.
\]
Together with~\eqref{eq:stage2_corrected_oracle}, this proves
\eqref{eq:stage2_corrected_intrinsic}. If $\ell_\eta^A=0$,
$A_X=0$ almost surely and the same conclusion holds directly.
Finiteness of the envelope follows from
\[
 \|Y_X\|\leq Y_{\max}:=2R/h+R+k/\tau,\qquad
 \|A_X\|\leq2\kappa,\qquad
 M_\eta\leq Y_{\max}(1+2\kappa/\sqrt\eta),
\]
where the last bound follows by comparing population ridge
regression with $U=0$.

Finally, put
$C_Z^{\rm in}=\mathbb E[\mathbf1_{\rm in}\Phi(Z)\Phi(Z)^\top]$
and $B_x=Q_\nu\operatorname{Proj}_{T_z\mathcal M}$ on the tube.
On $\mathbb R^m\otimes E$,
\[
 C_A=\mathbb E[\mathbf1_{\rm in}
       (\Phi(Z)\Phi(Z)^\top)\otimes(B_X^*B_X)]
 \preceq4C_Z^{\rm in}\otimes I_E.
\]
Operator monotonicity and the tensor spectrum imply
\begin{equation}
 \mathcal N_\eta^A\leq\ell\,\mathcal N_{\eta/4}(C_Z^{\rm in}).
 \label{eq:stage2_weighted_capacity}
\end{equation}
Apply the Gaussian quadratic-form estimate
\eqref{eq:closed43_gaussian} in $E$ to
$\int_{\mathcal M}\operatorname{Proj}_{T_z\mathcal M}
\,d\operatorname{Vol}_{\mathcal M}(z)$ with $u=u_\delta$.
With probability at least $1-e^{-u_\delta}$, the resulting
average $H^1$ feature norm is at most
$D_{u_\delta}/(\rho_+C_W)$.
The spectral approximation and rank-$N$ argument in
\eqref{eq:closed43_spectral_approximation}, applied to the
subprobability measure $\mu_Z^{\rm in}$, give
\[
 \mathcal N_{\eta/4}(C_Z^{\rm in})
 \leq\min\{m,\ 1+2(4D_{u_\delta}/\eta)^{k/(k+2)}\}
\]
simultaneously for every $\eta>0$.
Also $\ell_\eta^A\leq\|A_X\|_{\mathrm{op}}^2/\eta
\leq4\kappa^2/\eta$. This proves
\eqref{eq:stage2_corrected_capacity}; a union bound proves the
joint assertion. Since $\ell\leq1+k+k(k+1)/2$, the explicit
dimension in this statistical term is controlled by $k$.
\end{proof}

\subsection{Proof of Theorem~\ref{thm:full_sampling}}
\label{app:full_sampling_proof}
\label{app:hn_attention}

We use the notation and assumptions of Section~\ref{sec:end2end}.
All constants denoted by $C$ depend only on $k$, $g(k)$, and the
uniformly controlled geometric, density, and feature parameters
specified there. They do not depend on $d$, the noise levels, sample
sizes, or confidence parameter. Set $\gamma=k/(k+2)$.
We identify $E$ isometrically with $\mathbb R^r$, and write $p_h^E$
for the density of $a_hX_0+\sqrt h\,G_E$, where
$a_h=\sqrt{1-h}$ and $G_E\sim\mathcal N(0,I_E)$.

\label{app:full_sampling_construction}
Apply the Stage-1 dynamics of Theorem~\ref{thm:dimensional_collapse}
in $E$ at noise $h_1=h_-^{3/2}$, and let
$P:E\to B_E(0,R)$ be the clipped map at a fixed training time $s$.
Independently of Stage 1, draw the frozen features
$\Phi:E\to\mathbb R^m$ with the feature-law assumptions of
Theorem~\ref{thm:stage2_intrinsic_main}, interpreted in $E$.
Independently of both, draw $n_{\mathrm I}$ independent pairs
$(Z_i,G_i)\sim p_0\otimes\mathcal N(0,I_E)$.
For every $h\in[h_-,h_*]$, use these same pairs to form
\[
 X_i(h)=a_hZ_i+\sqrt h\,G_i,
 \qquad
 Y_i(h)=\frac{a_hZ_i-P(X_i(h))}{h}.
\]
Define
\begin{equation}
 \widehat U_h
 =\arg\min_{U\in\mathbb R^{r\times m}}
 \left\{
 \frac1{n_{\mathrm I}}\sum_{i=1}^{n_{\mathrm I}}
 \|U\Phi(P(X_i(h)))-Y_i(h)\|^2+h\|U\|_F^2
 \right\}.
 \label{eq:end2end_time_ridge}
\end{equation}
Thus $\eta(h)=h$, and
$\widehat s_E(u,h)=-(u-P(u))/h+\widehat U_h\Phi(P(u))$.
Writing $\varphi_i(h)=\Phi(P(X_i(h)))$, the unique minimizer is
\[
 \widehat U_h
 =
 \left(\frac1{n_{\mathrm I}}\sum_{i=1}^{n_{\mathrm I}}Y_i(h)\varphi_i(h)^\top\right)
 \left(\frac1{n_{\mathrm I}}\sum_{i=1}^{n_{\mathrm I}}
             \varphi_i(h)\varphi_i(h)^\top+hI_m\right)^{-1}.
\]
This formula defines a measurable family in $h$, and a continuous
family when $P$ is continuous. At each fixed $h$, the observations
have the required independent sampling law. Dependence across noise
levels will be handled by integrating expectations, without a union
bound over the continuum of times.

For the general sample sizes $n_{\mathrm I}$ and $n_{\mathrm H}$,
write $L=\log(16n_{\mathrm I}/h_-^3)$ and assume
\begin{equation}
 m\ge\frac{CL}{h_-},\qquad
 n_{\mathrm I}\ge\frac{C\kappa^2}{h_-}\log\frac{16mn_{\mathrm I}}{h_-^3}.
 \label{eq:full_sampling_samples}
\end{equation}

\begin{proof}[Proof of the high-noise estimate~\eqref{eq:hn_integrated}]
Fix $t\ge t_{\max}$, and abbreviate $a=a_t$ and $h=h(t)$.
For $z\in\mathbb R^d$, let
$g_z(x)=(2\pi h)^{-d/2}\exp(-\|x-az\|^2/(2h))$.
Let $z_1,\ldots,z_{n_{\mathrm H}}$ denote the high-noise clean samples, and set
$p(x)=\mathbb E_Zg_Z(x)=p_t(x)$ and
$p_{n_{\mathrm H}}(x)=n_{\mathrm H}^{-1}\sum_{i=1}^{n_{\mathrm H}}g_{z_i}(x)$.
The Gaussian-mixture identity gives
\[
 s^*(x,t)=-\frac{x}{h}+\frac{a}{h}m(x),
 \qquad
 s^{\mathrm{HN}}(x,t)=-\frac{x}{h}+\frac{a}{h}m_{n_{\mathrm H}}(x),
\]
where
\[
 m(x)=\frac{\mathbb E[Zg_Z(x)]}{p(x)},
 \qquad
 m_{n_{\mathrm H}}(x)=
 \frac{\sum_{i=1}^{n_{\mathrm H}}z_i g_{z_i}(x)}
      {\sum_{i=1}^{n_{\mathrm H}}g_{z_i}(x)}.
\]
The common Gaussian factor cancels in the latter ratio, yielding
the attention weights in~\eqref{eq:hn_attention}.

For fixed $x$, let $A_x=\{p_{n_{\mathrm H}}(x)\ge p(x)/2\}$.
The random vectors $(z_i-m(x))g_{z_i}(x)$ are independent and
centered, with $\|z_i-m(x)\|\le2R$. Consequently,
\[
 \mathbb E_{\mathrm{HN}}\!
 \left[\|m_{n_{\mathrm H}}(x)-m(x)\|^2\mathbf 1_{A_x}\right]
 \le\frac{16R^2\,\mathbb E_Zg_Z(x)^2}{n_{\mathrm H}p(x)^2}.
\]
Chebyshev's inequality gives
$\mathbb P_{\mathrm{HN}}(A_x^c)
\le4\mathbb E_Zg_Z(x)^2/(n_{\mathrm H}p(x)^2)$.
Since both means lie in $B(0,R)$, the contribution on $A_x^c$
has the same upper bound. Therefore
\begin{equation}
 \mathbb E_{\mathrm{HN}}\|m_{n_{\mathrm H}}(x)-m(x)\|^2
 \le
 \frac{32R^2\,\mathbb E_Zg_Z(x)^2}{n_{\mathrm H}p(x)^2}.
 \label{eq:hn_ratio_bound}
\end{equation}

We bound the integral of this density ratio using
\eqref{eq:end2end_covering}. For $0<\rho\le1$, choose a cover
of $\mathcal M$ by $N(\mathcal M,\rho)$ Euclidean balls of radius
$\rho$, and partition $\mathcal M$ measurably into cells contained
in these balls. Discard the cells of zero $p_0$-mass.
For cell $j$, let $w_j>0$ be its mass,
$\mu_j=\mathbb E[Z\mid j]$, and $p_j(x)=\mathbb E[g_Z(x)\mid j]$.
Because $p(x)\ge w_jp_j(x)$,
\begin{equation}
 \begin{aligned}
 I_h
 &:=\int_{\mathbb R^d}
       \frac{\mathbb E_Zg_Z(x)^2}{p(x)}\,dx\\
 &\le
 \sum_j\int_{\mathbb R^d}
       \frac{\mathbb E[g_Z(x)^2\mid j]}{p_j(x)}\,dx.
 \end{aligned}
 \label{eq:hn_cover_partition}
\end{equation}
If the ball containing cell $j$ has center $c_j$, then
\[
 \mathbb E[\|Z-\mu_j\|^2\mid j]
 \le \mathbb E[\|Z-c_j\|^2\mid j]\le\rho^2,
 \qquad
 \|z-\mu_j\|\le2\rho \quad (z\text{ in cell }j).
\]
Jensen's inequality gives
\[
 p_j(x)\ge g_{\mu_j}(x)
 \exp\!\left(
 -\frac{a^2}{2h}\mathbb E[\|Z-\mu_j\|^2\mid j]\right).
\]
Completing the Gaussian square gives the exact identity
\[
 \int_{\mathbb R^d}\frac{g_z(x)^2}{g_{\mu_j}(x)}\,dx
 =\exp\!\left(\frac{a^2\|z-\mu_j\|^2}{h}\right).
\]
Substituting these two estimates into
\eqref{eq:hn_cover_partition} yields
\[
 I_h\le
 N(\mathcal M,\rho)
 \exp\!\left(\frac{9a^2\rho^2}{2h}\right).
\]
Taking $\rho=\sqrt h$ and using $a^2\le1$ gives
\begin{equation}
 I_h\le e^{9/2}A_{\mathrm{cov}}h^{-k/2}.
 \label{eq:hn_intrinsic_ratio}
\end{equation}
This argument uses only the covering assumption and bounded support;
it does not require the subspace condition.

Multiplying~\eqref{eq:hn_ratio_bound} by $p(x)a^2/h^2$ and
integrating over $x$ now gives
\[
 \mathbb E_{\mathrm{HN}}\mathbb E_{X\sim p_t}
 \|s^{\mathrm{HN}}(X,t)-s^*(X,t)\|^2
 \le
 \frac{32R^2e^{9/2}A_{\mathrm{cov}}}{n_{\mathrm H}}
 \frac{a_t^2}{h(t)^{2+k/2}}.
\]
For the VP schedule, $dh=\beta(t)a_t^2\,dt$. Tonelli's theorem
therefore implies
\[
 \begin{aligned}
 \mathbb E_{\mathrm{HN}}\mathcal E_{\mathrm{HN}}
 &\le
 \frac{32R^2e^{9/2}A_{\mathrm{cov}}}{n_{\mathrm H}}
 \int_{h_*}^{h(T)}h^{-2-k/2}\,dh\\
 &\le
 \frac{C_{\mathrm{HN}}}{n_{\mathrm H}h_*^{1+k/2}},
 \qquad
 C_{\mathrm{HN}}=
 \frac{32R^2e^{9/2}A_{\mathrm{cov}}}{1+k/2}.
 \end{aligned}
\]
This proves~\eqref{eq:hn_integrated}. A separate application of
Markov's inequality gives its high-probability version; for the
end-to-end bound below we instead combine the two phases before
applying Markov.
\end{proof}

\begin{proof}[Proof of Theorem~\ref{thm:full_sampling}]
We condition throughout on the fixed Stage-1 map after its specified
training threshold.

For $x=u+v$, where $u=Qx\in E$ and $v=(I-Q)x\in E^\perp$,
the noisy density factorizes as
\[
 p_t(x)
 =p_{h(t)}^E(u)(2\pi h(t))^{-(d-r)/2}
       \exp\!\left(-\frac{\|v\|^2}{2h(t)}\right).
\]
It follows that
\[
 s^*(x,t)
 =\nabla_E\log p_{h(t)}^E(u)-\frac{v}{h(t)}.
\]
The $E^\perp$ component of~\eqref{eq:end2end_low_score} is exact.
Writing
\[
 \mathcal R_h
 :=\mathbb E_{X\sim p_h^E}
       \|\widehat s_E(X,h)-\nabla_E\log p_h^E(X)\|^2,
\]
we consequently have
\begin{equation}
 \mathcal E_{\mathrm I}
 :=
 \int_{h_-}^{h_*}\frac{\mathcal R_h}{1-h}\,dh.
 \label{eq:end2end_reduced_risk}
\end{equation}
In particular, no error contribution arises from the $d-r$
orthogonal Gaussian coordinates.

Let $K_E=\overline{U_{\tau/2}^E}$ be the closed tube around
$\mathcal M$ inside $E$. For the fixed interval $[h_-,h_*]$,
joint continuity and positivity of the noisy densities imply
\[
 \overline K
 :=
 \sup_{\substack{h\in[h_-,h_*]\\u\in K_E}}
 \frac{p_h^E(u)}{p_{h_1}^E(u)}
 <\infty,
 \qquad h_1=h_-^{3/2}.
\]
Let $J_{h_1}(q_s)$ denote the reduced Stage-1 objective
\eqref{eq:normalized_dsm}. By
Theorem~\ref{thm:dimensional_collapse}, there is a finite $s_0$
such that $\overline KJ_{h_1}(q_s)\le h_-^3$ for every
$s\ge s_0$.
The uniform denoiser expansion from
Proposition~\ref{prop:score_decomp}, applied in $E$, gives
$\sup_{K_E}\|D_{h_1}^E-\Pi_{\mathcal M}^E\|\le C h_1$,
where $D_h^E(u)=u+h\nabla_E\log p_h^E(u)\in B_E(0,R)$.
As in the noise-transfer step of
Appendix~\ref{app:stage2_detailed_bounds}, clipping and change of
measure therefore yield, for every $h\in[h_-,h_*]$,
\begin{equation}
 \begin{aligned}
 \mathbb E_{p_h^E}
   \|P(X)-\Pi_{\mathcal M}^E(X)\|^2
 &\le
 4\overline KJ_{h_1}(q_s)+C h_1^2\\
 &\quad+4R^2\mathbb P_{p_h^E}(X\notin K_E)
 \le C h^3.
 \end{aligned}
 \label{eq:end2end_uniform_projection}
\end{equation}
Indeed, $h_1^2=h_-^3\le h^3$. Moreover, the reduced Gaussian
tube-tail estimate is
\[
 \mathbb P_{p_h^E}(X\notin K_E)
 \le 2^{r/2}\exp\!\left(-\frac{\tau^2}{64h}\right)
 \le h^3
\]
for all sufficiently small $h$, after ensuring
$(1-\sqrt{1-h})R\le\tau/4$.
The threshold depends on $r\le g(k)$ and the controlled parameters,
but not on $d$. The projection outside the tube is any fixed
measurable nearest-point selection.

After its transfer step, the proof of
Theorem~\ref{thm:stage2_intrinsic_main} uses Stage 1 only through
the projection-risk bound~\eqref{eq:end2end_uniform_projection}.
Its remaining steps thus apply to the shared map $P$ at every fixed
$h\in[h_-,h_*]$.
Set
\[
 \delta_0=\frac{h_-^3}{n_{\mathrm I}},
 \qquad
 L=\log\frac{16}{\delta_0}
   =\log\frac{16n_{\mathrm I}}{h_-^3}.
\]
The width and sample requirements
\eqref{eq:full_sampling_samples} imply their fixed-noise
counterparts for all $h\ge h_-$.
Consequently, for each fixed $h$, with probability at least
$1-\delta_0$ over the features and the Stage-2 pairs,
\begin{equation}
 \mathcal R_h
 \le C_0h+\frac{C_1L}{n_{\mathrm I}h^{2+\gamma}}.
 \label{eq:end2end_pointwise_learning}
\end{equation}
The constants are uniform in $h$ under the assumptions of the
theorem. This statement does not claim a simultaneous
high-probability event over all $h$.

We also need a deterministic bound on the exceptional event.
Since $\|Y_i(h)\|\le2R/h$, comparison with $U=0$ in
\eqref{eq:end2end_time_ridge} gives
$\|\widehat U_h\|_F\le2Rh^{-3/2}$.
The bounds $\|\Phi\|\le\kappa$ and $\|D_h^E\|\le R$ then imply
\[
 \mathcal R_h\le M_0h^{-3},
 \qquad M_0=8R^2(1+\kappa^2).
\]
Writing $x_+=\max\{x,0\}$ and combining this with
\eqref{eq:end2end_pointwise_learning} yields
\begin{equation}
 \mathbb E_{\mathrm{train}}(\mathcal R_h-C_0h)_+
 \le
 \frac{C_1L}{n_{\mathrm I}h^{2+\gamma}}
 +\frac{M_0h_-^3}{n_{\mathrm I}h^3}.
 \label{eq:end2end_positive_excess}
\end{equation}
Here $\mathbb E_{\mathrm{train}}$ averages over the Stage-2
features and pairs, conditionally on Stage 1.

The measurable ridge construction permits Tonelli's theorem.
Define
\[
 V_{\mathrm I}
 :=
 \int_{h_-}^{h_*}\frac{(\mathcal R_h-C_0h)_+}{1-h}\,dh.
\]
Since $h_*\le1/2$, integration of
\eqref{eq:end2end_positive_excess} gives
\begin{equation}
 \begin{aligned}
 \mathbb E_{\mathrm{train}}V_{\mathrm I}
 &\le C\left(
       \frac{L}{n_{\mathrm I}h_-^{1+\gamma}}+\frac{h_-}{n_{\mathrm I}}
       \right)
 \le\frac{CL}{n_{\mathrm I}h_-^{1+\gamma}},\\
 \mathcal E_{\mathrm I}
 &\le C h_*^2+V_{\mathrm I}.
 \end{aligned}
 \label{eq:end2end_integrated_learning}
\end{equation}
Thus the deterministic geometric bias remains separate from the
random excess error.

Combining~\eqref{eq:hn_integrated} with
\eqref{eq:end2end_integrated_learning} and applying Markov's
inequality once to $V_{\mathrm I}+\mathcal E_{\mathrm{HN}}$
gives, with probability at least $1-\delta$,
\begin{equation}
 \mathcal E_{\mathrm I}+\mathcal E_{\mathrm{HN}}
 \le
 C h_*^2
 +\frac{CL}{\delta n_{\mathrm I}h_-^{1+\gamma}}
 +\frac{C_{\mathrm{HN}}}{\delta n_{\mathrm H}h_*^{1+k/2}}.
 \label{eq:end2end_joint_score_error}
\end{equation}
Independence between these two nonnegative errors is not needed
for this step.

Couple
$X_{t_{\min}}=a_{t_{\min}}X_0+\sqrt{h_-}\,G$
with $X_0$, where $G\sim\mathcal N(0,I_d)$ is independent of
$X_0$. Since $X_0\in B_E(0,R)$, convex projection onto this
ball cannot increase distance to $X_0$. The map
$D$ in~\eqref{eq:end2end_terminal} therefore satisfies
\begin{equation}
 \begin{aligned}
 \varepsilon_{\mathrm{end}}^2
 &:=W_2(p_0,D_\#p_{t_{\min}})^2\\
 &\le
 \mathbb E\|D(X_{t_{\min}})-X_0\|^2\\
 &\le \frac{h_-}{1-h_-}\mathbb E\|QG\|^2
 =\frac{rh_-}{1-h_-}
 \le\frac{g(k)h_-}{1-h_-}.
 \end{aligned}
 \label{eq:end2end_terminal_error}
\end{equation}
No true manifold projection or curvature is required at the
terminal step.

Write $q_T=\mathcal N(0,h(T)I_d)$.
The VP marginal is a mixture of $\mathcal N(a_Tz,h(T)I_d)$
with $z\sim p_0$, so convexity of relative entropy gives
\[
 \operatorname{KL}(p_T\|q_T)
 \le
 \mathbb E_{Z\sim p_0}
 \operatorname{KL}\!\left(
   \mathcal N(a_TZ,h(T)I_d)\middle\|q_T\right)
 \le\frac{a_T^2R^2}{2h(T)}.
\]
Let $\mathbb P$ be the true reverse path law on
$[t_{\min},T]$, and let $\widehat{\mathbb P}$ be the reverse path
law driven by~\eqref{eq:full_score} and initialized at $q_T$.
The chain rule for relative entropy and the assumed Girsanov
formula imply
\begin{equation}
 \operatorname{KL}(\mathbb P\|\widehat{\mathbb P})
 \le
 \frac12\left(
 \frac{a_T^2R^2}{h(T)}
 +\mathcal E_{\mathrm I}+\mathcal E_{\mathrm{HN}}\right).
 \label{eq:end2end_path_entropy}
\end{equation}
Data processing and Pinsker's inequality bound the total
variation between $D_\#p_{t_{\min}}$ and
$p_{\mathrm{gen}}=D_\#p_{\mathrm{gen},t_{\min}}$ by
$\sqrt{\operatorname{KL}(\mathbb P\|\widehat{\mathbb P})/2}$.
Both distributions are supported in $B_E(0,R)$.
Maximal coupling thus yields
\[
 W_2(D_\#p_{t_{\min}},p_{\mathrm{gen}})
 \le
 2R\left(
   \frac{\operatorname{KL}(\mathbb P\|\widehat{\mathbb P})}{2}
   \right)^{1/4}.
\]
The triangle inequality, together with
\eqref{eq:end2end_joint_score_error}--\eqref{eq:end2end_path_entropy},
gives the general bound
\begin{equation}
 \begin{aligned}
 W_2(p_0,p_{\mathrm{gen}})
 &\leq \sqrt{\frac{g(k)h_-}{1-h_-}}\\
 &\quad+2R\Bigg[\frac14\Bigg\{
       \frac{a_T^2R^2}{h(T)}+Ch_*^2\\
 &\qquad\qquad
       +\frac{CL}{\delta n_{\mathrm I}h_-^{1+k/(k+2)}}
       +\frac{C_{\mathrm{HN}}}{\delta n_{\mathrm H}h_*^{1+k/2}}
       \Bigg\}\Bigg]^{1/4}.
 \end{aligned}
 \label{eq:full_sampling_bound}
\end{equation}

Take $n_{\mathrm I}=n_{\mathrm H}=n$ and choose
\begin{equation}
 h_*=\left(\frac{\log n}{n\delta}\right)^{2/(k+6)},\qquad
 h_-=\frac{h_*}{2},\qquad \frac{a_T^2R^2}{h(T)}\le h_*^2
 \label{eq:full_sampling_parameters}
\end{equation}
as in the theorem. Write
$p=1+k/(k+2)$ and $q=1+k/2$. Since $k\ge1$,
\[
 q-p=\frac{k^2}{2(k+2)}>0,
 \qquad
 h_*^{q+2}=\frac{\log n}{n\delta}.
\]
For $n\ge2$ with $n\delta/\log n$ sufficiently large,
$h_-=h_*/2$ and $L=O_k(\log n)$.
As $h_*\le1$, the low-noise statistical term satisfies
\[
 \frac{L}{n\delta h_-^p}
 \le C_k\frac{\log n}{n\delta h_*^q}
 =C_kh_*^2.
\]
The high-noise term is also $O(h_*^2)$, as are the geometric
bias and the assumed initialization error.
The terminal term in~\eqref{eq:end2end_terminal_error} is
$O(\sqrt{h_*})$. Substitution into
\eqref{eq:full_sampling_bound} proves
\eqref{eq:full_sampling_rate}.
\end{proof}

\label{app:end2end_subspace_recovery}
Availability of $E$ can be obtained from a separate sample under
an additional nondegeneracy condition.
Let $E=\operatorname{span}(\operatorname{supp}p_0)$ and suppose
\[
 \left.\mathbb E[X_0X_0^\top]\right|_E
 \succeq\lambda_0 I_E,
 \qquad \lambda_0>0.
\]
Draw $n_0$ fresh independent clean samples
$Z_1^0,\ldots,Z_{n_0}^0$ and let
$\widehat E=\operatorname{span}\{Z_1^0,\ldots,Z_{n_0}^0\}$.
Each reduced outer product is positive semidefinite and bounded
by $R^2I_E$. Matrix Chernoff~\citep{tropp2012user} therefore gives
\[
 \mathbb P(\widehat E\ne E)
 \le
 \mathbb P\!\left(
 \lambda_{\min}\!\left[
    \left.\frac1{n_0}\sum_{i=1}^{n_0}
                    Z_i^0(Z_i^0)^\top\right|_E
 \right]\le\frac{\lambda_0}{2}\right)
 \le r\exp\!\left(-\frac{n_0\lambda_0}{8R^2}\right).
\]
Hence
$n_0\ge(8R^2/\lambda_0)\log(2r/\delta)$
ensures exact recovery with probability at least $1-\delta/2$.
Conditioning on this event and applying
Theorem~\ref{thm:full_sampling} with confidence parameter
$\delta/2$ gives an overall $1-\delta$ guarantee.
This extension uses noiseless clean samples and a lower bound
on $\lambda_0$ independent of $d$; it adds an explicit
preprocessing step to the stated construction.

\section{Additional experimental results}
\label{app:extra-experiments}

\subsection{Stage-1 noise scale}
\label{app:stage1-noise}

Table~\ref{tab:stage1_noise_sweep} sweeps the Stage-1 corruption scale on
Stacked MNIST with everything else held fixed: each arm trains its own
Stage-1 map at the stated $\sigma_{\mathrm{geom}}$ and its own Stage-2 score
network on the common interval $t_2\in[0.001,0.999]$, at three seeds (42, 43,
44). Across the three additive noise levels, mean FID stays within $7.95$--$8.20$, and the paired per-seed differences against the
$\sigma_{\mathrm{geom}}=0$ arm are not significant at either level
($\lvert t\rvert \le 0.96$, $n=3$), with the sign flipping across seeds. The
KL-regularized arm differs from all of them by $8$ FID and by a factor $1.8$
in reconstruction error. Within this range, FID is thus insensitive to the Stage-1
noise scale. Figure~\ref{fig:stacked_family_samples} shows the corresponding
uncurated samples.

We also ran a separate check, not shown in
Table~\ref{tab:stage1_noise_sweep}. For the Stacked
MNIST experiments we replace the VP corruption of
Algorithm~\ref{alg:practical-sild} by additive constant-$\sigma$ corruption;
this check restores the VP form at the same nominal $t_0=5\!\times\!10^{-3}$,
which induces an effective standard deviation of $2.74\!\times\!10^{-2}$,
$5.5\times$ larger than the additive arm at the same $t_0$. It therefore
varies the noise \emph{form} and its \emph{magnitude} at once, and we report
it only to confirm that the substitution is immaterial: FID
$8.271 \pm 0.031$, paired $t=+0.74$ against $\sigma_{\mathrm{geom}}=0$.

\begin{table}[h]
\centering
\caption{%
Stacked MNIST Stage-1 noise sweep. Every arm is separately trained and uses
the same CNN backbones, latent score architecture, training budget, and
Stage-2 interval $t_2\in[0.001,0.999]$. Mean$\pm$population std over seeds
42, 43, 44; FID on 10K samples, mode coverage on 25{,}600 samples. The block
varies the additive Stage-1 noise scale.%
}
\label{tab:stage1_noise_sweep}
\small
\begin{tabular}{lcccc}
\toprule
Stage-1 objective & $\sigma_{\mathrm{geom}}$ & \#modes & Recon MSE ($\times10^{-3}$) $\downarrow$ & FID $\downarrow$   \\
\midrule
$\beta$-VAE, KL-reg.\ (LDM-CNN)      & ---                   & 1000 & $6.284 \pm 0.052$ & $16.152 \pm 0.098$   \\
\midrule
SiLD, clean AE ($\sigma\!\to\!0$)    & $0$                   & 1000 & $3.542 \pm 0.036$ & $8.199 \pm 0.137$    \\
SiLD                                 & $1\!\times\!10^{-3}$  & 1000 & $3.533 \pm 0.014$ & $8.187 \pm 0.102$   \\
SiLD                                 & $5\!\times\!10^{-3}$  & 1000 & $3.539 \pm 0.017$ & $7.951 \pm 0.262$ \\
\bottomrule
\end{tabular}
\end{table}

\begin{figure}[h]
\centering
\includegraphics[width=\textwidth]{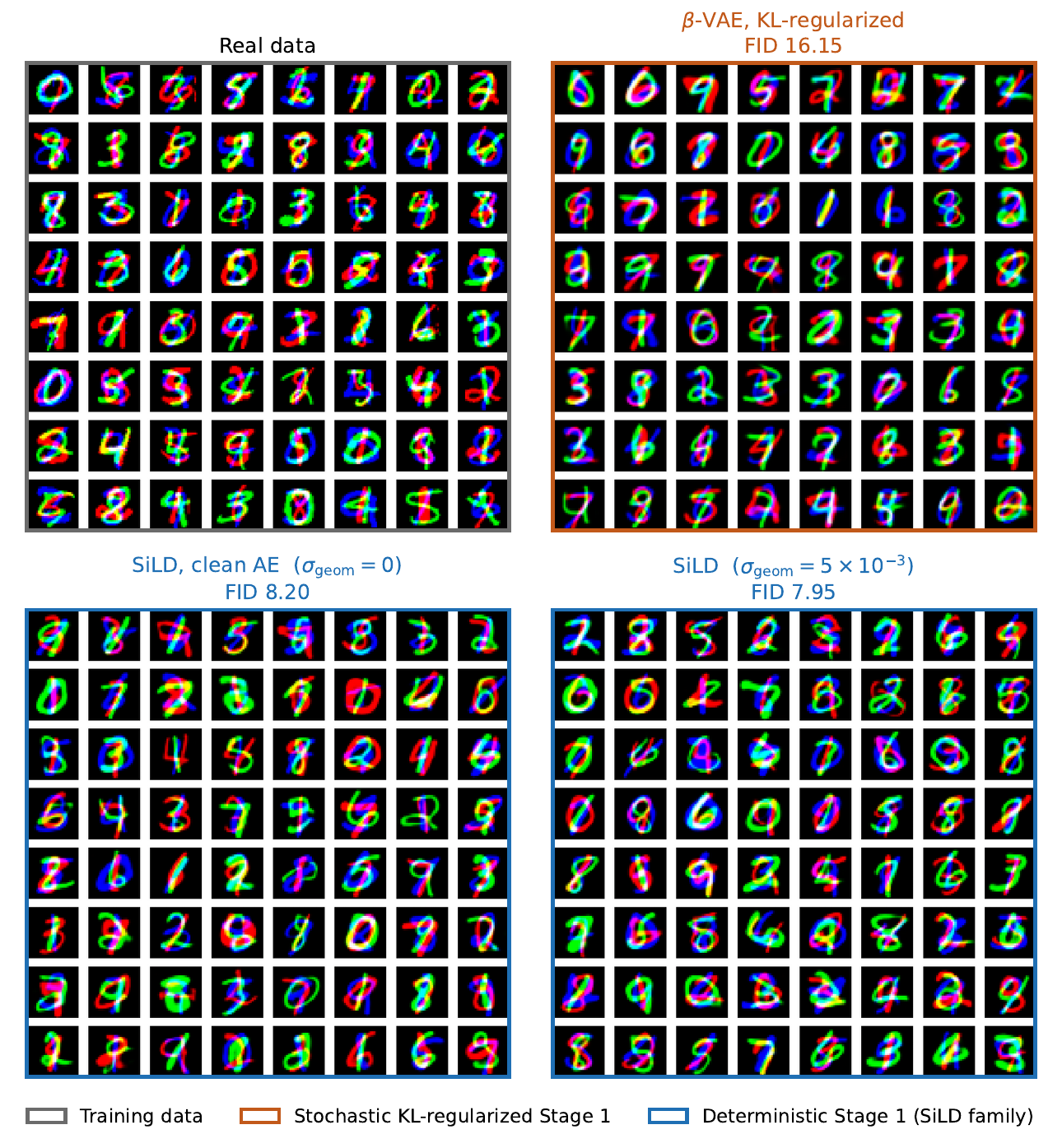}
\caption{%
Uncurated Stacked MNIST samples for the three arms of
Table~\ref{tab:stacked_mnist}, plus real data (seed 42; each image is three
random MNIST digits stacked into the R/G/B channels, so there are 1000
possible modes). All panels come from runs that share the CNN backbones, the
latent score architecture, the training budget, and the Stage-2 interval
$t_2\in[0.001,0.999]$; only the Stage-1 objective differs. Every generative
arm attains full $1000/1000$ mode coverage, so none of the panels shows mode
dropping, and the qualitative differences are modest: the figure is included
to show that the clean autoencoder and SiLD are visually interchangeable
while both differ from the KL-regularized arm, which is the claim
Table~\ref{tab:stage1_noise_sweep} makes quantitatively.
}
\label{fig:stacked_family_samples}
\end{figure}

\subsection{The Stage-1 objective family on molecules}
\label{app:stage1-family-mol}

Table~\ref{tab:stage1_family_mol} repeats the comparison of
Appendix~\ref{app:stage1-noise} on the molecular benchmarks. The degenerate
$\sigma\!\to\!0$ member of the family --- a clean deterministic autoencoder ---
matches SiLD on all four datasets to within $0.43$ FCD and $0.19$ percentage
points of uniqueness, with the sign of the difference inconsistent across
datasets and metrics. Each molecular cell is a single seed, so we do not
claim an ordering between the two deterministic arms; on Stacked MNIST, where
we do have three seeds, the corresponding paired difference is not
significant ($t=-0.96$). Both deterministic arms are separated from the
KL-regularized bottleneck by two orders of magnitude of uniqueness on HIV and
PCBA, and SiLD's autoencoder reaches $4.2$--$8.8\times$ lower reconstruction
error than the KL-regularized one, with the largest factor on QM9, where the
two generative models are otherwise closest.

Within the range tested, the Stage-1 noise scale therefore has no detectable
effect: the clean autoencoder and SiLD perform comparably, and both outperform
the tested KL-regularized baseline on some of the benchmarks. These
experiments do not separate determinism from the absence of KL
regularization, which change together across our arms, and the analysis does
not rule out stochastic bottlenecks in general.
Appendix~\ref{app:collapse} examines how the KL-regularized arm fails.

\begin{table}[h]
\centering
\caption{%
The Stage-1 objective family on molecules, at a common Stage-2 interval
$t_2\in[0.001,0.999]$ for all three arms (10K samples, seed 42, SELFIES;
validity is $100\%$ throughout). Encoder, decoder, and latent score network
are identical across rows; only the Stage-1 objective differs.
$\sigma\!\to\!0$ denotes the degenerate member of the family: an ordinary
clean deterministic autoencoder, with no input corruption and no KL term.
This family comparison needs its own Stage-2 interval because the paired rule
$t_2\in[t_0,0.999]$ used in Table~\ref{tab:main_results} degenerates at
$t_0=0$, so all arms are instead fixed at $t_{\min}=0.001$; the SiLD and
KL-regularized numbers therefore differ in the third digit from
Table~\ref{tab:main_results}. Bold marks the numerically best value per
dataset and metric, but every cell is a single seed and the two deterministic
arms trade wins at a margin we do not treat as meaningful.%
}
\label{tab:stage1_family_mol}
\resizebox{\textwidth}{!}{%
\begin{tabular}{llcccccc}
\toprule
\textbf{Dataset} & \textbf{Stage-1 objective}
& \textbf{Recon MSE$\downarrow$} & \textbf{Uniq.$\uparrow$}
& \textbf{Novelty$\uparrow$} & \textbf{Int.\ Div.$\uparrow$}
& \textbf{QED$\uparrow$} & \textbf{FCD$\downarrow$} \\
\midrule
\multirow{4}{*}{\textbf{QM9}}
 & Real data                               & --                      & 96.18\% & 0.0\%   & 0.919 & 0.466 & 5.05 \\
\cmidrule{2-8}
 & $\beta$-VAE, KL-regularized (LDM)       & $1.31\!\times\!10^{-3}$ & 96.94\% & 37.59\% & 0.915 & 0.473 & \textbf{4.30} \\
 & \textbf{SiLD} ($\sigma\!\to\!0$; clean AE) & $1.63\!\times\!10^{-4}$ & 97.43\% & 34.94\% & 0.917 & 0.471 & 4.41 \\
 & \textbf{SiLD} ($t_0=0.05$)              & $\mathbf{1.49\!\times\!10^{-4}}$ & \textbf{97.58\%} & 36.25\% & 0.917 & 0.468 & 4.39 \\
\midrule
\multirow{4}{*}{\textbf{HIV}}
 & Real data                               & --                      & 88.93\% & 0.0\%   & 0.898 & 0.513 & 3.24 \\
\cmidrule{2-8}
 & $\beta$-VAE, KL-regularized (LDM)       & $1.39\!\times\!10^{-3}$ & 0.99\%  & 100.0\% & 0.012 & 0.114 & 60.96 \\
 & \textbf{SiLD} ($\sigma\!\to\!0$; clean AE) & $3.28\!\times\!10^{-4}$ & \textbf{98.82\%} & 99.95\% & 0.904 & 0.430 & \textbf{10.29} \\
 & \textbf{SiLD} ($t_0=0.05$)              & $\mathbf{3.24\!\times\!10^{-4}}$ & 98.63\% & 99.95\% & 0.903 & 0.434 & 10.44 \\
\midrule
\multirow{4}{*}{\textbf{MUV}}
 & Real data                               & --                      & 95.03\% & 0.0\%   & 0.860 & 0.675 & 4.52 \\
\cmidrule{2-8}
 & $\beta$-VAE, KL-regularized (LDM)       & $6.55\!\times\!10^{-3}$ & 95.30\% & 99.98\% & 0.812 & 0.238 & 41.20 \\
 & \textbf{SiLD} ($\sigma\!\to\!0$; clean AE) & $1.60\!\times\!10^{-3}$ & \textbf{99.45\%} & 99.96\% & 0.892 & 0.513 & 19.85 \\
 & \textbf{SiLD} ($t_0=0.05$)              & $\mathbf{1.55\!\times\!10^{-3}}$ & 99.42\% & 99.97\% & 0.891 & 0.512 & \textbf{19.42} \\
\midrule
\multirow{4}{*}{\textbf{PCBA}}
 & Real data                               & --                      & 98.91\% & 0.0\%   & 0.861 & 0.637 & 9.12 \\
\cmidrule{2-8}
 & $\beta$-VAE, KL-regularized (LDM)       & $2.34\!\times\!10^{-3}$ & 0.57\%  & 94.74\% & 0.016 & 0.090 & 77.10 \\
 & \textbf{SiLD} ($\sigma\!\to\!0$; clean AE) & $5.41\!\times\!10^{-4}$ & \textbf{98.81\%} & 99.76\% & 0.895 & 0.483 & 18.04 \\
 & \textbf{SiLD} ($t_0=0.05$)              & $\mathbf{5.33\!\times\!10^{-4}}$ & 98.63\% & 99.78\% & 0.894 & 0.486 & \textbf{17.96} \\
\bottomrule
\end{tabular}%
}
\end{table}

\subsection{Stage-2 lower time bound}
\label{app:stage2-tmin}

Table~\ref{tab:stage2_tmin} varies the Stage-2 lower time bound on the
molecular benchmarks. This ablation does \emph{not} retrain Stage~1: a single
Stage-1 checkpoint (trained at $t_0=0.05$) is reused in every row, and only
the Stage-2 score network is retrained with the stated lower bound. It
therefore isolates the Stage-2 time range, not the Stage-1 noise scale, for
which Appendix~\ref{app:stage1-noise} is the relevant experiment. FCD varies by at most $1.21$ (MUV), and by $0.15$ and $0.19$ on QM9 and HIV,
respectively.

\begin{table}[h]
\centering
\caption{%
Stage-2 lower time bound $t_{\min}$ on the molecular benchmarks, with Stage~1
held fixed at a single checkpoint trained at $t_0=0.05$. 10K samples, 200
sampling steps, seed 42.%
}
\label{tab:stage2_tmin}
\small
\begin{tabular}{lcccccccc}
\toprule
& \multicolumn{2}{c}{\textbf{QM9}} & \multicolumn{2}{c}{\textbf{HIV}}
& \multicolumn{2}{c}{\textbf{MUV}} & \multicolumn{2}{c}{\textbf{PCBA}} \\
\cmidrule(lr){2-3}\cmidrule(lr){4-5}\cmidrule(lr){6-7}\cmidrule(lr){8-9}
$t_{\min}$ & Uniq. & FCD & Uniq. & FCD & Uniq. & FCD & Uniq. & FCD \\
\midrule
$0.01$ & 97.44\% & 4.32 & 98.81\% & 10.35 & 99.35\% & 19.14 & 98.64\% & 17.95 \\
$0.02$ & 97.69\% & 4.47 & 98.76\% & 10.34 & 99.35\% & 19.23 & 98.70\% & 17.89 \\
$0.05$ & 97.23\% & 4.43 & 98.68\% & 10.32 & 99.31\% & 19.35 & 98.65\% & 17.87 \\
$0.10$ & 97.21\% & 4.33 & 98.65\% & 10.51 & 99.48\% & 18.86 & 98.93\% & 17.67 \\
$0.20$ & 97.29\% & 4.44 & 98.86\% & 10.44 & 99.38\% & 18.14 & 98.94\% & 17.08 \\
\bottomrule
\end{tabular}
\end{table}

\subsection{Where the KL-regularized arm fails}
\label{app:collapse}

To locate the failure we decompose the pipeline into three decodes that share
one trained model: \textbf{(A)} \emph{round trip} --- encode real test
molecules and decode, which exercises Stage~1 alone; \textbf{(B)}
\emph{generation} --- the full reverse latent diffusion followed by decoding;
and \textbf{(C)} \emph{prior decode} --- decode samples from
$\mathcal N(0,I)$ directly, bypassing the score network. If a model's
diversity were destroyed by the latent diffusion, (A) would be healthy and
(B) degraded. Table~\ref{tab:collapse_abc} shows the opposite for the
KL-regularized arm: (A) is already collapsed, and (A), (B), and (C) agree,
so the decoder's output is nearly independent of its input.

\begin{table}[h]
\centering
\caption{%
Pipeline decomposition, 10K samples per cell, seed 42. (A) encode real
molecules and decode; (B) full latent-diffusion generation; (C) decode
$\mathcal N(0,I)$ samples directly. For the KL-regularized arm on HIV and
PCBA all three agree at near-zero uniqueness and IntDiv, identifying Stage~1
as the point of failure.%
}
\label{tab:collapse_abc}
\small
\resizebox{\textwidth}{!}{%
\begin{tabular}{llccccccc}
\toprule
& & \multicolumn{3}{c}{\textbf{(A) round trip}} & \multicolumn{2}{c}{\textbf{(B) generation}} & \multicolumn{2}{c}{\textbf{(C) prior decode}} \\
\cmidrule(lr){3-5}\cmidrule(lr){6-7}\cmidrule(lr){8-9}
\textbf{Dataset} & \textbf{Stage-1 objective} & Uniq. & Int.\ Div. & FCD & Uniq. & FCD & Uniq. & FCD \\
\midrule
\multirow{2}{*}{\textbf{QM9}}
 & $\beta$-VAE, KL-regularized & 100.0\% & 0.917 & 6.09 & 99.61\% & 4.72 & 100.0\% & 5.72 \\
 & \textbf{SiLD}               & 100.0\% & 0.919 & 5.94 & 100.0\% & 5.10 & 98.05\% & 8.68 \\
\midrule
\multirow{2}{*}{\textbf{HIV}}
 & $\beta$-VAE, KL-regularized & 0.95\%  & 0.019 & 60.95 & 0.98\%  & 60.91 & 0.93\%  & 61.17 \\
 & \textbf{SiLD}               & 87.76\% & 0.902 & 5.09  & 98.84\% & 10.40 & 93.97\% & 17.67 \\
\midrule
\multirow{2}{*}{\textbf{PCBA}}
 & $\beta$-VAE, KL-regularized & 0.53\%  & 0.020 & 77.10 & 0.58\%  & 77.12 & 0.65\%  & 77.06 \\
 & \textbf{SiLD}               & 97.79\% & 0.885 & 11.03 & 98.72\% & 18.04 & 92.13\% & 28.63 \\
\bottomrule
\end{tabular}%
}
\end{table}

Table~\ref{tab:collapse_latent} gives the mechanism. On the raw encoder
output, the KL-regularized arm has driven $99.6\%$ (HIV) and $92.6\%$ (PCBA)
of coordinates to standard deviation below $10^{-3}$, with a
largest-to-smallest ratio in the thousands; SiLD has no such coordinate and a
ratio below $5$. On QM9 --- the one dataset where the two methods are
comparable --- the KL arm has no collapsed coordinates either, which is
consistent with the claim that posterior collapse, not KL regularization per
se, is what the two-stage separation avoids. The effective ranks in the last
three columns are computed after per-coordinate standardization, so they measure how concentrated the covariance spectrum is rather than
collapse; in every arm the generated latents show a covariance-spectrum
concentration similar to that of the real latents, while the Gaussian prior is
far less concentrated.

\begin{table}[h]
\centering
\caption{%
Latent statistics, 10K encodings, seed 42. Dead coordinates and the
standard-deviation ratio (largest over smallest, and ``dead'' meaning
standard deviation below $10^{-3}$) are computed on the \emph{raw} encoder
output; effective rank (participation ratio) is computed after
per-coordinate standardization, on real encodings ($z_{\mathrm{real}}$), on
latents produced by Stage~2 ($z_{\mathrm{gen}}$), and on prior samples
($z_{\mathrm{prior}}$).%
}
\label{tab:collapse_latent}
\small
\resizebox{\textwidth}{!}{%
\begin{tabular}{llcccccc}
\toprule
\textbf{Dataset} & \textbf{Stage-1 objective} & $d_z$
& dead coords & std ratio
& eff.\ rank $z_{\mathrm{real}}$ & eff.\ rank $z_{\mathrm{gen}}$
& eff.\ rank $z_{\mathrm{prior}}$ \\
\midrule
\multirow{2}{*}{\textbf{QM9}}
 & $\beta$-VAE, KL-regularized & 128 & 0.0\%   & 80.4   & 3.10  & 3.42  & 84.9 \\
 & \textbf{SiLD}               & 128 & 0.0\%   & 2.0    & 25.08 & 26.84 & 85.8 \\
\midrule
\multirow{2}{*}{\textbf{HIV}}
 & $\beta$-VAE, KL-regularized & 512 & 99.6\%  & 2739.2 & 17.84 & 11.43 & 453.6 \\
 & \textbf{SiLD}               & 512 & 0.0\%   & 4.6    & 8.68  & 8.07  & 454.1 \\
\midrule
\multirow{2}{*}{\textbf{PCBA}}
 & $\beta$-VAE, KL-regularized & 256 & 92.6\%  & 1753.5 & 16.86 & 16.12 & 240.6 \\
 & \textbf{SiLD}               & 256 & 0.0\%   & 3.5    & 7.91  & 7.73  & 240.4 \\
\bottomrule
\end{tabular}%
}
\end{table}

In the SiLD latents the generated and real supports agree in effective rank
(participation ratio $8.07$ vs.\ $8.68$ out of $512$ on HIV, $7.73$ vs.\
$7.91$ out of $256$ on PCBA) while the Gaussian prior occupies $454$ and
$240$ respectively, so Stage~2 reproduces a support of the same low dimension
as the data rather than filling the latent cube.
Appendix~\ref{app:celeba-boundary} marks where the
deterministic-versus-KL ordering stops holding.

\subsection{Sampling-step budget}
\label{app:steps}

Table~\ref{tab:steps_sweep} varies the number of reverse-diffusion steps at
sampling time with the trained model held fixed. Quality is essentially flat
from $25$ to $500$ steps on all four molecular benchmarks: the FCD spread is
$0.30$ (QM9), $0.35$ (HIV), $0.97$ (MUV), and $0.68$ (PCBA), and the cheapest
setting is not systematically the worst. Generating 10K molecules takes
$1.3$--$1.7$\,s at $25$ steps against $21.6$--$22.3$\,s at $500$.

\begin{table}[h]
\centering
\caption{%
Reverse-diffusion step count at sampling time, with the trained SiLD model
held fixed. 10K samples, seed 42; wall-clock is the sampling time for the
full 10K batch on one GPU.%
}
\label{tab:steps_sweep}
\small
\begin{tabular}{lcccccccccc}
\toprule
& \multicolumn{2}{c}{\textbf{QM9}} & \multicolumn{2}{c}{\textbf{HIV}}
& \multicolumn{2}{c}{\textbf{MUV}} & \multicolumn{2}{c}{\textbf{PCBA}}
& \multicolumn{2}{c}{\textbf{wall-clock (s)}} \\
\cmidrule(lr){2-3}\cmidrule(lr){4-5}\cmidrule(lr){6-7}\cmidrule(lr){8-9}\cmidrule(lr){10-11}
steps & Uniq. & FCD & Uniq. & FCD & Uniq. & FCD & Uniq. & FCD & min & max \\
\midrule
$25$  & 97.44\% & 4.14 & 98.58\% & 10.81 & 99.53\% & 20.30 & 98.47\% & 18.44 & 1.3 & 1.7 \\
$50$  & 97.77\% & 4.27 & 98.69\% & 10.51 & 99.41\% & 19.73 & 98.59\% & 17.94 & 2.2 & 2.6 \\
$100$ & 97.36\% & 4.22 & 98.79\% & 10.49 & 99.24\% & 19.52 & 98.76\% & 17.76 & 4.3 & 4.8 \\
$200$ & 97.65\% & 4.38 & 98.78\% & 10.46 & 99.33\% & 19.33 & 98.67\% & 17.83 & 8.7 & 9.1 \\
$500$ & 97.73\% & 4.44 & 98.77\% & 10.49 & 99.44\% & 19.50 & 98.74\% & 17.99 & 21.6 & 22.3 \\
\bottomrule
\end{tabular}
\end{table}

\subsection{Novelty--fidelity trade-off on QM9}
\label{app:novelty-tradeoff}

On QM9, SiLD does not improve novelty over the
KL-regularized baseline (Table~\ref{tab:main_results}). The dataset is small
and nearly exhaustively enumerated, so novelty and distributional fidelity
are in direct tension there, and both are reachable from the same trained
model by rescaling the initial latent noise or the score at sampling time.
Table~\ref{tab:novelty_tradeoff} traces that frontier: inflating the initial
noise from $\sigma_{\mathrm{init}}=1.0$ to $1.3$ raises novelty from $35.8\%$
to $51.7\%$ and uniqueness from $97.5\%$ to $98.4\%$ while FCD degrades only
from $4.22$ to $5.19$; pushing further, or raising the score temperature,
buys novelty up to $97.7\%$ at a clear cost in FCD and QED. We report the
untuned $\sigma_{\mathrm{init}}=1.0$ setting everywhere in the main text.

\begin{table}[h]
\centering
\caption{%
Sampling-time novelty--fidelity trade-off on QM9, from a single trained SiLD
model (10K samples, 50 steps, seed 42). $\sigma_{\mathrm{init}}$ scales the
initial latent noise; $T$ scales the predicted score. The first row is the
setting used in the main text.%
}
\label{tab:novelty_tradeoff}
\small
\begin{tabular}{llccccc}
\toprule
$\sigma_{\mathrm{init}}$ & $T$ & Uniq.$\uparrow$ & Novelty$\uparrow$ & Int.\ Div.$\uparrow$ & QED$\uparrow$ & FCD$\downarrow$ \\
\midrule
$1.0$ & $1.0$ & 97.51\% & 35.8\% & 0.917 & 0.470 & 4.22 \\
$1.1$ & $1.0$ & 97.91\% & 42.0\% & 0.921 & 0.465 & 4.56 \\
$1.2$ & $1.0$ & 98.12\% & 46.5\% & 0.924 & 0.461 & 4.99 \\
$1.3$ & $1.0$ & 98.44\% & 51.7\% & 0.925 & 0.457 & 5.19 \\
$1.5$ & $1.0$ & 98.50\% & 61.3\% & 0.930 & 0.448 & 5.90 \\
$2.0$ & $1.0$ & 97.87\% & 80.7\% & 0.937 & 0.419 & 6.66 \\
\midrule
$1.0$ & $1.1$ & 98.47\% & 62.0\% & 0.928 & 0.448 & 5.43 \\
$1.0$ & $1.2$ & 97.76\% & 82.8\% & 0.937 & 0.418 & 6.20 \\
$1.0$ & $1.3$ & 91.45\% & 94.6\% & 0.944 & 0.366 & 7.94 \\
$1.0$ & $1.5$ & 86.37\% & 97.7\% & 0.946 & 0.324 & 9.61 \\
\midrule
$1.2$ & $1.2$ & 94.61\% & 91.8\% & 0.942 & 0.388 & 7.19 \\
$1.3$ & $1.2$ & 91.49\% & 94.4\% & 0.943 & 0.370 & 7.87 \\
$1.5$ & $1.3$ & 87.02\% & 97.8\% & 0.947 & 0.324 & 9.77 \\
\bottomrule
\end{tabular}
\end{table}

\subsection{Boundary of the deterministic-versus-KL ordering}
\label{app:celeba-boundary}

The ordering between deterministic and KL-regularized Stage-1 objectives is
not universal, and Table~\ref{tab:celeba_boundary} marks where it reverses.
On CelebA at $64\times64$ with a small KL weight ($10^{-4}$) and no auxiliary
regularizer, the KL-regularized arm attains the best FID and the best
reconstruction, and the pure deterministic autoencoder is worst. The latent
MMD to a standard Gaussian explains the reversal: the KL term is doing its
intended job of matching the prior ($0.0078$ against $0.0150$ for both
deterministic arms), and on this dataset it does so without collapsing any
coordinate --- exactly the regime in which Appendix~\ref{app:collapse}
predicts that removing KL buys nothing.

\begin{table}[h]
\centering
\caption{%
CelebA ($64\times64$), seed 42, identical encoder, decoder, and score network;
FID on 2048 samples, 500 sampling steps, $20$K Stage-1 and $60$K Stage-2
steps. ``Latent MMD'' is the maximum mean discrepancy between the encoded
data and $\mathcal N(0,I)$. No auxiliary regularizer (MMD, GAN, LPIPS) is
active in any row.%
}
\label{tab:celeba_boundary}
\small
\begin{tabular}{lcccc}
\toprule
Stage-1 objective & Recon MSE $\downarrow$ & Latent MMD $\downarrow$ & $\epsilon$-loss $\downarrow$ & FID $\downarrow$ \\
\midrule
$\beta$-VAE, KL-regularized ($\lambda_{\mathrm{KL}}=10^{-4}$) & $\mathbf{1.283\!\times\!10^{-3}}$ & $\mathbf{0.0078}$ & $\mathbf{0.277}$ & $\mathbf{63.64}$ \\
SiLD ($\sigma_{\mathrm{geom}}=5\!\times\!10^{-3}$)            & $1.338\!\times\!10^{-3}$ & $0.0150$ & $0.315$ & $64.70$ \\
SiLD, clean AE ($\sigma_{\mathrm{geom}}=0$)                   & $1.449\!\times\!10^{-3}$ & $0.0151$ & $0.318$ & $66.06$ \\
\bottomrule
\end{tabular}
\end{table}

\end{document}